\documentclass{article}

\usepackage{amsthm,amsmath,amssymb}
\usepackage{eucal}
\usepackage{xifthen}
\usepackage{mathtools}
\usepackage{enumerate}
\usepackage{microtype}
\usepackage{xspace}
\usepackage{bm}
\usepackage{cite}
\usepackage{fancyhdr}
\usepackage{lastpage}
\usepackage[small]{caption}
\usepackage{xcolor}
\usepackage{algorithm}
\usepackage{algpseudocode}
\allowdisplaybreaks

\long\def\comment#1{}

\newfont{\bbb}{msbm10 scaled 700}

\newfont{\bb}{msbm10 scaled 1100}
\newcommand{\CC}{\mbox{\bb C}}

\newcommand{\RR}{\mbox{\bb R}}

\newcommand{\mbs}[1]{\bm{#1}}
\newcommand{\vect}[1]{{\lowercase{\mbs{#1}}}}
\newcommand{\mat}[1]{{\uppercase{\mbs{#1}}}}

\newcommand{\Pmatrix}[1]{\begin{array}{ll}#1\end{array}}

\newcommand{\T}{{\scriptscriptstyle\mathsf{T}}}
\renewcommand{\H}{\scriptscriptstyle\mathsf{H}}

\renewcommand{\Re}[1][]{\ifthenelse{\isempty{#1}}{\operatorname{Re}}{\operatorname{Re}\left(#1\right)}}
\renewcommand{\Im}[1][]{\ifthenelse{\isempty{#1}}{\operatorname{Im}}{\operatorname{Im}\left(#1\right)}}

\newcommand{\av}{\vect{a}}

\newcommand{\ev}{\vect{e}}

\newcommand{\iv}{\vect{i}}
\newcommand{\jv}{\vect{j}}
\newcommand{\kv}{\vect{k}}

\newcommand{\mv}{\vect{m}}
\newcommand{\nv}{\vect{n}}

\newcommand{\tv}{\vect{t}}
\newcommand{\uv}{\vect{u}}
\newcommand{\vv}{\vect{v}}

\newcommand{\xv}{\vect{x}}

\newcommand{\zv}{\vect{z}}

\newcommand{\omegav}{\hbox{\boldmath$\omega$}}

\newcommand{\Deltam}{\hbox{\boldmath$\Delta$}}

\newcommand{\Pim}{\hbox{\boldmath$\Pi$}}

\newcommand{\Am}{\mat{a}}
\newcommand{\Bm}{\mat{b}}
\newcommand{\Cm}{\mat{c}}

\newcommand{\Em}{\mat{e}}
\newcommand{\Fm}{\mat{f}}

\newcommand{\Jm}{\mat{j}}
\newcommand{\Km}{\mat{k}}
\newcommand{\Lm}{\mat{l}}

\newcommand{\Pm}{\mat{p}}

\newcommand{\Rm}{\mat{r}}

\newcommand{\Tm}{\mat{t}}

\newcommand{\Xm}{\mat{x}}
\newcommand{\Ym}{\mat{y}}

\newcommand{\Bc}{{\mathcal B}}

\newcommand{\Lc}{{\mathcal L}}
\newcommand{\Mc}{{\mathcal M}}

\newcommand{\Id}{\mat{{I}}}
\newcommand{\one}{\mat{\mathrm{1}}}
\newcommand{\zero}{\mat{\mathrm{0}}}

\newcommand{\CN}[1][]{\ifthenelse{\isempty{#1}}{\mathcal{N}_{\mathbb{C}}}{\mathcal{N}_{\mathbb{C}}\left(#1\right)}}
\renewcommand{\P}[1][]{\ifthenelse{\isempty{#1}}{\mathbb{P}}{\mathbb{P}\left(#1\right)}}
\newcommand{\E}[1][]{\ifthenelse{\isempty{#1}}{\mathbb{E}}{\mathbb{E}\left(#1\right)}}
\renewcommand{\det}[1][]{\ifthenelse{\isempty{#1}}{\mathrm{det}}{\mathrm{det}\left(#1\right)}}
\newcommand{\trace}[1][]{\ifthenelse{\isempty{#1}}{\mathrm{tr}}{\mathrm{tr}\left(#1\right)}}
\newcommand{\rank}[1][]{\ifthenelse{\isempty{#1}}{\mathrm{rank}}{\mathrm{rank}\left(#1\right)}}
\newcommand{\diag}[1][]{\ifthenelse{\isempty{#1}}{\mathrm{diag}}{\mathrm{diag}\left(#1\right)}}

\DeclarePairedDelimiter\abs{\lvert}{\rvert}
\DeclarePairedDelimiter\Abs{\lvert}{\rvert^2}
\DeclarePairedDelimiter\norm{\lVert}{\rVert}
\DeclarePairedDelimiter\Norm{\lVert}{\rVert^2}
\DeclarePairedDelimiter\normf{\lVert}{\rVert_{\mathrm{F}}}
\DeclarePairedDelimiter\Normf{\lVert}{\rVert^2_{\mathrm{F}}}

\renewcommand{\Re}{{\rm Re}}
\renewcommand{\Im}{{\rm Im}}

\renewcommand{\vec}{{\rm vec}}

\allowdisplaybreaks

\DeclareMathAlphabet{\mathcal}{OMS}{cmsy}{m}{n}

\newcommand{\defeq}{\triangleq}

\newtheorem{remark}{Remark}
\newtheorem{theorem}{Theorem}
\newtheorem{lemma}{Lemma}

\usepackage{hyperref}
\hypersetup{
    bookmarks=true,         
    unicode=false,          
    pdftoolbar=true,        
    pdfmenubar=true,        
    pdffitwindow=false,     
    pdfstartview={FitH},    
    pdfnewwindow=true,      
    colorlinks=true,       
    linkcolor=red,          
    citecolor=green,        
    filecolor=magenta,      
    urlcolor=cyan           
}

\usepackage{wrapfig}
\usepackage{lipsum}
\renewcommand{\circ}{\mathrm{circ}}

\usepackage[preprint]{neurips_2020}

\usepackage[utf8]{inputenc} 
\usepackage[T1]{fontenc}    
\usepackage{booktabs}       
\usepackage{amsfonts}       
\usepackage{nicefrac}       
\usepackage{microtype}      

\title{Asymptotic Singular Value Distribution of Linear Convolutional Layers}

\author{%
  Xinping Yi\\
  Department of Electrical Engineering and Electronics\\
  University of Liverpool\\
  \texttt{xinping.yi@liverpool.ac.uk} 
}

\renewcommand{\H}{\scriptscriptstyle\mathsf{H}}

\begin{document}

\maketitle

\begin{abstract}
In convolutional neural networks, the linear transformation of multi-channel two-dimensional convolutional layers with \emph{linear} convolution is a block matrix with doubly Toeplitz blocks. 
Although a ``wrapping around'' operation can transform linear convolution to a circular one, by which the singular values can be approximated with reduced computational complexity by those of a block matrix with doubly circulant blocks, the accuracy of such an approximation is not guaranteed.
In this paper, we propose to inspect such a linear transformation matrix through its asymptotic spectral representation - the spectral density matrix - by which we develop a simple singular value approximation method with improved accuracy over the circular approximation, as well as upper bounds for spectral norm with reduced computational complexity. Compared with the circular approximation, we obtain moderate improvement with a subtle adjustment of the singular value distribution. We also demonstrate that the spectral norm upper bounds are effective spectral regularizers for improving generalization performance in ResNets.
\end{abstract}

\section{Introduction}
\label{intro}
The last decade has witnessed the success of convolutional neural networks (CNNs) in various artificial intelligence applications, such as computer vision and natural language processing.
In CNNs, convolutional layers perform efficient linear transformation from their input data through a linear or circular convolutional operation. Inspecting such a linear transformation lends itself to theoretical understanding of the behaviors of convolutional layers in CNNs with respect to, e.g., stability, generalization performance, and gradient explosion or vanishing effects.

Such a linear transformation of convolutional layers plays the same role as the weight matrix of the fully-connected layers. As a powerful means to matrix analysis, spectral methods have been applied to understand the properties of the weight matrices of deep neural networks through inspecting the behavior of their singular values, such as \citet{bartlett2017spectrally,sedghi2018the,Singla2019,miyato2018spectral,yoshida2017spectral,farnia2018generalizable,neyshabur2017pac,roth2019adversarial,Long2020Generalization}, to name just a few. 

The singular values play a key role in spectral analysis of deep neural networks, where spectral norm is the largest singular value and Frobenius norm involves all singular values of the weight matrices. 
It has been shown in \citet{neyshabur2017pac} that generalization error is upper bounded by spectral and Frobenius norms of the weight matrices of the layers, for which suppressing singular values can reduce the gap and therefore enhance the generalization performance. 
In addition, for CNNs, spectral regularization has been also applied to convolutional layers so as to guide the training process by e.g., clipping singular values within an interval to avoid explosion or vanishing of gradients \citet{sedghi2018the}, and bounding spectral norms to enhance generalization performance and robustness against adversarial examples \citet{yoshida2017spectral,Singla2019,miyato2018spectral}.

However, as the size of such linear transformation matrices grows with the input size of the layers, it is computationally challenging to find singular values. The straightforward singular value decomposition (SVD) incurs huge computational burden, which is even worse when singular values are required to be computed during the training process to guide spectral regularization and normalization \citet{miyato2018spectral,yoshida2017spectral}.
Fortunately, the structures of the linear transformation matrices can be exploited to reduce the computational complexity of SVDs. Of particular relevance is the work by \citet{sedghi2018the,bibi2018deep}, in which the linear convolutional layer is treated as a circular one by a ``wrapping round'' operation. In doing so, the linear transformation matrices are endowed with a circulant structure, by which efficient methods were proposed to compute a circular approximation of the convolutional layers with substantially reduced complexity. To further reduce computational complexity, upper bounds of spectral norm of the circular convolutional layers were derived in \citet{Singla2019} at the expense of degraded accuracy.

As a matter of fact, such a ``wrapping round'' operation is not always endowed in many convolutional layers, for which a linear, rather than circular, convolutional operation is applied. With such a linear convolution, the linear transformation matrix has a Toeplitz structure, which includes the circulant one as a special case. This has been pointed out by a number of previous works, e.g., \citet{goodfellow2016deep,wang2019orthogonal,appuswamy2016structured}, that the two-dimensional single-channel convolutional layer results in a doubly block Toeplitz matrix. A question then arises as to how close is the circular approximation to the exact linear Toeplitz case.\footnote{Although some theoretical analysis bounded the gap between large Toeplitz and circulant matrices \citet{ZhuTIT2017}, it seems only applied to Hermitian matrices (or symmetric for real matrices). The linear transformation matrices of linear convolution are {\em asymmetric} real matrices, which are {\em non-Hermitian} matrices.} This motivates the current work.

\paragraph{Our Contributions}
In this paper, we consider the linear convolutional layers, with main focus on the multi-channel two-dimensional linear convolution with stride size of 1, so that the linear transformation matrix is a block matrix with each block being a doubly Toeplitz matrix. By rows and columns permutation, we construct an alternative representation as a doubly block Toeplitz matrix with each element being a matrix, for which the singular values of both representations are identical. 
As such, we propose a spectral representation of the linear transformation matrix
by a spectral density matrix, 
by which the spectral analysis of the former can be alternatively done on the latter. Specifically, the main contributions are three-fold:
\begin{itemize}
\vspace{-2pt}
   \item The singular value distribution of linear transformation matrix of CNNs is cast to that of its spectral density matrix, thanks to an extension of the celebrated Szeg\"o Theorem for Hermitian Toeplitz matrices to non-Hermitian block doubly Toeplitz matrices. In doing so, the asymptotic spectral analysis of the linear convolutional layers can be alternatively done by inspecting the corresponding spectral density matrix. The circular convolution by ``wrapping around'' is a special case of such a spectral representation, by which the singular values can also be produced by uniformly sampling the spectral density matrix.
    \item By treating singular values of the spectral density matrix as random variables, the individual singular value distribution can be quantified by a quantile function. As such, we propose a simple yet effective algorithm to compute singular values of linear convolutional layers by subtly adjusting the singular value distribution obtained from the circular approximation.
    \item To upper-bound the spectral norm of the linear transformation matrix, we instead upper-bound that of its corresponding spectral density matrix. As a consequence, we come up with three spectral norm bounds that can be used for spectral regularization.
    \vspace{-2pt}
\end{itemize}
Experimental results demonstrate the superior accuracy of our singular value approximation method and the effectiveness of spectral norm bounds for regularization with respect to generalization in practical CNN models, e.g., ResNets. Notations and preliminaries can be found in Section \ref{sec:prelim}.

\section{Convolutional Neural Networks}
\subsection{Linear Convolutional Layer}
We consider multiple-channel two-dimensional {\em linear} convolutional layers with {\em arbitrary padding} schemes in CNNs before applying activation functions and pooling. For ease of presentation, we first consider the stride size 1, and the extension to larger stride size will be discussed in Section \ref{sec:discussion}.

Let the input be $\Xm \in \mathbb{R}^{c_{in}\times n \times n}$ and the linear convolutional filter be $\Km \in \mathbb{R}^{c_{out} \times c_{in} \times h \times w}$ with $h,w \le n$, where $n,h,w,c_{in},c_{out}$ are input size, filter height, filter width, the numbers of input and output channels, respectively. For convenience, we let the output $\Ym$ have the same size as the input $\Xm$ by arbitrary padding strategies,
and abuse $\Xm$ as the input with padding. 
By applying linear convolution of the filter $\Km$ to the input $\Xm$, the output $\Ym \in \mathbb{R}^{c_{out}\times n \times n}$ can be given by
\begin{align}
\Ym_{c,r,s}=\sum_{d=1}^{c_{in}} \sum_{p=1}^n \sum_{q=1}^n \Xm_{d,r+p,s+q} \Km_{c,d,p,q}
\end{align}
for $r,s \in [n]$ and $c \in [c_{out}]$ where $\Km_{c,d,p,q}=0$ if $p,q$ exceed the ranges of $h,w$.
A compact form of the above input-output relation can be rewritten as
\begin{align}
    \vec(\Ym) = \Am \vec(\Xm),
\end{align}
where $\Am \in \mathbb{R}^{c_{out}n^2 \times c_{in} n^2}$ is the linear transformation matrix of the convolutional layer. 
For the general case with multiple-input and multiple-output channels, the linear transformation can be represented as a $c_{out}\times c_{in}$ block matrix, i.e.,
\begin{align} \label{eq:M-1}
\Am = \begin{bmatrix}
\Am_{1,1} & \Am_{1,2} & \dots & \Am_{1,c_{in}}\\
\Am_{2,1} & \Am_{2,2} & \dots & \Am_{2,c_{in}}\\
\vdots & \vdots & & \vdots\\
\Am_{c_{out},1} & \Am_{c_{out},2} & \dots & \Am_{c_{out},c_{in}}
\end{bmatrix},
\end{align}
where each block is a doubly Toeplitz matrix, i.e., $[\Am_{c,d}]_{i_1,j_1}=\Am_{i_1-j_1}^{c,d}$ with $[\Am_{k}^{c,d}]_{i_2,j_2}=a_{k,i_2-j_2}^{c,d}$ (See a concrete representation in Section \ref{sec:prelim-toe}). In matrix analysis, $\Am$ is usually referred to as multi-block multi-level (doubly) Toeplitz matrix.
For $k\in [-h_1:h_2]$ and $l \in [-w_1:w_2]$,
we have
\begin{align} \label{eq:M4}
a_{k,l}^{c,d} = \Km_{c,d,h_1+k+1,w_1+l+1},
\end{align}
for all $c \in [c_{out}]$ and $d \in [c_{in}]$.

\subsection{Alternative Representation}
\label{sec:section2.2}
For ease of spectral analysis, we transform $\Am$ into a multi-level block Toeplitz matrix (whose entries of the last level are matrices) via vec-permutation operation \citet{henderson1981vec}, for which the matrix spectrum keeps unchanged. 

Denote by $\Tm \in \RR^{c_{out}n^2 \times c_{in} n^2}$ the alternative representation as a block Toeplitz matrix with 
$[\Tm]_{i_1,j_1}=\Tm_{i_1-j_1}$ where $[\Tm_k]_{i_2,j_2}=\Tm_{k,i_2-j_2}$ (See a concrete representation in Section \ref{sec:prelim-toe}).
For $k \in [-h_1:h_2]$ and $l\in [-w_1:w_2]$, each block $\Tm_{k,l} \in \RR^{c_{out} \times c_{in}}$ is given by %
\begin{align} \label{eq:toe-blocks}
\Tm_{k,l}=\begin{bmatrix}
t_{1,1}^{k,l} & t_{1,2}^{k,l} & \cdots & t_{1,c_{in}}^{k,l}\\
t_{2,1}^{k,l} & t_{2,2}^{k,l} & \cdots & t_{2,c_{in}}^{k,l}\\
\vdots & \ddots & \ddots & \vdots \\
t_{c_{out},1}^{k,l} & t_{c_{out},2}^{k,l} & \cdots & t_{c_{out},c_{in}}^{k,l}
\end{bmatrix}.
\end{align}
By such an alternative representation, we have
\begin{align} \label{eq:T_k_l}
t_{c,d}^{k,l}=\Km_{c,d,h_1+k+1,w_1+l+1} = a_{k,l}^{c,d},
\end{align}
for all $c \in [c_{out}]$ and $d \in [c_{in}]$.
In what follows, we show that the alternative representation $\Tm$ of the linear convolutional layers has the identical spectrum structure as the original form $\Am$.
\begin{lemma} \label{lemma:toe-permutation-invariant}
$\{\sigma_j(\Tm), \; \forall j\}=\{\sigma_j(\Am), \; \forall j\}$.
\end{lemma}
Lemma \ref{lemma:toe-permutation-invariant} says the block matrix with doubly Toeplitz matrix blocks (i.e., $\Am$) has the same set of singular values as the block doubly Toeplitz matrix (i.e., $\Tm$). This holds for any Toeplitz matrices which are not necessarily banded, and for any multi-level case but not limited to doubly (i.e., 2-level) Toeplitz case.
Equipped with this lemma, we hereafter treat $\Tm$ as the linear transformation matrix of linear convolutional layers for spectral analysis.

\subsection{Circular Approximation}
The ``wrapping around'' operation  makes linear transformation a circular convolution, which is deemed as a circular approximation of linear convolution.
As $h,w \le n$, we can construct a circulant matrix by ``wrapping around'' to assist the spectral analysis.

Given the doubly block Toeplitz matrix $\Tm=[\Tm_{i-j}]_{i,j=1}^n$ with $\Tm_k=0$ if $k>h_2$ or $k<-h_1$ and $\Tm_k=[\Tm_{k,p-q}]_{p,q=1}^n$ with $\Tm_{k,l}=0$ if $l>w_2$ or $l<-w_1$, the doubly block circulant matrix $\Cm=\circ(\Cm_0,\Cm_1,\dots,\Cm_{n-1})$ is as follows
\begin{align} \label{eq:Cm_k}
\Cm_k=\left\{
\Pmatrix{\Tm_{-k}, & k\in  \{0\}\cup[h_1]\\
\Tm_{n-k}, & k \in n-[h_2]\\
0,& \text{otherwise}}
\right.
\end{align}
where $\Cm_k=\circ(\Cm_{k,0},\Cm_{k,1},\dots,\Cm_{k,n-1})$ with
\begin{align} \label{eq:C_k_l}
\Cm_{k,l}=\left\{
\Pmatrix{\Tm_{-k,-l}, & k\in \{0\}\cup[h_1], \; l\in\{0\}\cup[w_1]\\
\Tm_{-k,n-l}, & k\in \{0\}\cup[h_1], \; l \in n-[w_2]\\
\Tm_{n-k,-l}, &k \in n-[h_2], \; l\in\{0\}\cup[w_1]\\
\Tm_{n-k,n-l}, & k \in n-[h_2], \; l \in n-[w_2]\\
0,& \text{otherwise}}
\right.
\end{align}
where $\Tm_{k,l}$ is defined in \eqref{eq:toe-blocks}.

In a similar way, the original block doubly Toeplitz matrix $\Am$ can also have a corresponding block doubly circulant matrix $\Cm(\Am)=[\Cm(\Am_{c,d})]_{c,d=1}^{c_{out},c_{in}}$ where
\begin{align} \label{eq:CM-1}
\Cm(\Am_{c,d})=\circ(&\Cm(\Am_0^{c,d}),\Cm(\Am_{-1}^{c,d}), \dots, \Cm(\Am_{-h_1}^{c,d}), 0, 
\dots,0, \Cm(\Am_{h_2}^{c,d}),\dots,\Cm(\Am_1^{c,d}))
\end{align}
with $\Cm(\Am_{c,d}) \in \RR^{n^2 \times n^2}$ where
\begin{align} \label{eq:CM-2}
\Cm(\Am_k^{c,d})=\circ(a_{k,0}^{c,d}, &a_{k,-1}^{c,d}, \dots, a_{k,-w_1}^{c,d}, 0, 
\dots,0, a_{k,w_2}^{c,d},\dots,a_{k,1}^{c,d})
\end{align}
with $\Cm(\Am_k^{c,d}) \in \RR^{n \times n}$.
Similarly to Lemma \ref{lemma:toe-permutation-invariant}, we have the following lemma.
\begin{lemma} \label{lemma:cir-permutation-invariant}
$\{\sigma_j(\Cm), \; \forall j\}=\{\sigma_j(\Cm(\Am)), \; \forall j\}$.
\end{lemma}

It can be easily verified that $\Cm(\Am)$ is essentially the linear transformation matrix of circular convolutional layers considered in \citet{sedghi2018the}. 
As a byproduct of Lemma \ref{lemma:cir-permutation-invariant}, we present an alternative calculation of the singular values for the circular convolutional layers that were characterized in \citet{sedghi2018the}.

\begin{lemma} \label{lemma:toe-equal-cir}
The linear transformation matrix $\Cm(\Am)$ can be block-diagonalized as
\begin{align} \label{eq:block-diag}
\Cm = (\Fm_n \otimes \Fm_n \otimes \Id_{c_{out}}) &\mathrm{blkdiag} (\Bm_{1,1},\Bm_{1,2},\dots, \Bm_{1,n},  \Bm_{2,1}, \dots \Bm_{n,n}) (\Fm_n \otimes \Fm_n \otimes \Id_{c_{in}})^{\H}    
\end{align}
where both $(\Fm_n \otimes \Fm_n \otimes \Id_{c_{out}})$ and $(\Fm_n \otimes \Fm_n \otimes \Id_{c_{in}})$ are unitary matrices.
Thus, the singular values of $\Cm(\Am)$ are the collection of singular values of 
$
    \{\Bm_{i,k}\}_{i,k=1}^n
$
where 
\begin{align}
\Bm_{i,k}= \sum_{p=0}^{n-1} \sum_{q=0}^{n-1} \Cm_{p,q} e^{-\jmath 2 \pi \frac{ p(i-1)+q(k-1)}{n}}
\end{align}
with $\Cm_{p,q}$ defined in \eqref{eq:C_k_l}.
\end{lemma}

The computation of $\Bm_{i,k}$ can be seen as a two-dim DFT of $\Cm_{p,q}$. With $hw$ non-zero submatrices $\{\Cm_{p,q}\}$, the computational complexity consists in $hw$ FFTs and $n^2$ SVDs, which is identical to that in \citet{sedghi2018the}. We also point out that this alternative approach essentially has the same flavor as that in \citet{bibi2018deep}.

Given Lemmas 1-3, we hereafter take $\Tm$ as the linear transformation matrix of the linear convolutional layer and $\Cm$ as its circular approximation, for asymptotic spectral analysis. 

\section{Asymptotic Spectral Analysis}
\label{sec:spectral} 
In what follows, we present asymptotic spectral analysis for the linear transformation matrix $\Tm$ of convolutional layers in CNNs, taking advantage of its Toeplitz structure \citet{Gray1972,avram1988bilinear,parter1986distribution,voois1996theorem,Tilli1998,miranda2000asymptotic,Tyrtyshnikov1996,Zizler2002,Bogoya2015}. The proofs and insights are relegated to Section \ref{sec:proofs}. %
\subsection{Spectral Representation}
\begin{theorem} \label{thm:szg-limit}
Given the linear transformation matrix $\Tm \in \CC^{rn^2\times sn^2}$, let a complex matrix-valued Lebesgue-measurable function $F:[-\pi,\pi]^2 \mapsto \mathbb{C}^{r \times s}$ be the generating function such that 
\begin{align*}
    \Tm_{k,l} = \frac{1}{(2\pi)^2} \int_{-\pi}^{\pi} \int_{-\pi}^{\pi} F(\omega_1,\omega_2) e^{-\jmath (k \omega_1 + l \omega_2)} d\omega_1 d\omega_2.
\end{align*}
It follows that, for any continuous function $\Phi$ with compact support in $\mathbb{R}$, we have
\begin{align} 
    \MoveEqLeft  \lim_{n \to \infty} \frac{1}{n^2} \sum_{j=1}^{\min\{r,s\}n^2} \Phi(\sigma_j(\Tm)) 
    = \frac{1}{(2\pi)^2} \int_{-\pi}^{\pi} \int_{-\pi}^{\pi} \sum_{j=1}^{\min\{r,s\}} \Phi(\sigma_j(F(\omega_1,\omega_2))) d\omega_1 d\omega_2, \notag
\end{align}
for which $\Tm$ is said to be equally distributed as $F(\omega_1,\omega_2)$ with respect to singular values, i.e., $\Tm \sim_{\sigma} F$. Specifically, for linear convolutional layers, the linear transformation matrix $\Tm$ has doubly banded structures, so that the generating function can be explicitly written as 
\begin{align} \label{eq:theorem-1-F}
    F(\omega_1,\omega_2)=\sum_{k=-h_1}^{h_2} \sum_{l=-w_1}^{w_2} \Tm_{k,l} e^{\jmath  (k \omega_1 + l \omega_2)},
\end{align}
which is also referred to as the spectral density matrix of $\Tm$.
\end{theorem}

Theorem 1 endows the linear transformation matrix $\Tm$ of linear convolutional layers with an asymptotic spectral representation - the spectral density matrix $F(\omega_1,\omega_2)$ - by establishing the collective equivalence of their asymptotic singular value distributions. As such, the spectral analysis of linear convolutional layers of CNNs can be alternatively done on its spectral representation $F(\omega_1,\omega_2)$.
The singular values of $\Tm$ can be clustered into $\min\{r,s\}$ non-overlapping subsets. When $n$ is sufficiently large, the singular values in the $j$-th subset concentrate on $\sigma_j(F)$, where $\sigma_j(F)$ is the $j$-th singular value function of $F(\omega_1,\omega_2)$.
As such, the singular values of $\Tm$ can be approximately obtained by sampling $\sigma_j(F)$ over a uniform gird in $ [-\pi,\pi]^2$, 
for all $j \in [\min\{r,s\}]$. It turns out that such approximation is equivalent to the circular approximation, which will be detailed in Theorem \ref{thm:circulant}.

\begin{theorem} \label{thm:circulant}
Given $\Tm$ and $\Cm$ as in \eqref{eq:Cm_k}-\eqref{eq:C_k_l},
there exists a constant $c_1>0$ such that
\begin{align}
    \lim_{n \to \infty} \frac{1}{n} \sum_{j=1}^{\min\{r,s\}n^2} \abs{\sigma_j(\Tm)-\sigma_j(\Cm)} \le c_1,
\end{align}
where the singular values of $\Cm$ are the collection of singular values of $\{\sigma_j(F(\omega_1,\omega_2))\}_j$ with
\begin{align}
\label{eq:uniform-samples}
(\omega_1,\omega_2) = (-\pi+ \frac{2 \pi j_1}{n}, -\pi &+ \frac{2 \pi j_2}{n}), \quad \forall j_1,j_2 \in [n]-1.
\end{align}
\end{theorem}

Theorem 2 shows that  the singular values of the circular approximation of the linear convolution can be alternatively obtained by uniformly sampling the spectral density matrix $F(\omega_1,\omega_2)$ over $(\omega_1,\omega_2) \in [-\pi,\pi]^2$, where the {\em average} difference of the overall singular values from the exact ones is bounded by $O(\frac{1}{n})$,\footnote{The big O notation $O(n)$ follows the standard Bachmann–Landau notation, meaning that there exists a positive constant $c>0$ such that the term is upper-bounded by $cn$.} and tends to zero as $n$ increases. 

\begin{remark}
The block diagonal matrices $\Bm_{i,k}$ of $\Cm$ in \eqref{eq:block-diag} is essentially the matrix-valued function $F(\omega_1,\omega_2)$ with uniform sampling on grids as in \eqref{eq:uniform-samples}, i.e.,
\begin{align}
    \Bm_{j_1,j_2} = F\Big(\frac{ 2 \pi (j_1-1)}{n},&\frac{ 2 \pi (j_2-1)}{n}\Big), \quad 
    \forall j_1,j_2 \in [n].
\end{align}
\end{remark}

Collecting all singular values $\{\sigma_j(F)\}_j$ according to the uniform sampling grids as in \eqref{eq:uniform-samples}, we sort them in non-decreasing order as $(\kappa_1,\kappa_2,\dots,\kappa_{N})$.
Let $\psi: [0,1] \mapsto \RR$ be a piece-wise linear non-decreasing function that interpolates the samples $(\kappa_1,\kappa_2,\dots,\kappa_{N})$ over the nodes $(0,\frac{1}{N},\frac{2}{N},\dots,1)$ such that $\psi(\frac{i}{N})=\kappa_i$ for all $i  \in \{0\} \cup [N]$ and $\psi(\cdot)$ is linear between any two consecutive nodes. Then we have
\begin{align}
    \frac{1}{(2\pi)^2} \int_{-\pi}^{\pi} \int_{-\pi}^{\pi} \sum_{j=1}^{\min\{r,s\}} &\Phi(\sigma_j(F(\omega_1,\omega_2))) d\omega_1 d\omega_2 
    = \int_{0}^1 \Phi(\psi(t)) dt.
\end{align}
It means the singular values of $\Tm$ can be approximately obtained by sampling the density function $\psi(t)$ in $[0,1]$. 
This motivates a singular value approximation method in Theorem \ref{thm:quantile}.

\subsection{Singular Value Approximation}
From a probabilistic perspective, Theorem \ref{thm:szg-limit} implies that the statistical average of the singular values of $\Tm$ converges to that of the singular values of the corresponding spectral density matrix $F$ in distribution with any continuous functions $\Phi$. Inspired by this, we propose a method to approximate $\sigma_j(\Tm)$ through the singular value distribution of $\sigma_j(F)$ with bounded approximation error.

\begin{theorem} \label{thm:quantile}
Let $\phi_j:[-\pi,\pi]^2 \mapsto \RR_+$ be the $j$-th singular value function of $F(\omegav)$ and $\sigma_k^{(j)}(\Tm)$ be $k$-th singular value of $j$-th cluster. It follows that
\begin{align}
   \sup_{u \in (\frac{k-1}{n^2},\frac{k}{n^2}]} \abs{\sigma_k^{(j)}(\Tm)-Q_{\phi_j}(u)} \le \frac{c_2}{n}, \quad 
     \forall 1\le k \le n^2, \;& 1 \le j \le \min\{r,s\}
\end{align}
where $c_2>0$ is a constant that only depends on $F(\omegav)$, and
\begin{align}
    Q_{\phi_j}(u)&=\inf\{v \in \RR: u \le G_{\phi_j}(v)\}\\
    G_{\phi_j}(v)&=\frac{1}{(2\pi)^2}\mu\{\omegav \in [-\pi,\pi]^2: \phi_j(\omegav)\le v\}
\end{align}
are quantile and cumulative distribution functions for $\phi_j(\omegav)$, respectively, and $\mu$ is Lebesgue measure. 
\end{theorem}

\begin{wrapfigure}{R}{0.58\textwidth}
\begin{minipage}{0.58\textwidth}
\vspace{-18pt}
\begin{algorithm}[H]
   \caption{Singular Values via Quantile Interpolation}
   \label{alg:quantile}
\begin{algorithmic}[1]
   \State {\bfseries Input:} Convolutional filter $\Km \in \RR^{c_{out} \times c_{in} \times h \times w}$
   \State Initialize $h_1,h_2,w_1,w_2$ 
   \State Construct $\Tm_{k,l}$ from $\Km$ according to \eqref{eq:toe-blocks}
   \For{$j_1=1$ {\bfseries to} $n$}
   \For{$j_2=1$ {\bfseries to} $n$}
   \State Set $(\omega_1,\omega_2)=(-\pi+\frac{2\pi j_1}{n},-\pi+\frac{2\pi j_1}{n})$
  \State Compute $F(\omega_1,\omega_2)$ by \eqref{eq:theorem-1-F}
  \State Compute SVD of $F(\omega_1,\omega_2)$
  \EndFor
  \EndFor
  \For{$j=1$ {\bfseries to} $\min\{r,s\}$}
  \State Collect singular values $\{\sigma_j(F(\omega_1,\omega_2))\}_{\omega_1,\omega_2}$ 
  \State Arrange $\sigma_j(F(\omega_1,\omega_2))$ in descending order
  \State{Estimate quantile $\hat{Q}_{\phi_j}$ by $\{\sigma_j(F(\omega_1,\omega_2))\}_{\omega_1,\omega_2}$}
  \State Interpolate quantile using e.g., kernel smoothing
  \State Select proper $u=\{\frac{j-\gamma_j}{n^2}\}_{j=1}^{n^2}$ with $\gamma_j \in (0,1)$
  \State Compute $\{\hat{Q}_{\phi_j}(u)\}_u$ as singular value estimates
  \EndFor
  \State {\bfseries Output:} Singular values $\{\{\hat{Q}_{\phi_j}(u)\}_u\}_j$
\end{algorithmic}
\end{algorithm}
\end{minipage}
\end{wrapfigure}
Theorem \ref{thm:quantile} reveals that the individual singular value of $\sigma_j(\Tm)$ can be approximated by sampling the quantile function of $\phi_j(\omegav)$ within each interval $ (\frac{k-1}{n^2},\frac{k}{n^2}]$.  If the estimation of the quantile function is perfect, this approach approximates each {\em individual} singular value with gap to the exact one within $O(\frac{1}{n})$. %

\begin{remark}
It is challenging to compute the closed-form expression of the singular value function\footnote{As $F(\omega_1,\omega_2)$ is a Laurent polynomial matrix with respect to $e^{\jmath \omega_1}$ and $e^{\jmath \omega_2}$, the singular value functions $\phi_j(\omegav)$ can be computed efficiently by, e.g., \citet{foster2009algorithm}.} $\phi_j(\omegav)$ from $F(\omegav)$, so is its quantile function.
Alternatively, $Q_{\phi_j}(u)$ can be estimated from some easily attainable samples, e.g., $\{\sigma_j(\Cm)\}_j$, which are the uniform sampling of $\sigma_j(F)$ on $[-\pi,\pi]^2$, followed by quantile interpolation/extrapolation with e.g., kernel smoothing tricks. As such, the singular value approximation can be done by properly sampling the interpolated quantile function.
In this way, the approximation accuracy of $\{\sigma_j(\Tm)\}_j$ depends on (1) the accuracy of quantile estimation from the samples, (2) the smoothing factors of quantile interpolation, and (3) the sampling grid in $ (\frac{k-1}{n^2},\frac{k}{n^2}]$. Alg.~\ref{alg:quantile} presents a simple approach to approximate $\{\sigma_j(\Tm)\}_j$ through quantile estimation and interpolation.
\end{remark}

For quantile interpolation, a simple way is linear interpolation, which uses linear polynomials to interpolate new values between two consecutive data points. Kernel density estimation can be used to smooth interpolation. Some other interpolation methods, such as t-Digests \citet{dunning2019computing}, are also available in Python and MATLAB from 2019b onward.

\subsection{Spectral Norm Bounding}
Thanks to the spectral representation, spectral analysis on the linear transformation matrix $\Tm$ can be alternatively done on the spectral density matrix $F(\omegav)$ with $\omegav \in [-\pi,\pi]^2$. For instance, to upper-bound spectral norm of $\Tm$, we can do it on $F$ due to the following lemma.
\begin{lemma}
\label{lemma:bounding-Tm}
$
\norm{\Tm}_2 \le \norm{F}_2.
$
\end{lemma}

Built upon Lemma \ref{lemma:bounding-Tm}, the spectral norm of $\Tm$ can be further upper-bounded in different ways.

\begin{theorem}
\label{theorem:bounding-norm}
The spectral norm $\norm{F}_2$ can be bounded by
\begin{align} 
\label{eq:upper-bound-F-2norm}
\norm{F}_2 &\le \min \Big\{\sqrt{hw}\norm{\Rm}_2, \sqrt{hw}\norm{\Lm}_2 \Big\},\\
\label{eq:upperibound-oneinfnorm}
\norm{F}_2 &\le \max_{\omegav} \sqrt{\norm{F(\omegav)}_1 \norm{F(\omegav)}_{\infty}},\\
\label{eq:upperibound-2norm}
    \norm{F}_2 &\le \sum_{k=-h_1}^{h_2} \sum_{l=-w_1}^{w_2} \norm{\Tm_{k,l}}_2,
\end{align}
where $\Rm \in \RR^{hc_{out} \times wc_{in}}$ is a $c_{out} \times c_{in}$ block matrix with $(c,d)$-th block being $\Km_{c,d,:,:} \in \RR^{h \times w}$ and $\Lm \in \RR^{wc_{out} \times hc_{in}}$ is a $c_{out} \times c_{in}$ block matrix with $(c,d)$-th block being $\Km_{c,d,:,:}^\T \in \RR^{w \times h}$.
\end{theorem}

In Theorem \ref{theorem:bounding-norm}, the first upper bound \eqref{eq:upper-bound-F-2norm} is identical to that in \citet{Singla2019}, however the derivation here is different as we directly work on $F$, while the bounds in \citet{Singla2019} is for the circulant approximation. This reveals that, with respect to spectral norm upper bounds, it may be not necessary to distinguish circular from linear convolutional layers.
With respect to computational complexity, the first bound \eqref{eq:upper-bound-F-2norm} requires to compute two spectral norms with sizes $hc_{out} \times wc_{in}$ and $wc_{out} \times hc_{in}$ respectively. The complexity of the second bound \eqref{eq:upperibound-oneinfnorm} depends on the sampling complexity of $\omegav$, which usually takes as $n^2$. As such, it requires to compute $n^2$ times of $\ell_1$ and $\ell_{\infty}$ norms with size $c_{out} \times c_{in}$.  The third bound \eqref{eq:upperibound-2norm} requires to compute $hw$ spectral norms with size $c_{out} \times c_{in}$.

\section{Experiments}
\label{sec:experiments-main}
\subsection{Singular Value Approximation}
To verify the singular value approximation in Section \ref{sec:spectral}, we conduct experiments with respect to four different methods on singular values calculation. The weights of filters are extracted from either the pre-trained networks, e.g., GoogLeNet \citet{szegedy2015going}, with ImageNet dataset or from the training process of ResNet-20 \citet{ResNet} on CIFAR-10 dataset. More experimental results using randomly generated weights and weights from pre-trained networks are given in Section \ref{sec:experiments-sva}.
\begin{itemize}
\vspace{-3pt}
    \item Exact Method: A block doubly Toeplitz matrix $\Tm$ is generated from the convolutional filter $\Km$ according to \eqref{eq:T_k_l}. The exact singular values of linear convolutional layers are computed by applying SVD to $\Tm$ directly. %
    \item Circular Approximation: A block doubly circulant matrix $\Cm$ is constructed according to \eqref{eq:Cm_k}-\eqref{eq:C_k_l}. The singular values are computed by applying SVD on $\Cm$ directly. %
    \item Uniform Sampling: The block diagonal matrices $\Bm_{j_1,j_2}$ is produced by uniformly sampling the spectral density matrix $F(\omega_1,\omega_2)$ with sampling grids $(\omega_1,\omega_2)=(-\pi+\frac{2\pi j_1}{n},-\pi+\frac{2\pi j_1}{n})$ for all $j_1, j_2 \in [n]$. The singular values are obtained by collecting all singular values of $\{\Bm_{j_1,j_2}\}_{j_1,j_2=1}^n$. This corresponds to lines 1-10 in Algorithm \ref{alg:quantile}.
    \item Quantile Interpolation: The singular values obtained from uniform sampling are arranged for each $1\le j \le \min\{c_{in},c_{out}\}$ in descending order. By quantile estimation using linear interpolation methods, the singular values are recomputed by selecting properly shifted sampling grids as outlined in Algorithm \ref{alg:quantile}.
    \vspace{-3pt}
\end{itemize}
The experiments are conducted on MATLAB 2020a, which is more friendly to matrix computation. For simplicity, we set $h_1=h_2$ and $w_1=w_2$, and the input size per channel is set to $10 \times 10$. Fig.~\ref{fig:Fig-1} presents the $(i-1)n+1$-th largest singular values ($i \in [n]$) of four methods with four different filter sizes. 
The first two filters are from the pre-trained GoogLeNet, and the last two are from the training process of ResNet-20. 
It can be observed that (1) both circular approximation and uniform sampling have identical singular values for different filter sizes, (2) quantile interpolation improves accuracy of the singular values over the circular approximation with negligible extra running time (see Section \ref{sec:experiments-sva}), and (3) during the training process the improvement of the largest singular value approximation is dominant, while for the well-trained networks, the improvement is mainly due to that on smaller singular values. This might be attributed to implicit regularization during training. 
\begin{figure}[t]
\vskip 0.1in
\hspace{-12pt}
{\includegraphics[width=0.27\columnwidth]{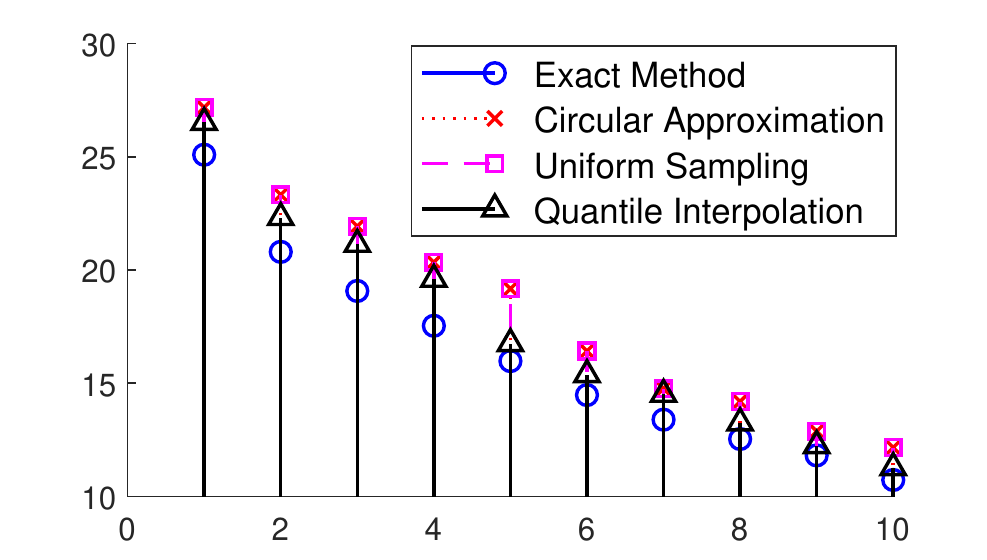}}
\hspace{-12pt}
{\includegraphics[width=0.27\columnwidth]{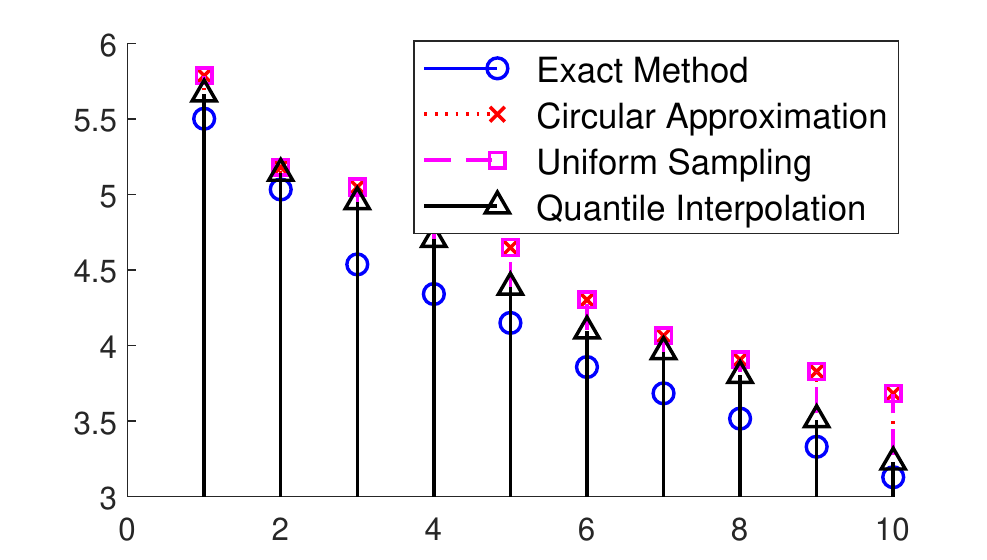}}
\hspace{-12pt}
{\includegraphics[width=0.27\columnwidth]{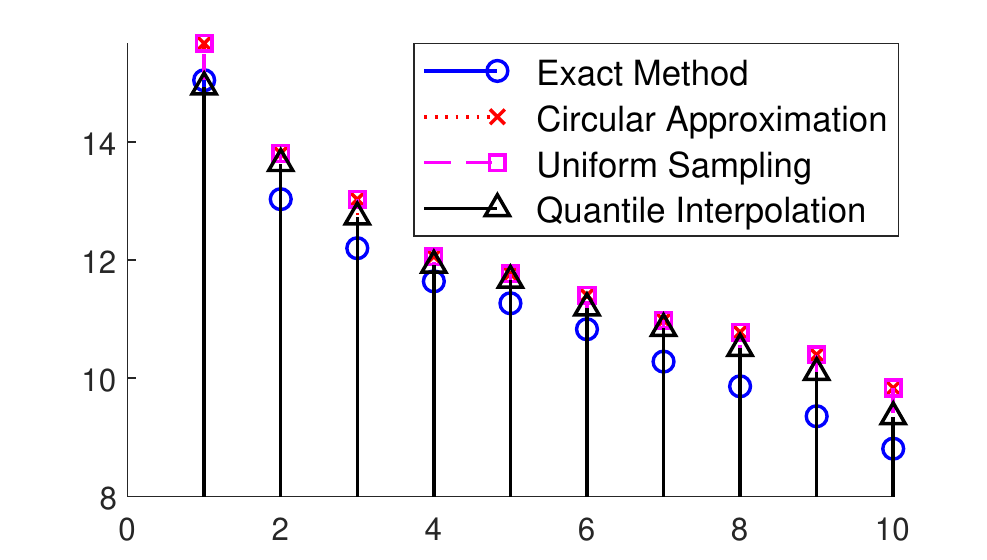}}
\hspace{-12pt}
{\includegraphics[width=0.27\columnwidth]{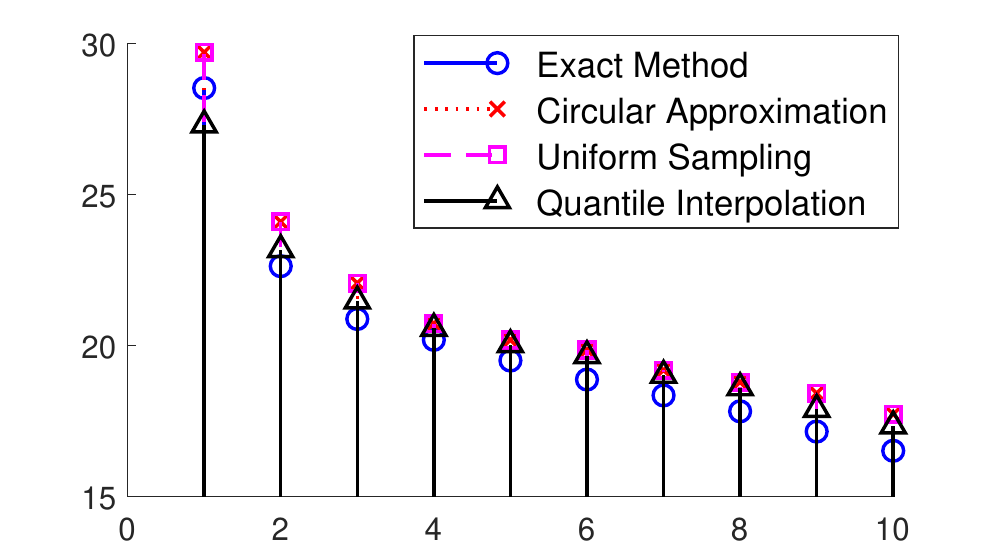}}
\hspace{-12pt}
\begin{center}
\vspace{-0.1in}
\caption{Exact and approximated singular values of linear convolutional layers arranged in descending order. Input size per channel is set to $10 \times 10$. For illustration, only 10 singular values are plotted. Four types of convolutional filters are considered from left to right with sizes $64 \times 3 \times 7 \times 7$ (pre-trained GoogLeNet conv1),  $32 \times 16 \times 5 \times 5$ (pre-trained GoogLeNet inception), $16 \times 3 \times 3 \times 3$ (ResNet-20 conv1 after 10 training epochs), and $16 \times 3 \times 3 \times 3$ (ResNet-20 conv1 after 100 training epochs), respectively. }
\label{fig:Fig-1}
\end{center}
\vskip -0.35in
\end{figure}

\subsection{Spectral Norm Bounding}
To verify the accuracy and running time of different spectral norm bounds, we conduct experiments on the pre-trained ResNet-18 model with ImageNet dataset on MATLAB 2020a on HP EliteBook. For the accuracy, we use the circular approximation as the reference and present the ratios to it. Table~\ref{sample-table} summarizes the results for different filters, where the numbers ``$a/b$'' read as $a$ times of the circular approximation in accuracy and $b$ milliseconds (ms) in running time.
We observe that (1) the first bound \eqref{eq:upper-bound-F-2norm} usually has the best accuracy except for the larger filter size, e.g., $7 \times 7$, while the second bound \eqref{eq:upperibound-oneinfnorm} works better for large filter size; (2) the third bound \eqref{eq:upperibound-2norm} has comparable accuracy as the first one \eqref{eq:upper-bound-F-2norm}, yet accounting for less than 10\% running time of the latter.

\begin{table}[t]
\caption{Comparison of spectral norm bounds ($a/b$: accuracy ratio/running time).}
\label{sample-table}
\vskip -0.1in
\begin{center}
\begin{small}
\begin{sc}
\begin{tabular}{lcccr}
\toprule
Filter size & \eqref{eq:upper-bound-F-2norm} &  \eqref{eq:upperibound-oneinfnorm} &  \eqref{eq:upperibound-2norm}\\
\midrule
$64 \times 3 \times 7 \times 7$    & 3.00/12.84 &  2.14/51.51 & 4.33/1.146 \\
$64 \times 64 \times 3 \times 3$    & 1.63/77.68 &  3.21/54.27 & 2.20/5.427 \\
$128 \times 64 \times 3 \times 3$    & 1.48/155.3 &  3.52/102.3 & 2.10/8.981 \\
$256 \times 256 \times 3 \times 3$    & 1.27/1285 &  4.66/671.7 & 1.56/68.74 \\
$512 \times 256 \times 3 \times 3$    & 1.10/2516 &  4.72/2010 & 1.27/124.6 \\
$512 \times 512 \times 3 \times 3$    & 1.13/7232 &  4.51/3215 & 1.26/288.5 \\
\bottomrule
\end{tabular}
\end{sc}
\end{small}
\end{center}
\vskip -0.25in
\end{table}

Following the same setting as \citet{Singla2019}, we conduct experiments for generalization using spectral norm bounds \eqref{eq:upper-bound-F-2norm} and \eqref{eq:upperibound-2norm} as regularizers. The sum of spectral norm bounds of all convolutional and fully-connected layers are used during training. The bound \eqref{eq:upper-bound-F-2norm} has been already evaluated in \citet{Singla2019}, so our focus will be placed on the evaluation of \eqref{eq:upperibound-2norm} by replacing the matrices of interest in the forward and backward propagation. We test the accuracy of CIFAR-10 dataset on ResNet-20 model with no weight decay and $\beta=0.0014$ as in \citet{Singla2019}. We observe an improvement of 0.8\% over the non-regularization case (i.e., $\beta=0$) using \eqref{eq:upperibound-2norm}, which is slightly worse (0.3\%) than that of \eqref{eq:upper-bound-F-2norm} after 150 training epochs. The learning rate is initialized as 0.1 and changed to 0.01 after 100 epochs. %
Although test accuracy does matter in generalization, we argue that the regularizer \eqref{eq:upperibound-2norm} would be more preferable as it substantially reduces the computational complexity (with more than 30\% running time saving) at the expense of slight performance degradation.
The detailed experimental setup and more results are given in Section \ref{sec:experiments-snr}.

\section{Conclusion}
In this paper, we proposed to use spectral density matrices to represent the linear convolutional layers in CNNs, for which the linear transformation matrices are block doubly Toeplitz matrices constructed from the convolutional filters.
By doing so, spectral analysis of linear convolutional layers can be alternatively done on the corresponding spectral density matrices. Such a spectral representation has been demonstrated to be useful in singular value approximation and spectral norm bounding. 
In particular, spectral norm bounds derived from the spectral density matrix can be used as regularizers to enhance generalization performance with substantially reduced computational complexity.
This spectral representation is expected to offer a different approach to understand linear convolutional layers, through analyzing the spectral density matrices associated to linear transformation.

\newpage
\section*{\Large Appendix}

\setcounter{lemma}{0}
\setcounter{theorem}{0}

The Appendix is organized as follows. In Section \ref{sec:prelim}, we present the notations used in this paper and the definitions of Toeplitz and circulant matrices. The detailed proofs of the key lemmas and theorems in the main text are detailed in Section \ref{sec:proofs}, together with some comments. For convenience, we restate these lemmas and theorems here. In Section \ref{sec:discussion}, the extensions are discussed with respect to larger stride size, higher dimensional linear convolution, and multiple convolutional layers in linear networks. We also present in Section \ref{sec:experiments} the detailed experimental setups in the main text as well as more results to demonstrate the applicability and practical usefulness of our methods in practical CNN models, e.g., VGG, ResNets.

\section{Notations and Preliminaries}
\label{sec:prelim}
\subsection{Notations and Definitions}
For two integers $m$ and $n$ satisfying $m<n$, define $[m] \defeq \{1,2,\dots,m\}$, $n-[m] \defeq \{n-1,n-2,\dots,n-m\}$, and $[m:n] \defeq \{m,m+1,\dots,n\}$. $x \in [a,b]$ is such that $a\le x \le b$. $\jmath$ is the imaginary unit. 

Denote by $a$, $\av$, $\Am$ scalars, vectors, and matrices/tensors, respectively. 
$\Am^{\T}$ and $\Am^{\H}$ represent matrix transpose and Hermitian transpose of $\Am$, respectively. A complex-valued matrix $\Am$ is Hermitian if $\Am=\Am^{\H}$. If $\Am$ is real-valued, $\Am$ is Hermitian is equivalent to $\Am$ is symmetric, i.e., $\Am=\Am^{\T}$.
We denote by $\mathrm{blkdiag}(\Am, \Bm, \dots)$ a block diagonal matrix with diagonal blocks being $\Am, \Bm, \dots$, and by $\circ(a,b,\dots)$ a circulant matrix with elements in the first row being $a,b,\dots$. Likewise, $\circ(\Am,\Bm,\dots)$ is the block circulant matrix with first row blocks being $\Am, \Bm, \dots$. An $n \times n$ matrix $\Fm_n$ is called Discrete Fourier Transform (DFT) matrix, where $[\Fm_n]_{ik}=\frac{1}{\sqrt{n}} e^{-\jmath 2 \pi (i-1)(k-1)/n}$ for $i,k \in [n]$. $\Id_n$ is the $n \times n$ identity matrix. For a tensor $\Am$, $\vec(\Am)$ denotes the vectorized version of $\Am$, and for a 4-order tensor $\Am$, $\Am_{i,j,k,l}$ is used to index its elements.

Denote by $\otimes$ the Kronecker product between two matrices. For a scalar $k$, it holds $\Am \otimes (k\Bm)=k(\Am \otimes \Bm)$ and $\Am \otimes (\sum_i \Bm_i)=\sum_i \Am \otimes \Bm_i$. For two matrices $\Am$ and $\Bm$, $\Am \otimes \Bm$ is permutation equivalent to $\Bm \otimes \Am$, i.e., there exist permutation matrices $\Pim_1$ and $\Pim_2$ such that $\Bm \otimes \Am= \Pim_1 (\Am \otimes \Bm) \Pim_2$.

A matrix-valued function $F:[a,b]^k \mapsto \CC^{m \times n}$ is such that $F(\xv) \in \CC^{m \times n}$ for $\xv \in [a,b]^k$. $F$ is Lebesgue measurable (resp. bounded, continuous) in $[a,b]^k$ if each of its element $F_{ij}$ is Lebesgue measurable (resp. bounded, continuous) in $[a,b]^k$. $F \in \Lc^2([-\pi,\pi]^2)$ means $\Norm{F} \defeq \frac{1}{(2\pi)^2} \int_{-\pi}^{\pi} \int_{-\pi}^{\pi} \Norm{F} d\omega_1 d \omega_2 < +\infty$.

For a matrix $\Am=(a_{ij})_{i,j=1}^{m,n}$ with $\rank(\Am)=r$, we denote by $\{\sigma_j(\Am)\}_{j}$ the collection of singular values of $\Am$ arranged in descending order, i.e., $\sigma_1(\Am) \ge \sigma_2(\Am)\ge \dots \ge \sigma_r(\Am)$.
The norm $\norm{\Am}_2 \defeq \sigma_1(\Am)$ is called spectral norm.
The Schatten $p$-norm is defined as $\norm{\Am}_p \defeq (\sum_{j=1}^r \sigma_j^p(\Am))^{\frac{1}{p}}$. When $p=2$, it coincides with Frobenius norm $\normf{\Am}\defeq \sqrt{\sum_{i=1}^m \sum_{j=1}^n \Abs{a_{ij}}}= \sqrt{\sum_{j=1}^r \sigma_j^2(\Am)}$. %
The matrix $\ell_1$ and $\ell_{\infty}$ norms are defined as $\norm{\Am}_{1}\defeq \max_{j} \sum_{i=1}^m \abs{a_{ij}}$ and $\norm{\Am}_{\infty}\defeq \max_{i} \sum_{j=1}^n \abs{a_{ij}}$, respectively. 
$\abs{a}$ is the absolute value or modulus of a scalar $a$.

\subsection{Toeplitz and Circulant Matrices}
\label{sec:prelim-toe}
A Toeplitz matrix $\Tm = [t_{i-j}]_{i,j=1}^n$ is an $n \times n$ matrix for which the entries come from a sequence $\{t_k, k=0, \pm1, \pm2, \dots,\pm(n-1)\}$. 
A circulant matrix is a special Toeplitz matrix, where $\Cm=[t_{(i-j) \mod n}]_{i,j=1}^n$. That is, $t_{-k}=t_{n-k}$ for $k=1,2,\dots,n-1$. We denote the circulant matrix by $\Cm=\circ(t_0,t_{-1},\dots,t_{-(n-1)})$ using its first row, where the rest rows are cyclic shift of the first row with $n$ times.

An $m \times m$ block Toeplitz matrix $\Bm=[\Am_{i-j}]_{i,j=1}^m \in \CC^{mp \times mq}$ is a Toeplitz matrix with each element being a $p \times q$ matrix.
Similarly, the block circulant matrix $\Cm$ is such that $\Cm=[\Am_{(i-j)\!\!\mod m}]_{i,j=1}^m$ with $0\!\!\mod m = m\!\!\mod m = 0$. That is, $\Am_{-k}=\Am_{m-k}$ for $k=1,2,\dots,m-1$, such that $\Cm=\circ(\Am_0,\Am_{-1},\dots,\Am_{-(m-1)})$ and the rest row blocks are block-wise cyclic shift of the first row block.

When $\{\Am_k, k=0,\pm1,\dots,\pm(m-1)\}$ are also $n \times n$ Toeplitz/circulant matrices, $\Bm$ is a block Toeplitz/circulant matrix with Toeplitz/circulant blocks, which is also known as doubly Toeplitz/circulant matrix.

A banded (block) Toeplitz matrix is a special Toeplitz matrix $\Tm$ [resp.~$\Bm$] such that $t_k=0$ [resp. $\Am_k=\zero$] when $k>r$ or $k<-s$ for some $1<r,s<n$ [resp.~$1<r,s<m$].

For the general case with multiple-input and multiple-output channels, the linear transformation of convolutional layers in CNNs can be represented as a $c_{out}\times c_{in}$ block matrix, i.e.,
\begin{align} \label{eq:M-1}
\Am = \begin{bmatrix}
\Am_{1,1} & \Am_{1,2} & \dots & \Am_{1,c_{in}}\\
\Am_{2,1} & \Am_{2,2} & \dots & \Am_{2,c_{in}}\\
\vdots & \vdots & & \vdots\\
\Am_{c_{out},1} & \Am_{c_{out},2} & \dots & \Am_{c_{out},c_{in}}
\end{bmatrix}.
\end{align}
Each block $\Am_{c,d}$ is a banded block Toeplitz matrix with
\begin{align} \label{eq:M-2} 
\Am_{c,d}=\begin{bmatrix}
\Am_0^{c,d} & \cdots & \Am_{-h_1}^{c,d} & 0 &  \dots & 0\\
\vdots & \Am_0^{c,d} & \ddots & \ddots & \ddots & \vdots\\
\Am_{h_2}^{c,d} & \ddots & \ddots & \ddots & \ddots & 0\\
0 & \ddots & \ddots & \ddots & \ddots & \Am_{-h_1}^{c,d}\\
\vdots & \ddots & \ddots & \ddots & \Am_0^{c,d} & \vdots\\
0 & \cdots & 0 & \Am_{h_2}^{c,d} & \cdots & \Am_{0}^{c,d}
\end{bmatrix}
\end{align}
where $h_1,h_2$ depend on the size of padding in height subject to $h=h_1+h_2+1$. Each block $\Am_{k}^{c,d}$ is still a banded Toeplitz matrix with
\begin{align} \label{eq:M-3}
\Am_{k}^{c,d}=\begin{bmatrix}
a_{k,0}^{c,d} & \cdots &  a_{k,-w_1}^{c,d} & 0 & \cdots & 0\\
\vdots & a_{k,0}^{c,d} & \ddots & \ddots  & \ddots & \vdots \\
 a_{k,w_2}^{c,d} & \ddots & \ddots & \ddots & \ddots & 0\\
 0 & \ddots & \ddots & \ddots & \ddots &  a_{k,-w_1}^{c,d}\\
 \vdots & \ddots & \ddots & \ddots & a_{k,0}^{c,d} &  \vdots\\
0 & \cdots & 0 & a_{k,w_2}^{c,d} & \cdots & a_{k,0}^{c,d}
\end{bmatrix}
\end{align}
where $w_1,w_2$ subject to $w_1+w_2+1=w$ that are determined by the size of padding in width. The elements in $\Am_{k}^{c,d}$ are weights in the filter [cf. \eqref{eq:M4}].

As stated in Section \ref{sec:section2.2} in the main text, the linear transformation matrix $\Am$ can be alternatively represented by doubly block Toeplitz matrix $\Tm$ without change of spectrum.

The alternative representation $\Tm \in \RR^{c_{out}n^2 \times c_{in} n^2}$ is a doubly block Toeplitz matrix
\begin{align} \label{eq:alter-repre}
\Tm=\begin{bmatrix}
\Tm_0 & \cdots & \Tm_{-h_1} & 0 &  \dots & 0\\
\vdots & \Tm_0 & \ddots & \ddots & \ddots & \vdots\\
\Tm_{h_2} & \ddots & \ddots & \ddots & \ddots & 0\\
0 & \ddots & \ddots & \ddots & \ddots & \Tm_{-h_1}\\
\vdots & \ddots & \ddots & \ddots & \Tm_0 & \vdots\\
0 & \cdots & 0 & \Tm_{h_2} & \cdots & \Tm_{0}
\end{bmatrix}
\end{align}
with each block $\Tm_k$ for all $k\in [-h_1:h_2]$ being still a block Toeplitz matrix 
\begin{align}
\Tm_k=\begin{bmatrix}
\Tm_{k,0} & \cdots & \Tm_{k,-w_1} & 0 &  \dots & 0\\
\vdots & \Tm_{k,0} & \ddots & \ddots & \ddots & \vdots\\
\Tm_{k,w_2} & \ddots & \ddots & \ddots & \ddots & 0\\
0 & \ddots & \ddots & \ddots & \ddots & \Tm_{k,-w_1}\\
\vdots & \ddots & \ddots & \ddots & \Tm_{k,0} & \vdots\\
0 & \cdots & 0 & \Tm_{k,w_2} & \cdots & \Tm_{k,0}
\end{bmatrix} 
\end{align}
where each block $\Tm_{k,l} \in \RR^{c_{out} \times c_{in}}$ with $l\in [-w_1:w_2]$ is
\begin{align}
\Tm_{k,l}=\begin{bmatrix}
t_{1,1}^{k,l} & t_{1,2}^{k,l} & \cdots & t_{1,c_{in}}^{k,l}\\
t_{2,1}^{k,l} & t_{2,2}^{k,l} & \cdots & t_{2,c_{in}}^{k,l}\\
\vdots & \ddots & \ddots & \vdots \\
t_{c_{out},1}^{k,l} & t_{c_{out},2}^{k,l} & \cdots & t_{c_{out},c_{in}}^{k,l}
\end{bmatrix}.
\end{align}
Each element of $\Tm_{k,l}$ comes from the weight of the filter $\Km$ [cf. \eqref{eq:T_k_l}].

\section{Proofs of Main Theorems}
\label{sec:proofs}

\subsection{Proof of Lemmas}
\begin{lemma} 
$\{\sigma_j(\Tm), \ \forall j\}=\{\sigma_j(\Am), \ \forall j\}$.
\end{lemma}
\begin{proof}
While the following proof is dedicated to the banded Toeplitz matrices, it can be straightforwardly extended to any Toeplitz matrix without loss of generality.

Let $\ev_i$ be $i$-th column of identity matrix and $\Em_{i,j}=\ev_i \ev_j^{\T}$ be a $c_{out} \times c_{in}$ matrix with only the $(i,j)$-th element being 1 and 0 elsewhere. Define $\Pm_k$ as an $n \times n$ matrix with $[\Pm_k]_{i,j}=1$ if $i-j=k$ and 0 otherwise. Thus, the original linear transformation matrix $\Am$ can be represented as
\begin{align}
    \Am &= \sum_{c=1}^{c_{out}}  \sum_{d=1}^{c_{in}} \Em_{c,d} \otimes \Am_{c,d} \\
    &= \sum_{c=1}^{c_{out}}  \sum_{d=1}^{c_{in}} \Em_{c,d} \otimes (\sum_{k=-h_1}^{h_2} \Pm_k \otimes \Am_k^{c,d}) \\
    &= \sum_{c=1}^{c_{out}}  \sum_{d=1}^{c_{in}} \Em_{c,d} \otimes (\sum_{k=-h_1}^{h_2} \Pm_k \otimes (\sum_{l=-w_1}^{w_2} \Pm_l \otimes a_{k,l}^{c,d}))  \\
    &= \sum_{c=1}^{c_{out}} \sum_{d=1}^{c_{in}} \sum_{k=-h_1}^{h_2} \sum_{l=-w_1}^{w_2} a_{k,l}^{c,d} \Em_{c,d} \otimes \Pm_k \otimes \Pm_l 
\end{align}
where the last equality is because $a_{k,l}^{c,d}$ is a scalar.
The alternative one $\Tm$ can be represented as
\begin{align}
    \Tm&=\sum_{k=-h_1}^{h_2}\Pm_k \otimes \Tm_{k} \\
    &= \sum_{k=-h_1}^{h_2} \Pm_k \otimes (\sum_{l=-w_1}^{w_2} \Pm_l \otimes \Tm_{k,l})\\
    &= \sum_{k=-h_1}^{h_2}\Pm_k \otimes (\sum_{l=-w_1}^{w_2} \Pm_l \otimes (\sum_{c=1}^{c_{out}} \sum_{d=1}^{c_{in}} t_{c,d}^{k,l} \Em_{c,d} ))\\
    &= \sum_{c=1}^{c_{out}}  \sum_{d=1}^{c_{in}} \sum_{k=-h_1}^{h_2} \sum_{l=-w_1}^{w_2}  t_{c,d}^{k,l} \Pm_k \otimes \Pm_l \otimes \Em_{c,d} 
\end{align}
where the last equality is because $t_{c,d}^{k,l}$ is a scalar.

According to \cite{henderson1981vec}, $\Pm_k \otimes \Pm_l \otimes \Em_{c,d}$ is permutation equivalent to $\Em_{c,d} \otimes\Pm_k \otimes \Pm_l$, for which there exist two permutation matrices $\Pim_1$ and $\Pim_2$, such that $\Pm_k \otimes \Pm_l \otimes \Em_{c,d} = \Pim_1 (\Em_{c,d} \otimes\Pm_k \otimes \Pm_l) \Pim_2$. Given the fact that $a_{k,l}^{c,d}=t_{c,d}^{k,l}$, it follows that
\begin{align}
   \Tm = \Pim_1 \Am \Pim_2. 
\end{align}
Because permutation matrices are also orthogonal matrices, and thus unitary, $\Tm$ and $\Am$ have an identical set of singular values. This completes the proof.
\end{proof}

\begin{lemma}
$\{\sigma_j(\Cm), \ \forall j\}=\{\sigma_j(\Cm(\Am)), \ \forall j\}$.
\end{lemma}
\begin{proof}
The proof is similar to that of Lemma \ref{lemma:toe-permutation-invariant} and thus omitted. The only difference is that, for the representation of $n \times n$ circulant matrices, we have $[\Pm_k]_{i,j}=1$ if $(i-j) \!\! \mod n = k$ and 0 otherwise.
\end{proof}

\begin{lemma} 
The linear transformation matrix $\Cm$ can be block-diagonalized as
\begin{align} 
\Cm = (\Fm_n \otimes \Fm_n \otimes \Id_{c_{out}}) \mathrm{blkdiag} (\Bm_{1,1},\Bm_{1,2},\dots, \Bm_{1,n},  \Bm_{2,1}, \dots \Bm_{n,n}) (\Fm_n \otimes \Fm_n \otimes \Id_{c_{in}})^{\H}    
\end{align}
where $(\Fm_n \otimes \Fm_n \otimes \Id_{c_{out}})$ and $(\Fm_n \otimes \Fm_n \otimes \Id_{c_{in}})$ are unitary matrices.
Thus, the singular values of $\Cm$ are the collection of singular values of 
$
    \{\Bm_{i,k}\}_{i,k=1}^n
$
where 
\begin{align}
\Bm_{i,k}= \sum_{p=0}^{n-1} \sum_{q=0}^{n-1} \Cm_{p,q} e^{-\jmath 2 \pi \frac{ p(i-1)+q(k-1)}{n}}
\end{align}
with $\Cm_{p,q}$ defined in \eqref{eq:C_k_l}.
\end{lemma}
\begin{proof}
By extending Lemma 5.1 in \cite{book_block_toe} from block circulant matrices to doubly block circulant matrices, we conclude that
the linear transformation matrix $\Cm$ can be block-diagonalized as
\begin{align} 
\Cm = (\Fm_n \otimes \Fm_n \otimes \Id_{c_{out}}) \mathrm{blkdiag} (\Bm_{1,1},\Bm_{1,2},\dots, \Bm_{1,n},  \Bm_{2,1}, \dots \Bm_{n,n}) (\Fm_n \otimes \Fm_n \otimes \Id_{c_{in}})^{\H}    
\end{align}
where both  $(\Fm_n \otimes \Fm_n \otimes \Id_{c_{out}})$ and $(\Fm_n \otimes \Fm_n \otimes \Id_{c_{in}})$ are unitary matrices. As such, the singular values of $\Cm$ are the collection of singular values of $n^2$ matrices $\{\Bm_{i,k}\}_{i,k=1}^{n}$.

By Lemma 5.1 in \cite{book_block_toe}, for each $i,k \in [n]$, we compute $\Bm_{i,k} \in \CC^{c_{out} \times c_{in}}$ by
\begin{align}
\Bm_{i,k}= \sum_{p=0}^{n-1} \sum_{q=0}^{n-1} \Cm_{p,q} e^{-\jmath 2 \pi \frac{ p(i-1)+q(k-1)}{n}}.
\end{align}
The singular values of $\Bm_{i,k}$ can be therefore obtained by applying off-the-shelf singular-value decomposition algorithms.
\end{proof}

\subsection{Proof of Theorem 1}
\begin{theorem} 
Given a block doubly Toeplitz matrix $\Tm \in \CC^{rn^2\times sn^2}$, let a complex matrix-valued Lebesgue-measurable function $F:[-\pi,\pi]^2 \mapsto \mathbb{C}^{r \times s}$ be the generating function such that 
\begin{align} \label{eq:App-T-F}
    \Tm_{k,l} = \frac{1}{(2\pi)^2} \int_{-\pi}^{\pi} \int_{-\pi}^{\pi} F(\omega_1,\omega_2) e^{-\jmath (k \omega_1 + l \omega_2)} d\omega_1 d\omega_2.
\end{align}
It follows that, for any continuous function $\Phi$ with compact support in $\mathbb{R}$, we have
\begin{align} 
    \MoveEqLeft  \lim_{n \to \infty} \frac{1}{n^2} \sum_{j=1}^{\min\{r,s\}n^2} \Phi(\sigma_j(\Tm))
    = \frac{1}{(2\pi)^2} \int_{-\pi}^{\pi} \int_{-\pi}^{\pi} \sum_{j=1}^{\min\{r,s\}} \Phi(\sigma_j(F(\omega_1,\omega_2))) d\omega_1 d\omega_2,
\end{align}
for which $\Tm$ is said to be equally distributed as $F(\omega_1,\omega_2)$ with respect to singular values, i.e., $\Tm \sim_{\sigma} F$. Specifically, for linear convolutional layers, the linear transformation matrix $\Tm$ has doubly banded structures, so that the generating function can be explicitly written as 
\begin{align}
    F(\omega_1,\omega_2)=\sum_{k=-h_1}^{h_2} \sum_{l=-w_1}^{w_2} \Tm_{k,l} e^{\jmath  (k \omega_1 + l \omega_2)},
\end{align}
which is also referred to as the spectral density matrix of $\Tm$.
\end{theorem}
\begin{remark}
Theorem \ref{thm:szg-limit} is a generalization of the celebrated Szeg\"o Theorem \cite{Gray1972}, which deals with real scalar-valued generating functions $F: [-\pi,\pi] \mapsto \RR$ that correspond to Hermitian Toeplitz matrices. It was extended to non-Hermitian matrices \cite{avram1988bilinear,parter1986distribution}, block Toeplitz matrices \cite{Tilli1998}, and multi-level Toeplitz matrices \cite{Tyrtyshnikov1996,voois1996theorem}.
The linear transformation matrix $\Tm$ is an asymmetric real matrix and hence non-Hermitian, with doubly block Toeplitz structure, which corresponds to a complex matrix-valued generating function $F: [-\pi,\pi]^2 \mapsto \CC^{s \times r}$.
In particular, when $s=r=1$, Theorem \ref{thm:szg-limit} reduces to single-channel 2D convolutional layers, for which
$\Tm \sim_{\sigma} \abs{F(\omega_1,\omega_2)}$.
When it comes to signal-channel 1D convolutional layer, Theorem \ref{thm:szg-limit} indicates $\Tm \sim_{\sigma} \abs{F(\omega)}$. 
\end{remark}
\begin{proof}
The proof is an extension of 
those in \cite{voois1996theorem,Tilli1998,miranda2000asymptotic,Tyrtyshnikov1996} that consider block Toeplitz matrices or doubly Toeplitz matrices.
The main proof technique is to relate Toeplitz matrices to their circulant counterpart, which has been shown efficient in many similar settings. This technique is also applied here.
In particular, we follow the footsteps of \cite{Tilli1998,miranda2000asymptotic} to extend the proofs to non-Hermitian block doubly Toeplitz matrices $\Tm$, by relating to the block doubly circulant matrices $\Cm$.

First, we show that both $\Tm$ and $\Cm$ have the same asymptotic singular values distribution as $n \to \infty$.
As $\Cm$ is constructed from $\Tm$ and both of them are banded matrices, by Lemma \ref{lemma:proof-thm1} below, it can be easily verified that
\begin{align}
    \lim_{n \to \infty} \frac{1}{n^2} \Normf{\Tm -\Cm}=0
\end{align}
as the values of the elements in $\Cm$ and $\Tm$ are upper-bounded, and the total number of different elements between $\Cm$ and $\Tm$ does not scale as $n^2$. 
According to Chapter 2 in \cite{gray2006toeplitz} , it follows that $\Cm$ and $\Tm$ are asymptotically equivalent. 

Let us introduce two Hermitian matrices 
\begin{align}
\tilde{\Cm} = \begin{bmatrix} \zero & \Cm \\ \Cm^H & \zero \end{bmatrix}, \qquad \tilde{\Tm} = \begin{bmatrix} \zero & \Tm \\ \Tm^H & \zero \end{bmatrix}.
\end{align}
It follows that $\tilde{\Cm}$ and $\tilde{\Tm}$ are asymptotically equivalent as well. It is worth noting that the sets of eigenvalues of $\tilde{\Cm}$ and $\tilde{\Tm}$ are exactly the respective sets of singular values of ${\Cm}$ and ${\Tm}$, according to Theorem 7.3.3 in \cite{horn2012matrix}. Thus, according to Theorem 2.1 in \cite{Gray1972}, we have
\begin{align}
    \lim_{n \to \infty} \frac{1}{\min\{r,s\}n^2} \sum_{j=1}^{\min\{r,s\}n^2} (\sigma_j(\Tm))^p = \lim_{n \to \infty} \frac{1}{\min\{r,s\}n^2} \sum_{j=1}^{\min\{r,s\}n^2} (\sigma_j(\Cm))^p
\end{align}
for any positive integer $p$. By Stone-Weierstrass theorem \cite{Gray1972}, it follows that, any continuous function $\Phi(\cdot)$ with compact support, there exists a set of polynomials that uniformly converges to it. Thus, we have 
\begin{align} \label{eq:equally-distributed}
    \lim_{n \to \infty} \frac{1}{\min\{r,s\}n^2} \sum_{j=1}^{\min\{r,s\}n^2} (\Phi(\sigma_j(\Tm)) -\Phi(\sigma_j(\Cm)))=0.
\end{align}

Second, we show that the singular value distribution of the block doubly circulant matrix $\Cm$ converges to that of the generating function $F$.
The doubly circulant matrix $\Cm$ can be block-diagonalized as
\begin{align} 
\Cm = (\Fm_n \otimes \Fm_n \otimes \Id_{c_{out}}) &\mathrm{blkdiag} (\Bm_{1,1},\Bm_{1,2},\dots, \Bm_{1,n},  \Bm_{2,1}, \dots \Bm_{n,n})  (\Fm_n \otimes \Fm_n \otimes \Id_{c_{in}})^{\H}    
\end{align}
where by Lemma \ref{lemma:toe-equal-cir}  and \eqref{eq:C_k_l} in the main text, we have 
\begin{align}
\Bm_{j_1,j_2}&= \sum_{p=0}^{n-1} \sum_{q=0}^{n-1} \Cm_{p,q} e^{-\jmath 2 \pi \frac{ p(j_1-1)+q(j_2-1)}{n}}\\
&=\sum_{p=0}^{h_1} \sum_{q=0}^{w_1} \Cm_{p,q} e^{-\jmath 2 \pi \frac{ p(j_1-1)+q(j_2-1)}{n}} 
+ \sum_{p=0}^{h_1} \sum_{q=n-w_2}^{n-1} \Cm_{p,q} e^{-\jmath 2 \pi \frac{ p(j_1-1)+q(j_2-1)}{n}} \\
& \qquad + \sum_{p=n-h_1}^{n-1} \sum_{q=0}^{w_1} \Cm_{p,q} e^{-\jmath 2 \pi \frac{ p(j_1-1)+q(j_2-1)}{n}}
+ \sum_{p=n-w_2}^{n-1} \sum_{q=0}^{w_1} \Cm_{p,q} e^{-\jmath 2 \pi \frac{ p(j_1-1)+q(j_2-1)}{n}}\\
&=\sum_{k=-h_1}^{0} \sum_{l=-w_1}^{0} \Tm_{k,l} e^{\jmath 2 \pi \frac{ k(j_1-1)+l(j_2-1)}{n}} 
+ \sum_{k=-h_1}^0 \sum_{l=1}^{w_2} \Tm_{k,l} e^{\jmath 2 \pi \frac{ k(j_1-1)+(l-n)(j_2-1)}{n}} \\
& \qquad + \sum_{k=1}^{h_2} \sum_{l=-w_1}^{0} \Tm_{k,l} e^{\jmath 2 \pi \frac{ (k-n)(j_1-1)+l(j_2-1)}{n}}
+ \sum_{k=1}^{h_2} \sum_{l=1}^{w_2} \Tm_{k,l} e^{\jmath 2 \pi \frac{ (k-n)(j_1-1)+(l-n)(j_2-1)}{n}}\\
&= \sum_{k=-h_1}^{h_2} \sum_{l=-w_1}^{w_2} \Tm_{k,l} e^{\jmath 2 \pi \frac{ k(j_1-1)+l(j_2-1)}{n}}\\
&= F(\frac{2\pi (j_1-1)}{n},\frac{2\pi (j_2-1)}{n}), \label{eq:B-to-F}
\end{align}
for $j_1,j_2 \in [n]$. 
Consequently, the collection of singular values of block doubly circulant matrix $\Cm$ is the collection of singular values of $F$ over the uniform grids 
\begin{align} 
    \Mc \defeq \left\{ (\omega_1,\omega_2) = \left(-\pi+ \frac{2 \pi j_1}{n}, -\pi + \frac{2 \pi j_2}{n} \right),  \forall\;j_1,j_2 \in [n]-1 \right\}.
\end{align}

As such, for any integer $p \ge 0$, we have
\begin{align}
    \frac{1}{\min\{r,s\}n^2} \sum_{j=1}^{\min\{r,s\}n^2} (\sigma_j(\Cm))^p 
    &= \frac{1}{\min\{r,s\}n^2} \sum_{j=1}^{\min\{r,s\}} \sum_{(\omega_1,\omega_2) \in \Mc} (\sigma_j(F(\omega_1,\omega_2)))^p\\
     &= \frac{1}{\min\{r,s\}}  \sum_{j=1}^{\min\{r,s\}} \frac{1}{n^2} \sum_{(\omega_1,\omega_2) \in \Mc} (\sigma_j(F(\omega_1,\omega_2)))^p\\
     & \stackrel{n \to \infty}{=}  \frac{1}{\min\{r,s\}} \int_{-\pi}^{\pi} \int_{-\pi}^{\pi}  (\sigma_j(F(\omega_1,\omega_2)))^p d \omega_1 d\omega_2
\end{align}
where the last equation is due to the fact that the Riemann sum converges to the integral of the function $(\sigma_j(F(\omega_1,\omega_2)))^p$ over $[-\pi, \pi]^2$, as $n \to \infty$. 

Further, by Stone-Weierstrass theorem \cite{Gray1972}, we have
\begin{align}
    \lim_{n \to \infty} \frac{1}{\min\{r,s\}n^2} \sum_{j=1}^{\min\{r,s\}n^2} \Phi(\sigma_j(\Cm)) = \frac{1}{\min\{r,s\}} \int_{-\pi}^{\pi} \int_{-\pi}^{\pi}  \Phi(\sigma_j(F(\omega_1,\omega_2))) d \omega_1 d\omega_2
\end{align}

Finally, together with \eqref{eq:equally-distributed}, the proof of Theorem 1 is completed.
\end{proof}

\begin{lemma}
\label{lemma:proof-thm1}
Given the banded block doubly Toeplitz and circulant matrices $\Tm$ and $\Cm$, it follows that 
\begin{align}
\norm{\Cm-\Tm}_p^p \le O(n).   
\end{align}
where $\norm{\Am}_p\defeq (\sum_{i=1}^n \sum_{j=1}^n \abs{\Am_{i,j}}^p)^{\frac{1}{p}}$ for $1 \le p < \infty$. When $p=2$, $\norm{\Am}_p$ boils down to the Frobenius norm $\normf{\Am}$.
\end{lemma}
\begin{proof}
Given $\Tm$ and $\Cm$, we define the difference of the $(k,l)$-th block $\Deltam_{k,l} \in \CC^{r \times s}$, where $k \in [-(n-1),(n-1)]$ and $l \in [-(n-1),(n-1)]$ are indices of two levels of Toeplitz and circulant matrices but not the indices of rows and columns, in the following way
\begin{align}
    \Deltam_{k,l} &\defeq [\Cm-\Tm]_{k,l} \\
    &\stackrel{(a)}{=} \sum_{m_1 \in \{-1,0,1\}} \sum_{m_2 \in \{-1,0,1\}}  \Tm_{k+nm_1,l+nm_2} (1 - \delta(m_1,m_2))
\end{align}
where $\delta(m_1,m_2)=1$ if and only if $m_1=m_2=0$, and $(a)$ is due to the banded structure of circulant matrix as in \eqref{eq:C_k_l}.
It can be easily verified that $\Cm-\Tm$ is still a block doubly Toeplitz matrix with blocks  $\{\Deltam_{k,l}\}_{k,l}$.
Thus, we have
\begin{align}
   \MoveEqLeft \norm{\Cm-\Tm}_p^p \notag \\ &\stackrel{(a)}{=} \sum_{k=-(n-1)}^{n-1} \sum_{l=-(n-1)}^{n-1} (n-\abs{k})(n-\abs{l}) \norm{\Deltam_{k.l}}_p^p\\
    &\stackrel{(b)}{\le} \sum_{k=-(n-1)}^{n-1} \sum_{l=-(n-1)}^{n-1}  \sum_{m_1 = -1}^1 \sum_{m_2=-1}^1 (n-\abs{k})(n-\abs{l})  (1 - \delta(m_1,m_2)) \norm{\Tm_{k+nm_1,l+nm_2}}_p^p\\
    &\stackrel{(c)}{=} \sum_{(k,l) \in \Bc_{12}}  (n-\abs{k})(n-\abs{l}) \norm{\Tm_{k,l+n}}_p^p + \sum_{(k,l) \in \Bc_{13}}  (n-\abs{k})(n-\abs{l}) \norm{\Tm_{k,l-n}}_p^p\\
    & \qquad + \sum_{(k,l) \in \Bc_{21}}  (n-\abs{k})(n-\abs{l}) \norm{\Tm_{k+n,l}}_p^p + \sum_{(k,l) \in \Bc_{22}}  (n-\abs{k})(n-\abs{l}) \norm{\Tm_{k+n,l+n}}_p^p \\
    &\qquad + \sum_{(k,l) \in \Bc_{23}}  (n-\abs{k})(n-\abs{l}) \norm{\Tm_{k+n,l-n}}_p^p
    + \sum_{(k,l) \in \Bc_{31}}  (n-\abs{k})(n-\abs{l}) \norm{\Tm_{k-n,l}}_p^p \\
    &\qquad + \sum_{(k,l) \in \Bc_{32}}  (n-\abs{k})(n-\abs{l}) \norm{\Tm_{k-n,l+n}}_p^p + \sum_{(k,l) \in \Bc_{33}}  (n-\abs{k})(n-\abs{l}) \norm{\Tm_{k-n,l-n}}_p^p\\
    &\stackrel{(d)}{\le} hw_1^2C_pn + hw_2^2C_pn + h_1^2wC_pn + h_1^2 w_1^2 C_p + h_1^2w_2^2C_p + h_2^2 wC_pn + h_2^2w_1^2C_p + h_2^2w_2^2C_p \\
     &\stackrel{(e)}{=} an+b 
\end{align}
where $(a)$ is due the definition of the element-wise $p$-norm, $(b)$ is due to the sub-additivity of matrix norms, in $(c)$ we define 
\begin{align}
\Bc_{11}&=\{(k,l):k \in [-h_1,h_2] \text{ and } l \in [-w_1,w_2]\}\\
\Bc_{12}&=\{(k,l):k \in [-h_1,h_2] \text{ and } l \in [-(n-1),-(n-w_1)]\}\\
\Bc_{13}&=\{(k,l):k \in [-h_1,h_2] \text{ and } l \in [(n-w_2),(n-1)]\}\\
\Bc_{21}&=\{(k,l):k \in [-(n-1),-(n-h_1)] \text{ and } l \in [-w_1,w_2]\}\\
\Bc_{22}&=\{(k,l):k \in [-(n-1),-(n-h_1)] \text{ and } l \in [-(n-1),-(n-w_1)]\}\\
\Bc_{23}&=\{(k,l):k \in [-(n-1),-(n-h_1)] \text{ and } l \in [(n-w_2),(n-1)]\}\\
\Bc_{31}&=\{(k,l):k \in [(n-h_2),(n-1)] \text{ and } l \in [-w_1,w_2]\}\\
\Bc_{32}&=\{(k,l):k \in [(n-h_2),(n-1)] \text{ and } l \in [-(n-1),-(n-w_1)]\}\\
\Bc_{33}&=\{(k,l):k \in [(n-h_2),(n-1)] \text{ and } l \in [(n-w_2),(n-1)]\}
\end{align}
for which $\Tm_{k+nm_1,l+nm_2} \ne \zero$ in $\Bc_{11}$ if and only if $m_1=m_2=0$ which invokes $\delta(m_1,m_2)=1$, $(d)$ is due to $\norm{\Tm_{k,l}}_p^p$ is upper-bounded by a constant, say $C_p$ for all $k,l$, and in $(e)$, $a=C_p(h(w_1^2+w_2^2)+(h_1^2+h_2^2)w)$ and $b=C_p(h_1^2+h_2^2)(w_1^2+w_2^2)$. This completes the proof.
\end{proof}

\subsection{Proof of Theorem 2}
\begin{theorem} 
Given $\Tm$ and $\Cm$ as in \eqref{eq:Cm_k}-\eqref{eq:C_k_l}, 
there exists a constant $c_1>0$ such that
\begin{align}
    \lim_{n \to \infty} \frac{1}{n} \sum_{j=1}^{\min\{r,s\}n^2} \abs{\sigma_j(\Tm)-\sigma_j(\Cm)} \le c_1,
\end{align}
where the singular values of $\Cm$ are the collection of singular values of $\{\sigma_j(F(\omega_1,\omega_2))\}_j$ with
\begin{align}
(\omega_1,\omega_2) = (-\pi+ \frac{2 \pi j_1}{n}, -\pi &+ \frac{2 \pi j_2}{n}), \quad \forall j_1,j_2 \in [n]-1.
\end{align}
\end{theorem}
\begin{remark}
It is worth noting that \cite{ZhuTIT2017} dealt with eigenvalues of Hermitian Toeplitz matrices that correspond to real scalar-valued generating functions, for which the Theorem 2 in \cite{ZhuTIT2017} regarding the bounded circular approximation error can not be taken as granted e.g., \cite{sedghi2018the,Singla2019},  to justify the circular approximation of linear convolutional layers whose transformation matrix $\Tm$ is non-Hermitian block doubly Toeplitz matrices. As the proof of \cite{ZhuTIT2017}[Theorem 2] relies on Sturmian Separation Theorem that deals with eigenvalues for Hermitian matrices, %
it is not guaranteed that the difference of individual singular values between non-Hermitian Toeplitz and circulant matrices can be bounded in the same way. %
\end{remark}
\begin{proof}
Given the generating function $F(\omega_1,\omega_2)$ defined in \eqref{eq:theorem-1-F}, we introduce an auxiliary matrix $\Cm(F)$ generated by $F$ in the following form
\begin{align}
    \Cm(F) = (\Fm_n \times \Fm_n \times \Id_{r}) \mathrm{blkdiag} \Big(\{F(\omega_1,\omega_2), (\omega_1,\omega_2) \in \Mc\}\Big) (\Fm_n \times \Fm_n \times \Id_s)^{\H} 
\end{align}
where $\Mc$ is the uniform sampling over $[-\pi,\pi]^2$ defined as \eqref{eq:uniform-samples}.
It can be readily verified that $\Cm(F)$ is also a block doubly circulant matrix, similar to $\Cm$.

First, we show $\Cm(F)$ and $\Cm$ are identical, and thus uniform sampling $F$ yields singular values of $\Cm$.
Denote by $[\Cm(F)]_{p,q} \in \CC^{r \times s}$ the $(p,q)$-th block of $\Cm(F)$, where $p$ and $q$ indicate the indices of the first and second levels of circulant blocks, similar to the definition of $\Cm_{p,q}$ in \eqref{eq:C_k_l}. Therefore, we have
\begin{align}
[\Cm(F)]_{p,q} 
&\stackrel{}{=} \frac{1}{n^2}\sum_{j_1=0}^{n-1} \sum_{j_2=0}^{n-1} F(\frac{2\pi j_1}{n},\frac{2\pi j_2}{n})e^{-\jmath 2 \pi \frac{ pj_1+qj_2}{n}}\\
&\stackrel{}{=} \frac{1}{n^2}\sum_{j_1=0}^{n-1} \sum_{j_2=0}^{n-1} \sum_{k=-h_1}^{h_2} \sum_{l=-w_1}^{w_2} \Tm_{k,l}e^{\jmath \frac{2 \pi}{n} ( (k-p)j_1+(l-q)j_2)}\\
&\stackrel{}{=} \frac{1}{n^2} \sum_{k=-h_1}^{h_2} \sum_{l=-w_1}^{w_2} \Tm_{k,l} \sum_{j_1=0}^{n-1} e^{\jmath \frac{2 \pi j_1}{n} (k-p)} \sum_{j_2=0}^{n-1} e^{\jmath \frac{2 \pi j_2}{n} (l-q)}  \\
&\stackrel{(a)}{=} \sum_{m_1=-\infty}^{\infty} \sum_{m_2=-\infty}^{\infty}  \Tm_{-p+nm_1,-q+nm_2}\\
&\stackrel{(b)}{=} \left\{
\Pmatrix{\Tm_{-p,-q}, & p\in \{0\}\cup[h_1], \; q\in\{0\}\cup[w_1]\\
\Tm_{-p,n-q}, & p\in \{0\}\cup[h_1], \; q \in n-[w_2]\\
\Tm_{n-p,-q}, &q \in n-[h_2], \; q\in\{0\}\cup[w_1]\\
\Tm_{n-p,n-q}, & p \in n-[h_2], \; q \in n-[w_2]\\
\zero,& \text{otherwise}}
\right.\\
&\stackrel{(c)}{=}\Cm_{p,q}
\end{align}
for $p, q \in [n]-1$, where $(a)$ is due to
\begin{align}
   \sum_{j=0}^{n-1} e^{\jmath \frac{2 \pi j}{n} (k-p)} = \left\{ \Pmatrix{n,& (k-p) \!\!\! \mod n = 0 \\ 0, & \text{otherwise}} \right.,
\end{align}
$(b)$ is due to $\Tm_{p,q}=\zero$ if $p \notin [-h_1,h_2]$ or $q \notin [-w_1,w_2]$, and $(c)$ is from \eqref{eq:C_k_l}. 

For each $p,q \in [n]-1$, the $(p,q)$-th blocks of $\Cm(F)$ and $\Cm$ are identical. Thus, we have
\begin{align}
    \Cm(F) = \Cm.
\end{align}

Therefore, by Lemma \ref{lemma:toe-equal-cir}, we conclude that the singular values of $\Cm$ can be given by those of $F(\omega_1,\omega_2)$ with uniform sampling on $[-\pi,\pi]^2$, i.e.,  
\begin{align}
 \left\{\sigma_{j}(F(\omega_1,\omega_2)): (\omega_1,\omega_2) \in \Mc \right\},
\end{align}
where $\Mc$ is the uniform sampling grids defined in \eqref{eq:uniform-samples}.

Second, we show that the accumulated difference of the singular values between $\Cm$ and $\Tm$ is upper-bounded.

By inspecting $\Tm$ and $\Cm$, we find from Lemma \ref{lemma:proof-thm1} that $\Deltam_{k,l}=0$ if and only if $(k,l) \in \Bc_{11}$. The number of rows and columns with indices outside $\Bc_{11}$ scales as $n$.
As such, invoking Theorem 3.1 in \cite{Zizler2002}, we conclude that
\begin{align}
    \sum_{j=1}^{\min\{r,s\}n^2} \abs{\sigma_j(\Tm)-\sigma_j(\Cm)} \le O(n).
\end{align}
Thus, we have
\begin{align}
    \lim_{n \to \infty}   \frac{1}{n}\sum_{j=1}^{\min\{r,s\}n^2} \abs{\sigma_j(\Tm)-\sigma_j(\Cm)}=O(1)
\end{align}
This completes the proof.
\end{proof}

\begin{remark}
The intuition behind is that the number of different elements between two matrices scales as $n$ but not $n^2$ because of the banded structure of $\Cm$ and $\Tm$. Although not rigorously proved, it looks the equality holds with the term $O(1)$ strictly larger than 0, meaning that the circular approximation can be arbitrarily loose as $n$ tends to infinity.
\end{remark}

\subsection{Proof of Theorem 3}
\begin{theorem} 
Let $\phi_j:[-\pi,\pi]^2 \mapsto \RR_+$ be the $j$-th singular value function of $F(\omegav)$ and $\sigma_k^{(j)}(\Tm)$ be $k$-th singular value of $j$-th cluster. There exists a constant $c_2>0$ which only depends on $F(\omegav)$ such that
\begin{align}
   \sup_{u \in (\frac{k-1}{n^2},\frac{k}{n^2}]} \abs{\sigma_k^{(j)}(\Tm)-Q_{\phi_j}(u)} &\le \frac{c_2}{n}, \quad
     \forall 1\le k \le n^2, \; 1 \le j \le \min\{r,s\}
\end{align}
where
\begin{align}
    Q_{\phi_j}(u)&=\inf\{v \in \RR: u \le G_{\phi_j}(v)\}\\
    G_{\phi_j}(v)&=\frac{1}{(2\pi)^2}\mu\{\omegav \in [-\pi,\pi]^2: \phi_j(\omegav)\le v\}
\end{align}
are quantile and cumulative distribution functions for $\phi_j(\omegav)$, respectively, and $\mu$ is Lebesgue measure. 
\end{theorem}
\begin{proof}
Without loss of generality, we let $r \le s$, i.e., $r=\min\{r,s\}$. We divide all $\{\sigma_j(\Tm)\}_{j=1}^{rn^2}$ into $r$ clusters $\{\sigma_k^{(j)}(\Tm),k\in[n^2]\}_{j=1}^{r}$ according to their localization, each of which is arranged in ascending order, i.e., 
\begin{align}
    \sigma_1^{(j)}(\Tm) \le \sigma_2^{(j)}(\Tm) \le \dots \le \sigma_{n^2}^{(j)}(\Tm), \quad \forall j \in [r].
\end{align}
From Theorem 1, we have
\begin{align} \label{eq:theorem-1-proof}
  \frac{1}{r} \sum_{j=1}^{r} \lim_{n \to \infty}  \frac{1}{n^2} \sum_{k=1}^{n^2} \Phi(\sigma_k^{(j)}(\Tm)) = \frac{1}{r}\sum_{j=1}^{r} \frac{1}{(2\pi)^2} \int_{-\pi}^{\pi} \int_{-\pi}^{\pi}  \Phi(\sigma_j(F(\omega_1,\omega_2))) d\omega_1 d\omega_2.
\end{align}

Let $\phi_j:[-\pi,\pi]^2 \mapsto \RR_+$ be the $j$-th singular value function of $F(\omegav)$, i.e., $\phi_j(\omegav)=\sigma_j(F(\omega_1,\omega_2))$. When taking $\omegav$ as a multivariate random variable with uniform distribution on $[-\pi,\pi]^2$, we can treat $\phi_j(\omegav)$ as a continuous random variable, such that the right-hand side of \eqref{eq:theorem-1-proof} can be interpreted as 
\[
\frac{1}{r} \sum_{j=1}^{r} \E_{\omegav} [\Phi(\phi_j(\omegav))]
\]
Similarly, we can treat $\{\sigma_k^{(j)}(\Tm)\}_{k=1}^{n^2}$ as realizations of discrete random variable $X_n^{(j)}$ with equal probability $\Pr(X_n^{(j)}=\sigma_k^{(j)}(\Tm))=\frac{1}{n^2}$, and interpret the left-hand side of \eqref{eq:theorem-1-proof} as
\[
\frac{1}{r} \sum_{j=1}^{r} \lim_{n \to \infty} \E_{X_n^{(j)}} [\Phi(X_n^{(j)})]
\]
Thus, from a probabilistic perspective, Theorem 1 says, for the sequence of random variables $\{X_1^{(j)},X_2^{(j)},\dots,X_n^{(j)},\dots\}$, $\E_{X_n^{(j)}} [\Phi(X_n^{(j)})]$ converges to $\E_{\omegav} [\Phi(\phi_j(\omegav))]$ in distribution for any continuous function $\Phi$.

For both random variables $X_n^{(j)}$ and $\phi_j(\omegav)$, let us define the cumulative distribution and quantile functions as
\begin{align}
    G_{X_n^{(j)}}(v)&=\frac{1}{n^2} \max \{k \in [n^2]: \sigma_k^{(j)}(\Tm) \le v\}\\
    Q_{X_n^{(j)}}(u)&=\inf \{v \in \RR: u \le G_{X_n^{(j)}}(v)\}\\
    G_{\phi_j}(v)&=\frac{1}{(2\pi)^2}\mu\{\omegav \in [-\pi,\pi]^2: \phi_j(\omegav)\le v\}\\
    Q_{\phi_j}(u)&=\inf\{v \in \RR: u \le G_{\phi_j}(v)\}
\end{align}
where $\mu$ is the Lebesgue measure of $\omegav$ on $[-\pi,\pi]^2$. As $\{\sigma_k^{(j)}(\Tm)\}_{k=1}^{n^2}$ is ordered and $G_{X_n^{(j)}}(v)$ is right continuous and non-decreasing over $v$, it follows from \cite{Bogoya2015}[Proposition 2.5] that
\begin{align} \label{eq:thm3-proof-Q-sigma}
    Q_{X_n^{(j)}}(\frac{k}{n^2}) = \sigma_k^{(j)}(\Tm).
\end{align}

By Portmanteau Lemma \cite{Bogoya2015}[Lemma 3.1], the fact that $\E_{X_n^{(j)}} [\Phi(X_n^{(j)})]$ converges to $\E_{\omegav} [\Phi(\phi_j(\omegav))]$ in distribution for any continuous function $\Phi$ leads to (1) $G_{X_n^{(j)}}(v)$ converges to $G_{\phi_j}(v)$ for every $v \in \RR$ at which  $G_{\phi_j}$ is continuous, and (2) $Q_{X_n^{(j)}}(u)$ converges to $Q_{\phi_j}(u)$ for every $u \in (0,1]$ at which $Q_{\phi_j}$ is continuous.

Inspired by \cite{Zizler2002}[Theorem 3.2, Corollary 3.3], we can further bound the gap between $G_{X_n^{(j)}}(v)$ and $G_{\phi_j}(v)$.
\begin{lemma} \label{lemma:thm3-gap}
There exists a constant $c_1$ such that
\begin{align} \label{eq:apdx-G-gap}
   \max_{j \in [r]} \;\; \abs{G_{X_n^{(j)}}(v)-G_{\phi_j}(v)} \le \frac{c_1}{n}
\end{align}
for every $n>1$.
\end{lemma}
\begin{proof}
Due to Theorem 2, the singular values of $\Cm$ can be given by those of $F(\omegav)$ with uniform sampling on $[-\pi,\pi]^2$, i.e.,  
\begin{align}
\{\sigma_k^{(j)}(\Cm)\}_{k=1}^{n^2} = \left\{\sigma_{j}(F(\omega_1,\omega_2)): (\omega_1,\omega_2) = (-\pi+ \frac{2 \pi j_1}{n}, -\pi + \frac{2 \pi j_2}{n}), \forall j_1,j_2 \in [n]-1 \right\}
\end{align}
for $j \in [r]$.
Following the same footsteps of \cite{Zizler2002}[Theorem 2.2], we have
\begin{align}
    \abs{\sum_{k=1}^{n^2} \sigma_k^{(j)}(\Cm) - \frac{n^2}{(2\pi)^2}\int_{-\pi}^{\pi} \int_{-\pi}^{\pi} \sigma_j(F) d \omega_1 d \omega_2 } \le c'_0 n
\end{align}
where $c'_0>0$ is a constant that does not depend on $n$.
Due to Theorem 2, there must exist a constant $c_0>0$ such that
\begin{align}
   \sum_{k=1}^{n^2} \abs{\sigma_k^{(j)}(\Tm) -  \sigma_k^{(j)}(\Cm)} \le c_0 n.
\end{align}

It follows that, there exists a constant $c_1>0$ that does not depend on $n$ such that
\begin{align}
     \abs{\sum_{k=1}^{n^2} \sigma_k^{(j)}(\Tm) - \frac{n^2}{(2\pi)^2}\int_{-\pi}^{\pi} \int_{-\pi}^{\pi} \sigma_j(F) d \omega_1 d \omega_2 } \le c_1n
\end{align}

By \cite{Zizler2002}[Corollary 3.3], for a real value $v$, we have
\begin{align}
   \abs{G_{X_n^{(j)}}(v)-G_{\phi_j}(v)} \le \frac{c_1}{n}
\end{align}
for all $j$, where $\phi_j(\omegav)$ takes values of $\sigma_j(F(\omegav))$ that are upper bounded given the fact that $F(\omegav)$ is a Laurent polynomial matrix with respect to $e^{\jmath \omegav}$, each element of which is a Laurent polynomial. This completes the proof.
\end{proof}

Let $\epsilon=\frac{c_1}{n}$ and $\frac{k-1}{n^2} < u \le \frac{k}{n^2}$. By \cite{Bogoya2015}[Proposition 2.2], we have $u \le G_{\phi_j}(Q_{\phi_j}(u))$. Together with Lemma \ref{lemma:thm3-gap}, we have
\begin{align}
    u &= u+\epsilon - \epsilon \le G_{\phi_j}(Q_{\phi_j}(u+\epsilon)) - \epsilon \le G_{X_n^{(j)}}(Q_{\phi_j}(u+\epsilon))
\end{align}

Let $\delta=c\epsilon$ with $c>0$ being a constant. Given the fact that $Q_{\phi_j}(u-\epsilon) \ge Q_{\phi_j}(u-\epsilon) - \delta$, we have
\begin{align}
    u-\epsilon \ge G_{\phi_j}(Q_{\phi_j}(u-\epsilon)-\delta)  \ge G_{X_n^{(j)}}(Q_{\phi_j}(u-\epsilon)-\delta) - \epsilon
\end{align}
Thus, due to the fact that $u \le G_{X_n^{(j)}}(v)$ if and only if $Q_{X_n^{(j)}}(u) \le v$, we have
\begin{align} \label{eq:thm3-proof-Q-bounds} 
     Q_{X_n^{(j)}}(u) &\le Q_{\phi_j}(u+\epsilon)  \\
     Q_{X_n^{(j)}}(u) &\ge Q_{\phi_j}(u-\epsilon) - \delta.
\end{align}

Before proceeding further, we investigate the Lipschitz continuity of $\phi_j$.
\begin{lemma} \label{lemma:thm3-lip}
The singular value function $\phi_j(\omegav)=\sigma_j(F(\omegav))$ is Lipschitz continuous for every $j$. 
\end{lemma}
\begin{proof}
According to the generalized Hoffman-Wielandt theorem for singular values \cite{mirsky1960symmetric}[Theorem 5] and \cite{sun1983perturbation}[Theorem 5.1], we have
\begin{align}
    \sqrt{\sum_{j=1}^{r} \Abs{\sigma_j(F(\omegav))-\sigma_j(F(\omegav'))}} & \stackrel{}{\le} \normf{F(\omegav)-F(\omegav')}\\
    &= \normf{\sum_{k_1=-h_1}^{h_2} \sum_{k_2=-w_1}^{w_2} \Tm_{k_1,k_2} (e^{\jmath \kv^\T\omegav} - e^{\jmath \kv^\T\omegav'})}\\
     &\stackrel{(a)}{\le} \sum_{k_1=-h_1}^{h_2} \sum_{k_2=-w_1}^{w_2} \normf{ \Tm_{k_1,k_2} }
    \abs{e^{\jmath \kv^\T\omegav} - e^{\jmath \kv^\T\omegav'}}\\
     &\stackrel{(b)}{\le} \sum_{k_1=-h_1}^{h_2} \sum_{k_2=-w_1}^{w_2} \normf{ \Tm_{k_1,k_2} }
    \abs{ \kv^\T (\omegav-\omegav') }\\
    &\stackrel{(c)}{\le} \sum_{k_1=-h_1}^{h_2} \sum_{k_2=-w_1}^{w_2} \norm{\kv} \normf{ \Tm_{k_1,k_2} }
    \norm{\omegav-\omegav'}
\end{align}
where $(a)$ is due to the triangle inequality of matrix norm, $(b)$ is due to the non-negativity of matrix norms and the following inequality
\begin{align}
    \abs{e^{\jmath \kv^\T\omegav} - e^{\jmath \kv^\T\omegav'}} &= \abs{\int_{\omegav'}^{\omegav} \jmath e^{\jmath \kv^\T\tv} \kv^\T d \tv}\\
    &\le \abs{\int_{\omegav'}^{\omegav} \abs{\jmath e^{\jmath \kv^\T\tv}} \kv^\T d \tv}\\
    &\le \abs{ \kv^\T \int_{\omegav'}^{\omegav} d \tv}\\
    &\le \abs{ \kv^\T (\omegav-\omegav') }
\end{align}
and $(c)$ is due to Cauchy-Schwarz inequality.

Let $K=\sum_{k_1=-h_1}^{h_2} \sum_{k_2=-w_1}^{w_2} \norm{\kv} \normf{ \Tm_{k_1,k_2} }$, which is a positive constant that does not depend on $\omegav$. Thus, we have
\begin{align}
 \abs{\sigma_j(F(\omegav))-\sigma_j(F(\omegav'))}  \le K \norm{\omegav-\omegav'}
\end{align}
for all $j$, which means that $\sigma_j(F(\omegav))$ is $K$-Lipschitz continuous, so is $\phi_j(\omegav)$ by definition.
\end{proof}

Provided Lemma \ref{lemma:thm3-lip}, following the same footsteps of Proposition 2.7 in \cite{Bogoya2015}, we conclude that $Q_{\phi_j}(u)$ is also Lipschitz continuous, i.e.,
\begin{align} \label{eq:apdx-Q-phi}
    \abs{Q_{\phi_j}(u_1) - Q_{\phi_j}(u_2)} \le L \abs{u_1-u_2}
\end{align}
for all $u_1,u_2\in(0,1]$.

Now, equipped with the Lipschitz continuity, by \eqref{eq:thm3-proof-Q-sigma} and \eqref{eq:thm3-proof-Q-bounds}, we have 
\begin{align} \label{eq:thm3-proof-left}
    \sigma_k^{(j)}(\Tm) &= Q_{X_n^{(j)}}(u) \le Q_{\phi_j}(u+\epsilon) \le Q_{\phi_j}(u)+L\epsilon \\
    \sigma_k^{(j)}(\Tm) &= Q_{X_n^{(j)}}(u) \ge Q_{\phi_j}(u-\epsilon) - \delta \ge Q_{\phi_j}(u)-L\epsilon - \delta
\end{align}
for $u \in (\frac{k-1}{n^2},\frac{k}{n^2}]$.
This implies that
\begin{align}
\abs{\sigma_k^{(j)}(\Tm)-Q_{\phi_j}(u)} \le  L\epsilon + \delta \defeq \frac{c_2}{n}
\end{align}
for all $k \in [n^2]$ and $j \in [r]$.
This completes the proof.
\end{proof}
\begin{remark}
Theorem 3 offers a better approximation method for the individual singular values of the linear transformation matrix $\Tm$. Although the quantile approximation approach has the same scaling law  of accumulated approximation error (i.e., $O(n)$) as the circular approximation, the individual singular value approximation accuracy is somewhat guaranteed with vanishing error as $n$ tends to infinity. In contrast, this may not be guaranteed by the circular approximation. From Theorem~2, it is possible that the largest singular value by circular approximation can scale as $n$. Albeit promising from a theoretical point of view, it is challenging to characterize the exact quantile function. A compromised way is to estimate such a quantile function through the circular approximation, from which the singular value distribution can be adjusted so as to reach a relatively better approximation. The experimental results show that a naive subtle adjustment of singular values obtained by the circular approximation (i.e., uniform sampling of $F$) yields notable improvement on approximation accuracy, especially for the largest singular value.
\end{remark}

\subsection{Proof of Theorem 4}
To upper bound the spectral norm of the linear transformation matrix $\Tm$, we bound it by the spectral norm of its spectral representation - the spectral density matrix $F(\omega_1,\omega_2)$.
\begin{lemma}
$
\norm{\Tm}_2 \le \norm{F}_2.
$
\end{lemma}
\begin{proof}
Inspired by Theorem 4.1 of \cite{Tilli1998}, we extend the proof from block Toeplitz to doubly block Toeplitz matrices.

Given a singular value of $\Tm \in \RR^{rn^2 \times sn^2}$, say $\sigma(\Tm)$, there exist $\uv \in \RR^{rn^2}$ and $\vv \in \RR^{sn^2}$ subject to $\norm{\uv}_2=\norm{\vv}_2=1$ such that $\sigma(\Tm)=\uv^{\T} \Tm \vv$, where $\uv=[\uv_{k,l}]_{k,l}$ and $\vv=[\vv_{k,l}]_{k,l}$, with the $(k,l)$-th block vector $\uv_{k,l} \in \RR^{r \times s}$ and $\vv_{k,l} \in \RR^{r \times s}$ corresponding to $\Tm_{k,l}$. According to the definition of $\Tm_{k,l}$ in \eqref{eq:App-T-F}, we have
\begin{align}
    \sigma(\Tm) = \frac{1}{(2 \pi)^2} \int_{-\pi}^{\pi} \int_{-\pi}^{\pi} u(\omega_1,\omega_2)^{\T} F(\omega_1,\omega_2) v(\omega_1,\omega_2) d \omega_1 d \omega_2
\end{align}
where $u(\omega_1,\omega_2)$ and $v(\omega_1,\omega_2)$ are Fourier transforms of $\uv_{k,l}$ and $\vv_{k,l}$, respectively, i.e.,
\begin{align}
    u(\omega_1,\omega_2) = \sum_{k=1}^n \sum_{l=1}^n \uv_{k,l} e^{\jmath (k \omega_1 + l \omega_2)}, \\
    v(\omega_1,\omega_2) = \sum_{k=1}^n \sum_{l=1}^n \vv_{k,l} e^{\jmath (k \omega_1 + l \omega_2)}.
\end{align}
Thus, we have
\begin{align}
    \sigma(\Tm) &\stackrel{(a)}{\le} \frac{1}{(2 \pi)^2} \int_{-\pi}^{\pi} \int_{-\pi}^{\pi}
    \sigma_{\max}(F)
    \norm{u(\omega_1,\omega_2)}_2  \norm{v(\omega_1,\omega_2)}_2 d \omega_1 d \omega_2\\
    &\stackrel{(b)}{\le} \sigma_{\max}(F)
    \frac{1}{(2 \pi)^2} \sqrt{\int_{-\pi}^{\pi} \int_{-\pi}^{\pi}
    \Norm{u(\omega_1,\omega_2)}_2 d \omega_1 d \omega_2}
    \sqrt{\int_{-\pi}^{\pi} \int_{-\pi}^{\pi}\Norm{v(\omega_1,\omega_2)}_2 d \omega_1 d \omega_2}\\
    &\stackrel{(c)}{=} \sigma_{\max}(F)
     \norm{\uv}_2 \norm{\vv}_2\\
    &= \sigma_{\max}(F)
\end{align}
where $(a)$ is from the definition of the largest singular value, i.e., $\sigma_{\max}(F) = \sup \frac{u^{\T} F v}{\norm{u}_2\norm{v}_2}$, $(b)$ is due to Cauchy inequality, and $(c)$ is resulted directly from the computation of integrals.
Thus, it follows immediately that $\norm{\Tm}_2 \le \norm{F}_2$.
\end{proof}
\begin{theorem}
The spectral norm $\norm{F}_2$ can be bounded by
\begin{align} 
\norm{F}_2 &\le \min \Big\{\sqrt{hw}\norm{\Rm}_2, \sqrt{hw}\norm{\Lm}_2 \Big\}\\
\norm{F}_2 &\le \max_{\omegav} \sqrt{\norm{F(\omegav)}_1 \norm{F(\omegav)}_{\infty}}\\
    \norm{F}_2 &\le \sum_{k=-h_1}^{h_2} \sum_{l=-w_1}^{w_2} \norm{\Tm_{k,l}}_2
\end{align}
where $\Rm \in \RR^{hc_{out} \times wc_{in}}$ is a $c_{out} \times c_{in}$ block matrix with $(c,d)$-th block being $\Km_{c,d,:,:} \in \RR^{h \times w}$ and $\Lm \in \RR^{wc_{out} \times hc_{in}}$ is a $c_{out} \times c_{in}$ block matrix with $(c,d)$-th block being $\Km_{c,d,:,:}^\T \in \RR^{w \times h}$.
\end{theorem}
\begin{proof}
Let $z_1=e^{\jmath \omega_1}$ and $z_2 = e^{\jmath \omega_2}$. The $(c,d)$-th element of the spectral density matrix $F(\omega_1,\omega_2)$ can be rewritten as
\begin{align}
    F_{c,d}(z_1,z_2)=\sum_{k=-h_1}^{h_2} \sum_{l=-w_1}^{w_2} t_{c,d}^{k,l} z_1^k z_2^l.
\end{align}
which is a polynomial with respect to $z_1$ and $z_2$.

Let $\Rm_{c,d}=[t_{c,d}^{k,l}]_{k,l} \in \RR^{h \times w}$, $\zv_1=[z_1^{-h_2},\dots,z_1^{h_1}]$ and $\zv_2=[z_2^{-w_2},\dots,z_2^{w_1}]$. Thus, we can represent $F_{c,d}$ in the following two ways. 
\begin{align}
    F_{c,d}=\zv_1 \Rm_{c,d} \zv_2^\T =\zv_2 \Rm^{\T}_{c,d} \zv_1^\T.
\end{align}
Hence, the spectral density matrix $F$ can be represented as
\begin{align}
    F &= (\Id_{r} \otimes \zv_1) \Rm
    (\Id_{s} \otimes \zv_2^{\T})\\
    &= (\Id_{r} \otimes \zv_2) \Lm
    (\Id_{s} \otimes \zv_1^{\T})
\end{align}
where 
\begin{align}
    \Rm = \begin{bmatrix} 
    \Rm_{1,1} & \Rm_{1,2} & \cdots & \Rm_{1,s}\\
    \Rm_{2,1} & \cdots & \cdots & \Rm_{2,s}\\
    \vdots & \vdots & \vdots &\vdots \\
    \Rm_{r,1} & \cdots & \cdots & \Rm_{r,s}
    \end{bmatrix}, \hspace{20pt}
    \Lm = \begin{bmatrix} 
    \Rm_{1,1}^\T & \Rm_{1,2}^\T & \cdots & \Rm_{1,s}^\T\\
    \Rm_{2,1}^\T & \cdots & \cdots & \Rm_{2,s}^\T\\
    \vdots & \vdots & \vdots &\vdots \\
    \Rm_{r,1}^\T & \cdots & \cdots & \Rm_{r,s}^\T
    \end{bmatrix},
\end{align}
with $\Rm \in \RR^{rh \times sw}$ and $\Lm \in \RR^{rw \times sh}$.
Note that 
\begin{align}
    {(\Id_{r} \otimes \zv_1)} (\Id_{r} \otimes \zv_1)^{\H} = h \Id_r\\
(\Id_{s} \otimes \zv_2) (\Id_{s} \otimes \zv_2)^{\H} = w \Id_s
\end{align}
where the columns are orthogonal. So, we have
\begin{align}
    \norm{F}_2 \le \sqrt{hw}\norm{\Rm}_2, \quad \norm{F}_2 \le \sqrt{hw}\norm{\Lm}_2.
\end{align}
This gives us the first bound.

For the second bound, given any $\omegav \in [-\pi,\pi]^2$, we have
\begin{align}
    \Norm{F(\omegav)}_2 \le \norm{F(\omegav)}_1 \norm{F(\omegav)}_{\infty}.
\end{align}
As $\norm{F}_2 \le \max_{\omegav} \norm{F(\omegav)}_2$, we have the second spectral norm bound.

For the third bound, we have
\begin{align}
    \norm{F(\omega_1,\omega_2)}_2 &= \norm{\sum_{k} \sum_{l} \Tm_{k,l} e^{\jmath(k\omega_1+l\omega_2)}}_2\\
    &\le \sum_{k} \sum_{l} \norm{\Tm_{k,l}}_2 \abs{e^{\jmath(k\omega_1+l\omega_2)}}\\
    &= \sum_{k} \sum_{l} \norm{\Tm_{k,l}}_2,
\end{align}
where the inequality is due to Cauchy--Schwarz inequality.
\end{proof}

\section{Extensions and Discussions}
\label{sec:discussion}
Some more general cases are discussed with respect to larger stride size, higher dimensional linear convolution, and multiple convolutional layers in linear networks without activation functions and pooling layers.
\subsection{Stride Larger Than 1}
In previous sections, we were dedicated to linear convolution with stride size 1. When the stride size $g$ is larger than 1, i.e., $g > 1$, the linear transformation matrix $\Tm$ becomes a block $g$-Toeplitz matrix, denoted by $\Tm^g$. For simplicity, we consider the same stride side on both horizontal and vertical directions. Thus, we have $\Tm^g=[\Tm_{gk}]_{k=0}^{n-1}$ where $\Tm_{gk}=[\Tm_{gk,gl}]_{l=0}^{n-1}$ with $\Tm_{k,l}$ defined in \eqref{eq:toe-blocks}.

According to \cite{ngondiep2010spectral}, we have an analogous result to Theorem \ref{thm:szg-limit}.

Let $F:[-\pi,\pi]^2 \mapsto \mathbb{C}^{r \times s}$ be a matrix-valued function, subject to $F \in \mathcal{L}^2([-\pi,\pi]^2)$. The linear transformation matrix $\Tm^g$ with stride $g$ converges to the generating function $F$, i.e., $\Tm^g \sim_{\sigma} F(\omegav,\mv)$, where
\begin{align}
    F(\omegav,\mv)= \sqrt{\frac{1}{g^2} \sum_{m_1=0}^{g-1} \sum_{m_2=0}^{g-1} f^2(\omegav,\mv)}
\end{align}
if $\mv=(m_1,m_2) \in [0,\frac{1}{g}]^2$ and 0 otherwise, with
\begin{align*}
 f(\omegav,\mv) = \sum_{k}^{} \sum_{l}^{} {\Tm}_{gk,gl} e^{\jmath \frac{1}{g}(k (\omega_1+2 \pi m_1) + l (\omega_2 + 2 \pi m_2))}.   
\end{align*}
By this, the singular value distribution of $\Tm^g$ can be alternatively studied on the generating function $F(\omegav,\mv)$. 

\subsection{Higher Dimensional Convolution}
According to \cite{Oudin2009}, a block multi-level Toeplitz matrix $\Tm=\{\Tm_{\iv-\jv}\}_{\iv,\jv=\one}^\nv$ with $\iv=(i_1,\dots,i_d)$, $\jv=(j_1,\dots,j_d)$, and $\nv=(n_1,\dots,n_d)$, it can be alternatively represented as 
\begin{align}
    \Tm = \sum_{\abs{k_1}<n_1} \dots \sum_{\abs{k_d}<n_d} [\Jm_{n_1}^{(k_1)} \otimes \dots \otimes \Jm_{n_d}^{(k_d)}] \otimes \Tm_{\kv}
\end{align}
where $\Jm_{n_j}^{(k_j)}$ is a $n_j \times n_j$ binary matrix with $(p,q)$-th entry being 1 of $p-q=k_j$ and 0 elsewhere, and
\begin{align}
    \Tm_{\kv}=\frac{1}{(2\pi)^d} \int_{\Omega} F(\omegav) e^{-\jmath <\kv,\omegav>} d \omegav
\end{align}
with $\Omega=[-\pi,\pi]^d$, $\kv=(k_1,\dots,k_d)$, $\omegav=(\omega_1,\dots,\omega_d)$ and $<\kv,\omegav>=\sum_{j=1}^d k_j \omega_j$. Then it follows that Theorem 1 can be generalized to $d$-dim linear convolutional layers
\begin{align} \label{eq:multi-level-theorem-1}
    \MoveEqLeft  \lim_{\nv \to \infty} \frac{1}{N} \sum_{j=1}^{\min\{r,s\} N} \Phi(\sigma_j(\Tm)) 
    = \frac{1}{(2\pi)^d} \int_{\Omega} \sum_{j=1}^{\min\{r,s\}} \Phi(\sigma_j(F(\omegav))) d\omegav
\end{align}
with $N=\prod_{i=1}^d n_i$,
for which the asymptotic singular value distribution of higher dimensional linear convolutional layers can be studied through $F: [-\pi,\pi]^d \mapsto \CC^{r \times s}$.

\subsection{Multiple Linear Convolutional Layers}
The collective effect of multiple linear convolutional layers without activation function or pooling layers in CNNs can be seen as the product of the linear transformation matrices of multiple convolutional layers.

For convolutional layers, denote by $\Tm(F_i)$ the linear transformation matrix generated from the spectral density matrix $F_i: [-\pi,\pi]^2 \mapsto \CC^{r \times s}$, for $i=1,\dots,M$. It follows from \cite{barbarino2020block}[Theorem 2.46] that
\begin{align}
    \lim_{n \to \infty} \frac{1}{n^2} \norm{\prod_{i=1}^M \Tm(F_i) - \Tm(\prod_{i=1}^M F_i)}_1 = 0
\end{align}
which means that the product of Toeplitz matrices is asymptotically equal to the Toeplitz matrix generated by the product of all generating functions associated to each linear convolutional layer. 

By this, the spectral analysis of $M$ linear convolutional layers can be alternatively studied on the product of generating functions $\prod_{i=1}^M F_i$. %

\section{Additional Experimental Results}
\label{sec:experiments}
\subsection{Singular Value Approximation}
\label{sec:experiments-sva}
To evaluate the accuracy of our singular value approximation method, we consider three different types of weights that are: (1) randomly generated according to uniform and Gaussian distributions (as CNNs are usually initialized), (2) extracted from pre-trained networks on ImageNet dataset (as CNNs finally converge), and (3) extracted from the training process of ResNets on CIFAR-10 dataset (as CNNs are updated with training epochs).

For simplicity, we set $h_1=h_2$ and $w_1=w_2$, and the input size $n \times n$ per channel is set to $10 \times 10$.
In what follows, the plots present the $(i-1)n+1$-th largest singular values ($i \in [n]$) of four methods with different filter sizes. 
It is worth noting that each singular value shown in the figures represents a cluster of singular values with similar behavior. For instance, the first spike shows the largest singular value, and the following $n-1$ large singular values between the first and the second spikes, which have not been shown in the figure, have similar approximation behavior.

\subsubsection{Randomly Generated Weights}
\begin{figure}[t]
\hspace{-12pt}
{\includegraphics[width=0.26\columnwidth]{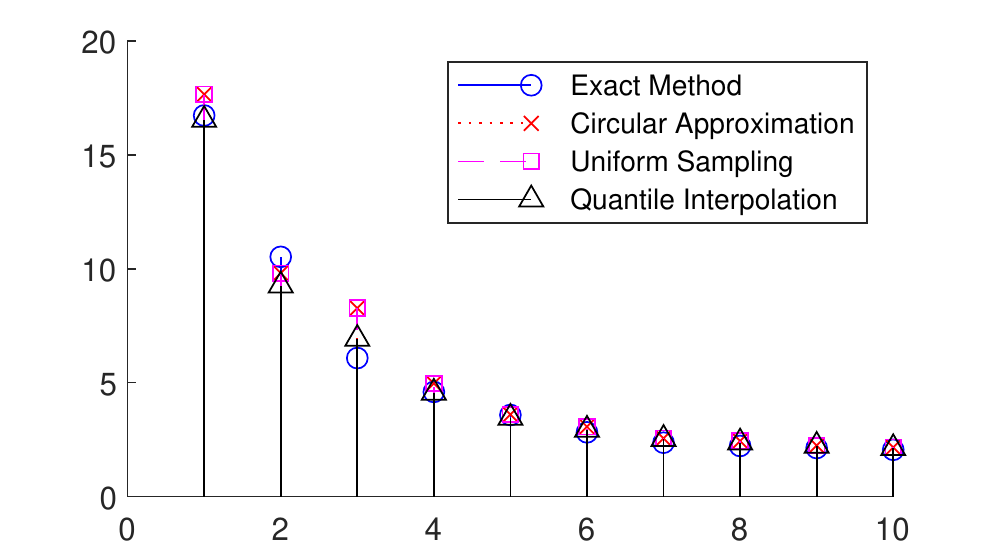}}
\hspace{-12pt}
{\includegraphics[width=0.26\columnwidth]{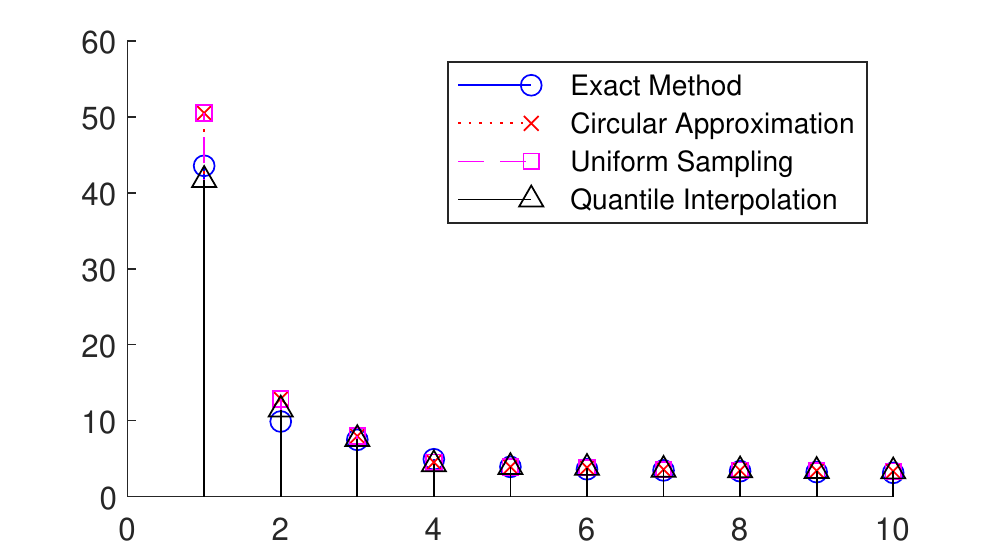}}
\vspace{-12pt}
{\includegraphics[width=0.26\columnwidth]{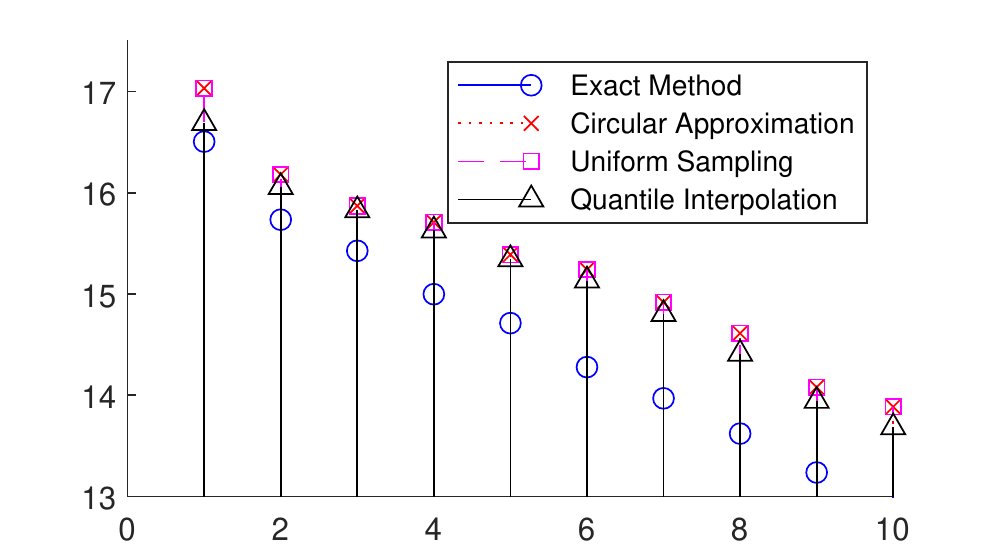}}
\hspace{-12pt}
{\includegraphics[width=0.26\columnwidth]{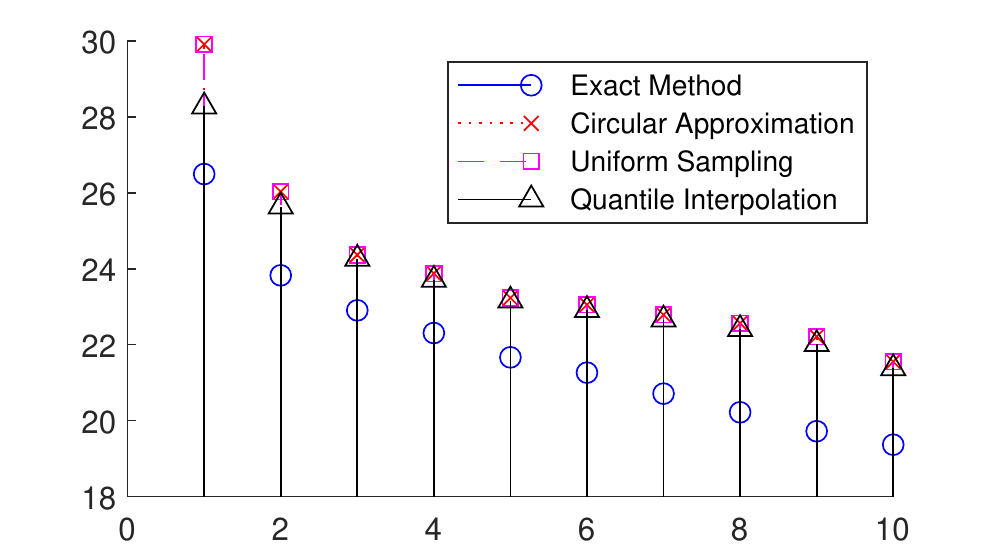}}
\hspace{-12pt}
\begin{center}
\vspace{0.1in}
\caption{Exact and approximated singular values of linear convolutional layers arranged in descending order. With input size per channel $10 \times 10$, only $10$ singular values are shown, each of which represents the behavior of a cluster of singular values. Four types of convolutional filters are considered from left to right with sizes $8 \times 8 \times 3 \times 3$ (uniform distribution),  $8 \times 8 \times 5 \times 5$ (uniform distribution), $8 \times 8 \times 3 \times 3$ (Gaussian distribution), and $8 \times 8 \times 5 \times 5$ (Gaussian distribution), respectively. }
\label{fig:APP-Fig-1}
\end{center}
\vskip -0.35in
\end{figure}
We consider two distributions used for weights initialization.
It has been observed in \cite{Thoma:2017} that weight of CNN layers are located within [-0.5, 0.5]. As such, we randomly generate the weights of convolutional filters following uniform distribution in [-0.5, 0.5]. In addition, the Gaussian distribution initialized weights with zero mean and unit variance are also considered.

First, as in the main text, we illustrate the singular value approximation accuracy among three methods - circular approximation, uniform sampling, and quantile interpolation - against the exact method.
Fig. \ref{fig:APP-Fig-1} presents the $(i-1)n+1$-th largest singular values ($i \in [n]$) of four methods with four different filter sizes. 
Differently from the observations in the main text, we observe that, (1) for uniformly distributed weights, quantile interpolation substantially improve over the circular approximation on the larger singular values (where each spike in the figure represents a number of them with similar behavior), while the smaller singular values obtained from both the circular approximation and quantile interpolation approach the exact values; (2) for Gaussian distributed weights, while quantile interpolation has significant improvement over the circular approximation for the larger singular values (including the largest one and those that are not shown in the figure), the improvement for small singular values is little, because the circular approximation is very inaccurate and the simple adjustment of singular value distribution using linear interpolation does not improve much the accuracy. It calls for more sophisticated nonlinear interpolation methods.  

Besides the singular value illustration as above, we also compute the average accuracy of different approximation methods by Monte-Carlo simulation. We randomly generate 100 different realizations, and calculate the average accuracy over these 100 realizations.
We mainly consider three input sizes $10 \times 10$, $20 \times 20$, and $10 \times 30$ with stride 1 due to limited computing resources (i.e., HP EliteBook with Intel i5 CPU and 8GB RAM). To reduce computational complexity, the inputs with larger size usually have larger stride, which can be roughly seen as a smaller input size with stride 1. Table \ref{tab:accuracy} collects the accuracy performance of different approximation methods (CA=circular approximation, QI=quantile interpolation) compared with the exact solution. As the circular approximation is identical to the uniform sampling method, we only collect the performance of circular approximation for brevity. We mainly consider the approximation error of overall singular values and the largest one, for which the overall error is defined as $\frac{\sum_j\abs{\sigma_j(\Tm)-\hat{\sigma}_j}}{\sum_j\abs{\sigma_j(\Tm)}}$
and the error for the first singular value is $\frac{\abs{\sigma_1(\Tm)-\hat{\sigma}_1}}{\abs{\sigma_1(\Tm)}}$
with $\hat{\sigma}_j$ being the approximated value by different methods.

\begin{table}[]
\caption{Accuracy of approximation methods.}
\vskip 0.1in
    \centering
    \begin{tabular}{c|c|c|c}
\hline
\centering
     Input size & Convolutional filter size &  Overall Error  & Error for 1st Singular Value\\ \hline
     $10 \times 10$ & $8 \times 8 \times 3 \times 3$ & CA=10.4\%, QI=8.3\% & CA=5.6\%, QI=0.9\% \\ \hline
     & $8 \times 8 \times 3 \times 5$ & CA=15.6\%, QI=12.7\% & CA=10.7\%, QI=3.9\% \\ \hline
     & $8 \times 8 \times 5 \times 3$ & CA=14.5\%, QI=11.3\% & CA=10.7\%, QI=3.7\% \\ \hline
     & $8 \times 8 \times 5 \times 5$ & CA=20.4\%, QI=14.8\% & CA=16.1\%, QI=3.9\% \\ \hline
     & $8 \times 8 \times 7 \times 7$ & CA=30.9\%, QI=23.2\% & CA=31.4\%, QI=8.7\% \\ \hline
     & $8 \times 8 \times 9 \times 9$ & CA=46.4\%, QI=31.8\% & CA=51.9\%, QI=11.3\% \\ \hline
     & $8 \times 8 \times 5 \times 9$ & CA=29.9\%, QI=16.4\% & CA=32.8\%, QI=9.9\% \\\hline
     & $8 \times 8 \times 9 \times 5$ & CA=30.1\%, QI=17.0\% & CA=32.7\%, QI=9.3\% \\\hline
     & $16 \times 3 \times 5 \times 9$ & CA=29.4\%, QI=15.3\% & CA=32.8\%, QI=10.1\% \\\hline
     & $16 \times 8 \times 5 \times 9$ & CA=31.1\%, QI=18.5\% & CA=32.8\%, QI=9.8\% \\\hline
     & $16 \times 16 \times 5 \times 9$ & CA=31.9\%, QI=19.9\% & CA=32.9\%, QI=9.7\% \\\hline
     $20 \times 20$ & $8 \times 8 \times 5 \times 5$ & CA=9.1\%, QI=7.7\% & CA=4.3\%, QI=0.6\% \\ \hline
     & $8 \times 8 \times 7 \times 7$ & CA=14.0\%, QI=11.1\% & CA=8.4\%, QI=1.5\% \\ \hline
      & $8 \times 8 \times 7 \times 9$ & CA=12.5\%, QI=9.9\% & CA=11.1\%, QI=0.8\% \\ \hline
      & $8 \times 8 \times 9 \times 9$ & CA=18.4\%, QI=13.2\% & CA=13.7\%, QI=3.0\% \\ \hline
      $10 \times 30$ & $8 \times 8 \times 7 \times 7$ & CA=19.9\%, QI=16.9\% & CA=16.8\%, QI=11.8\% \\ \hline
      & $8 \times 8 \times 5 \times 11$ & CA=13.5\%, QI=10.3\% & CA=12.6\%, QI=1.1\% \\ \hline
      & $8 \times 8 \times 7 \times 11$ & CA=24.0\%, QI=19.1\% & CA=19.9\%, QI=7.6\% \\ \hline
\end{tabular}
    \label{tab:accuracy}
\end{table}

It is observed from Table \ref{tab:accuracy} that for smaller filter size, e.g., $3 \times 3$, the circular approximation looks good enough, and the improvement by quantile interpolation is not much, e.g., by 2.1\% for overall performance, and by 4.7\% for the large singular value. However, as the filter size increases, e.g., $7 \times 7$, the circular approximation can be as large as 30\% away from the exact singular values, and the improvement by quantile interpolation is significant, e.g., by 7.7\% for overall performance and by 22.7\% for the largest singular value. %
As the input size increases, the approximation accuracy for both circular approximation and quantile interpolation is improved, and therefore the improvement of the latter over the former is not substantial compared with those of the smaller input size. It appears in Table \ref{tab:accuracy} that the numbers of channels of the input and the output do not have much influence on the accuracy performance.
It appears that (1) the circular approximation is sufficiently good when the input size is large and the convolutional filter size is small, and it leaves no much room to improve by quantile interpolation; (2) the quantile interpolation approach makes a difference when convolutional filter size is large, yet there is still certain gap to the exact values. This is mainly due to the simple quantile estimation method. It is expected to have larger improvement with more accurate quantile interpolation methods. We leave it to our future work.

\subsubsection{Weights from Pre-trained Networks}
In what follows, we present more experimental results on the accuracy of singular value approximation on pre-trained network models, such as VGG16, VGG19 \cite{vgg16-paper},  AlexNet, DenseNet \cite{huang2017densely}, GoogLeNet \cite{szegedy2015going}, InceptionResNetv2, Inceptionv3, and ResNets \cite{ResNet}, which are trained on ImageNet dataset.

\begin{figure}[t]
{\includegraphics[width=0.26\columnwidth]{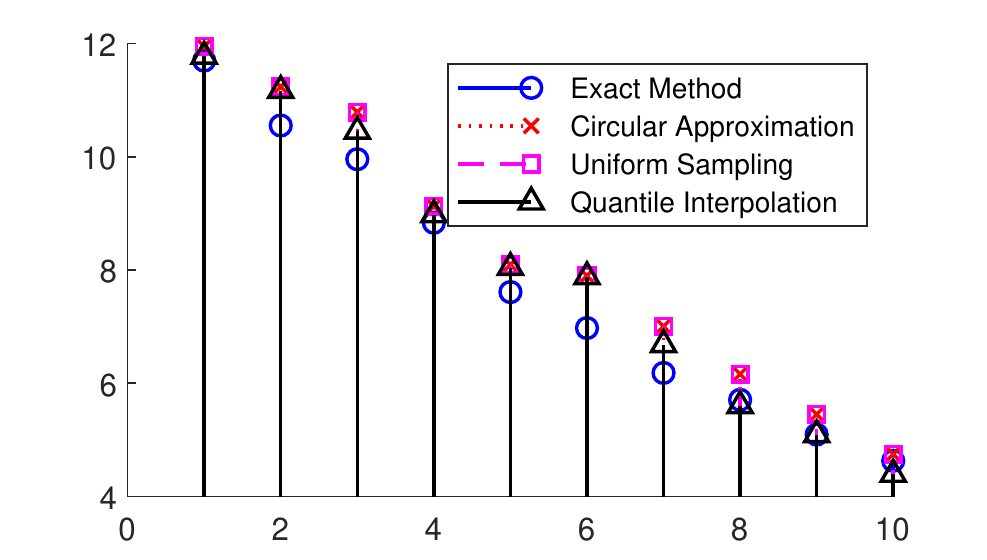}}
\hspace{-12pt}
{\includegraphics[width=0.26\columnwidth]{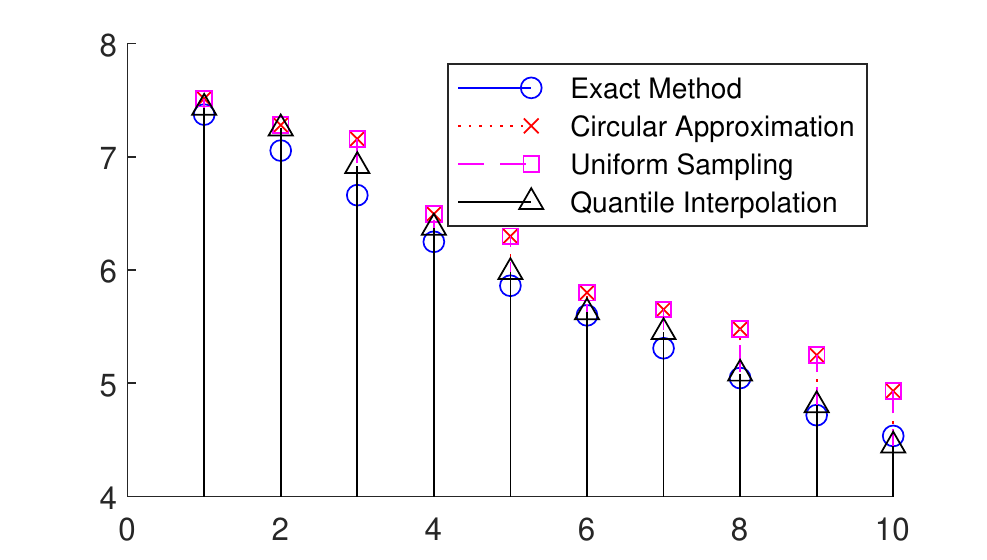}}
\hspace{-12pt}
{\includegraphics[width=0.26\columnwidth]{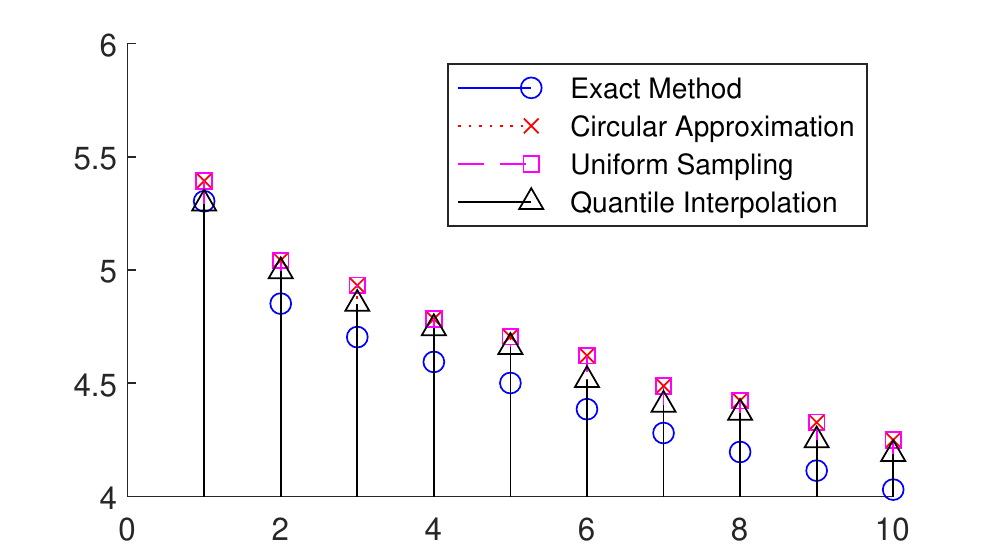}}
\hspace{-12pt}
{\includegraphics[width=0.26\columnwidth]{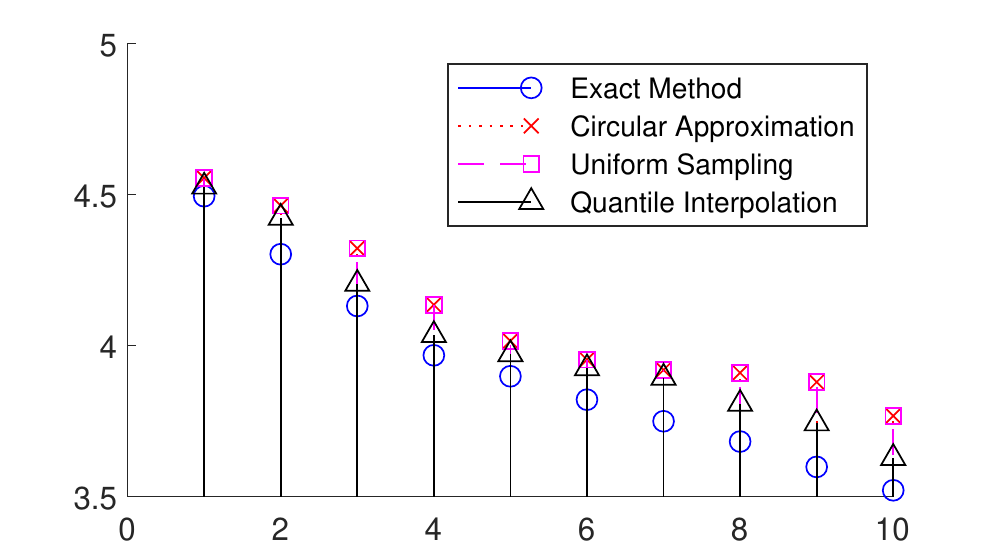}}
\\
\hspace{-12pt}
{\includegraphics[width=0.26\columnwidth]{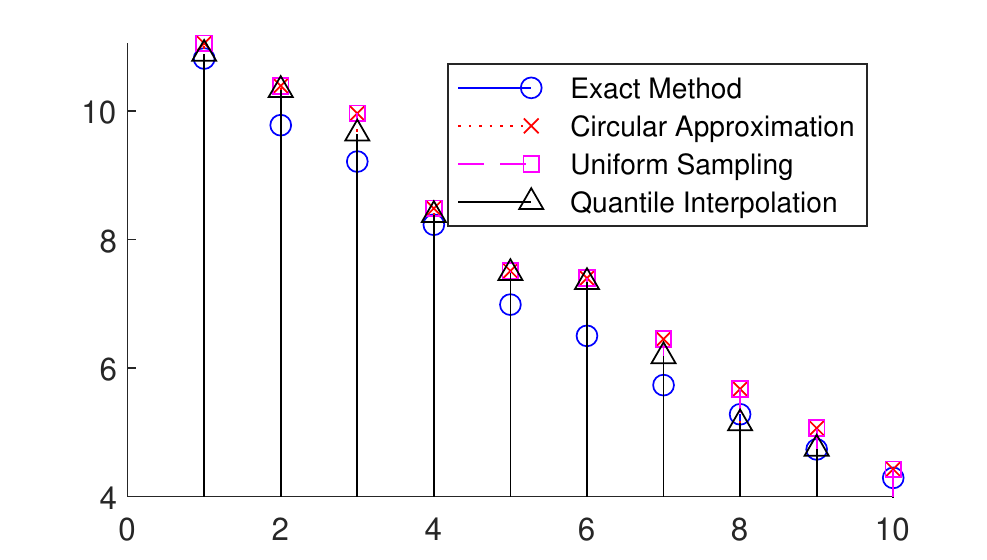}}
\hspace{-12pt}
{\includegraphics[width=0.26\columnwidth]{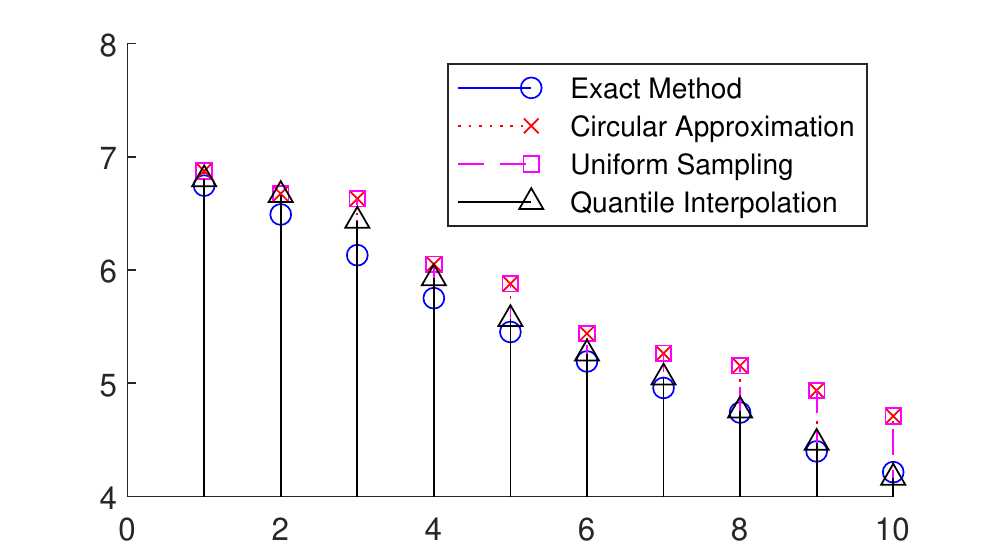}}
\hspace{-12pt}
{\includegraphics[width=0.26\columnwidth]{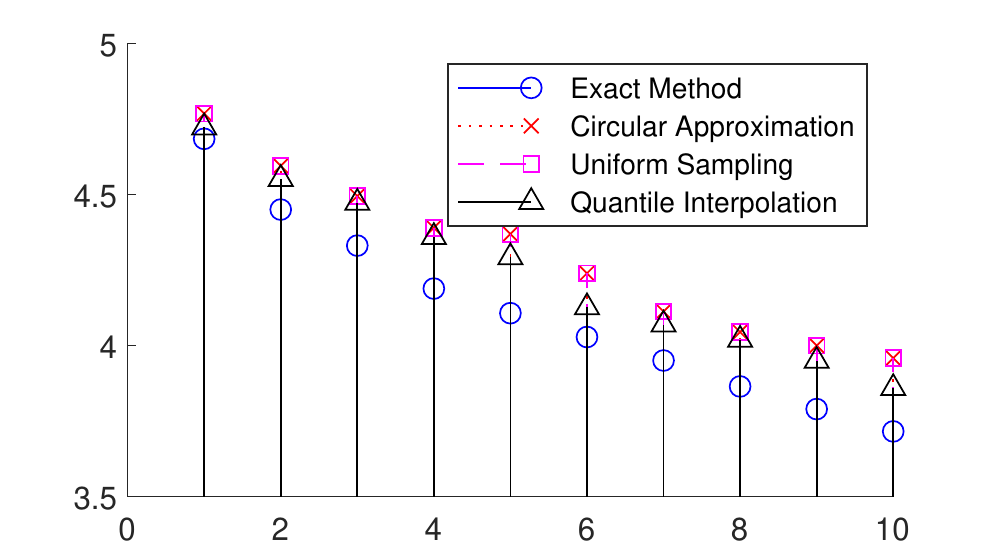}}
\hspace{-12pt}
{\includegraphics[width=0.26\columnwidth]{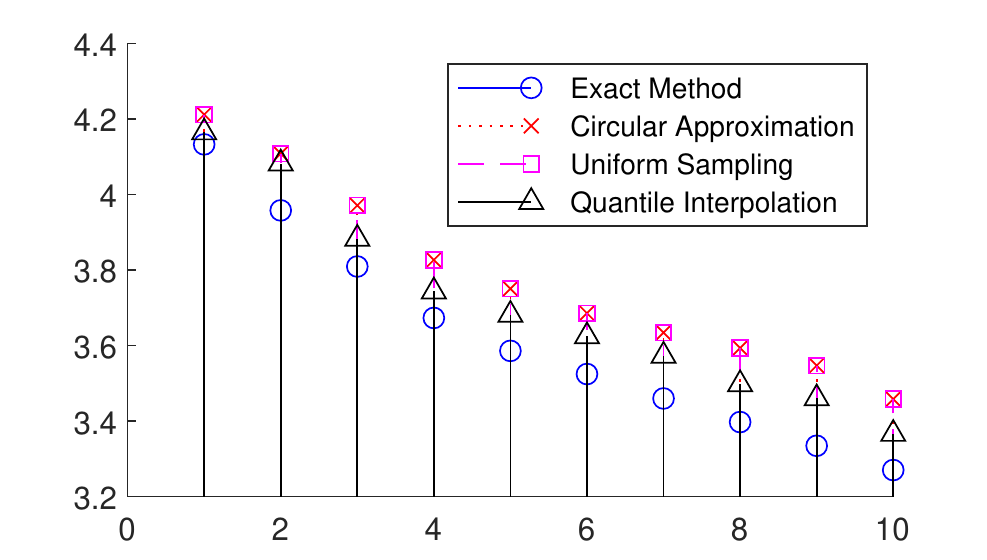}}
\hspace{-12pt}
\begin{center}
\caption{Exact and approximated singular values of linear convolutional layers arranged in descending order. For illustration, only 10 singular values are plotted, each of which represents the behavior of a cluster of singular values. Four types of convolutional filters of pre-trained VGG16 (first row) and VGG19 (second row) networks on ImageNet dataset are considered from first to last column with sizes $64 \times 3 \times 3 \times 3$ (conv1\_1),  $64 \times 64 \times 3 \times 3$ (conv1\_2), $128 \times 64 \times 3 \times 3$ (conv2\_1), and $128 \times 128 \times 3 \times 3$ (conv2\_2), respectively. }
\label{fig:Fig-2}
\end{center}
\vskip -0.35in
\end{figure}

Figure \ref{fig:Fig-2} presents the singular values of pretrained VGG models, where most convolutional layers have size $3  \time 3$ filters. With respect to singular values, VGG16 and VGG19 have similar spectral behavior. The improvement of the quantile interpolation over the circular approximation lies in small singular values, while for the largest singular value the improvement is subtle. %
It is also observed that, as the number of input/output channels increases, the singular values are decreasing.

\begin{figure}[t]
\vskip 0.1in
{\includegraphics[width=0.26\columnwidth]{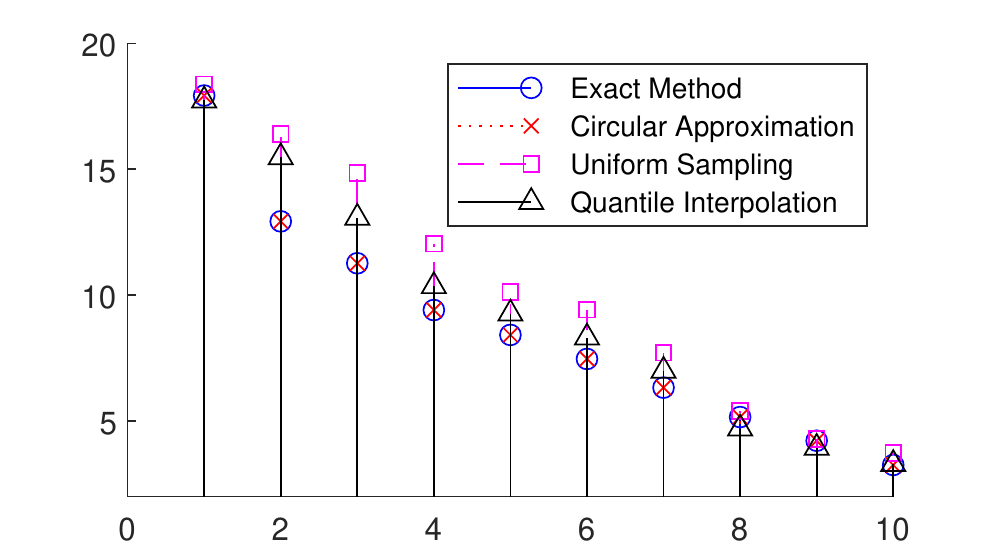}}
\hspace{-12pt}
{\includegraphics[width=0.26\columnwidth]{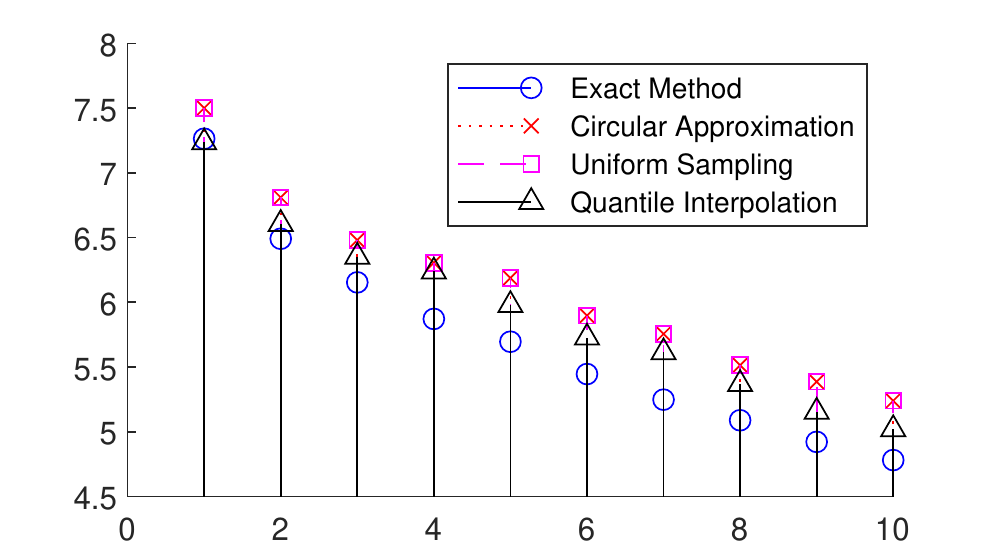}}
\hspace{-12pt}
{\includegraphics[width=0.26\columnwidth]{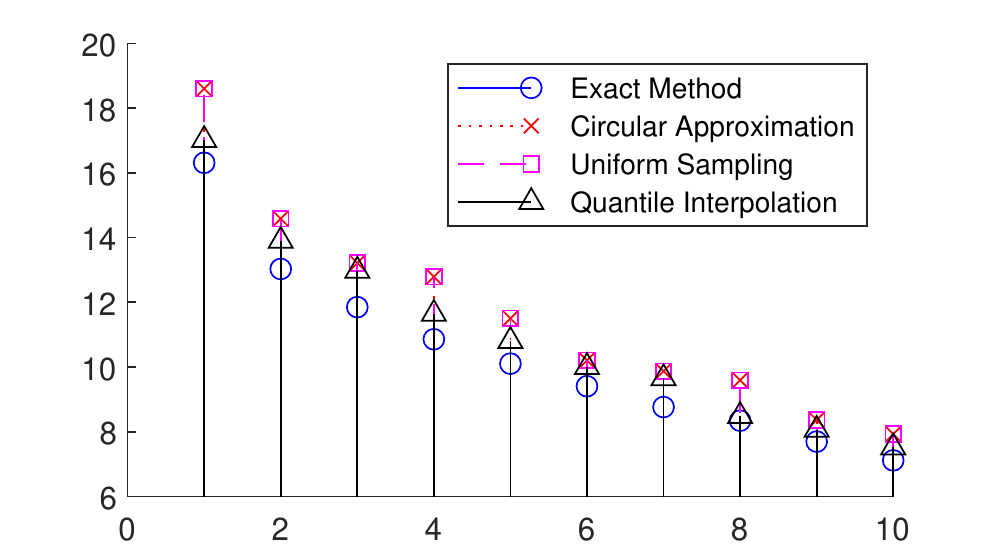}}
\hspace{-12pt}
{\includegraphics[width=0.26\columnwidth]{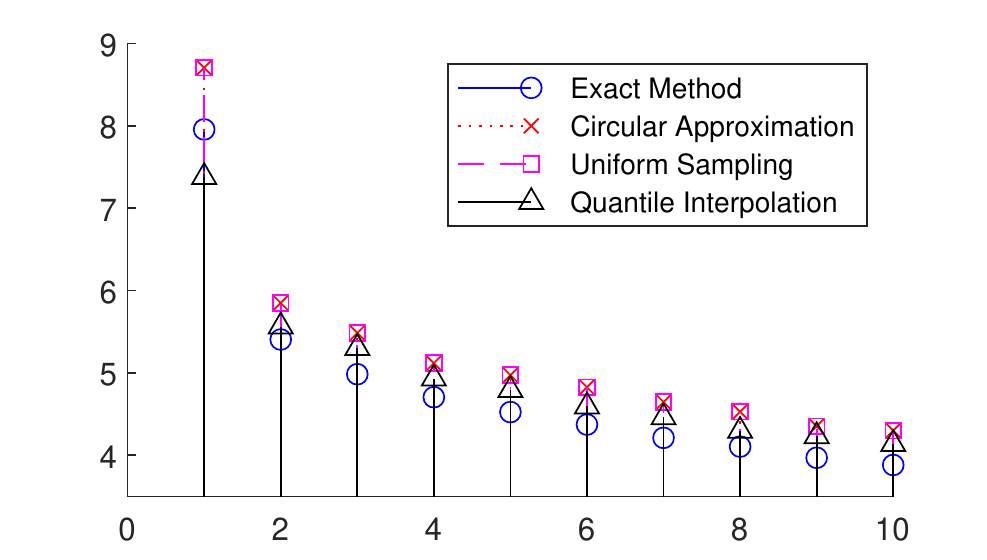}}
\\
\hspace{-12pt}
{\includegraphics[width=0.26\columnwidth]{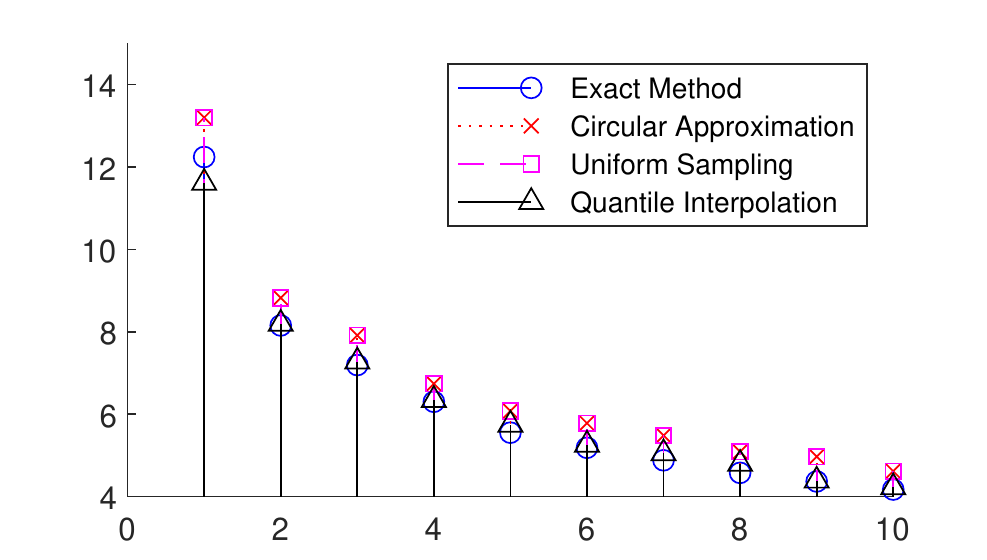}}
\hspace{-12pt}
{\includegraphics[width=0.26\columnwidth]{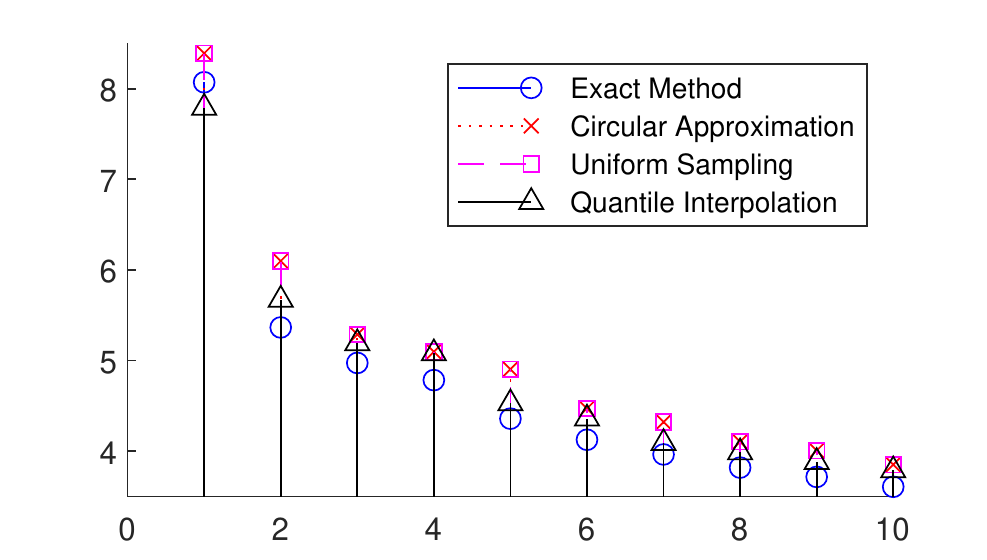}}
\hspace{-12pt}
{\includegraphics[width=0.26\columnwidth]{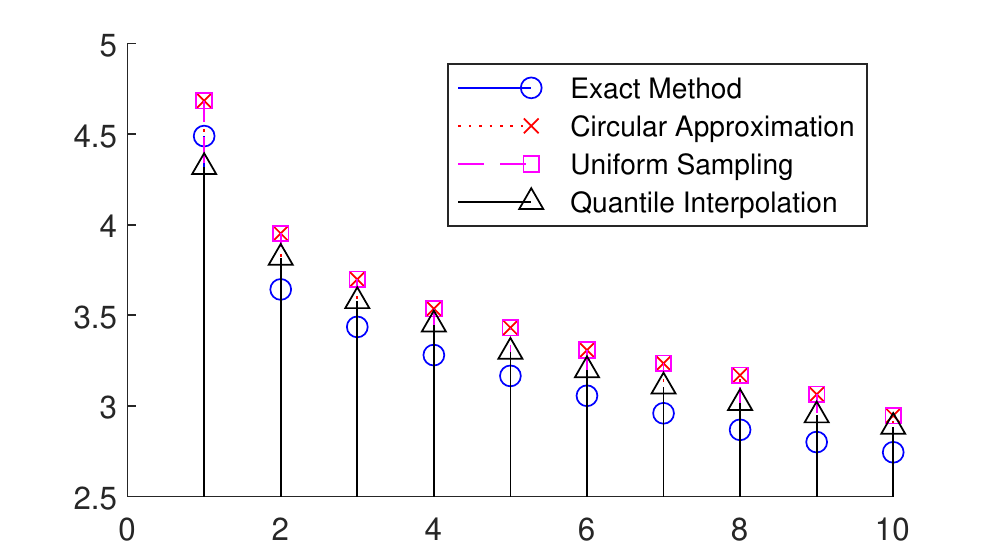}}
\hspace{-12pt}
{\includegraphics[width=0.26\columnwidth]{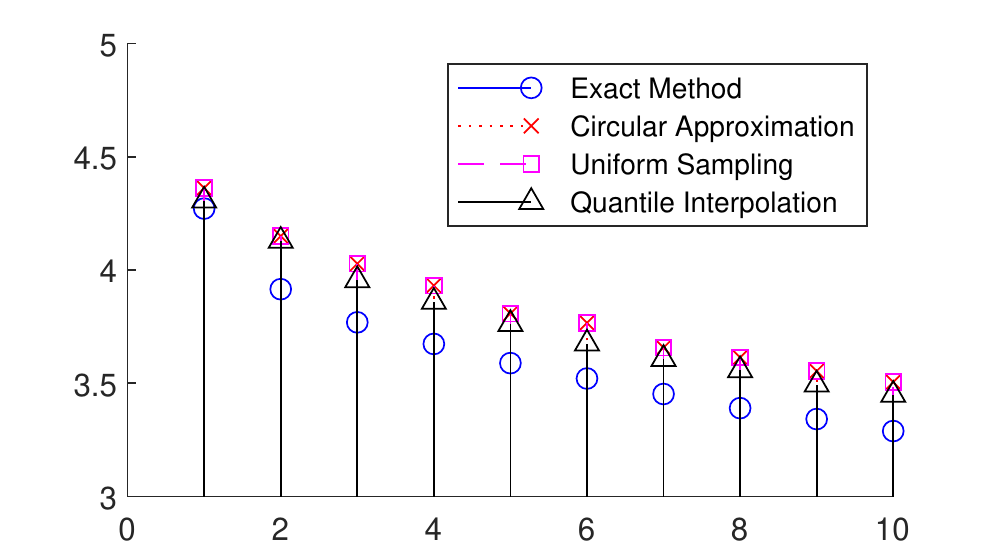}}
\hspace{-12pt}
\begin{center}
\vspace{-0.1in}
\caption{Exact and approximated singular values of linear convolutional layers arranged in descending order. For illustration, only 10 singular values are plotted, each of which represents the behavior of a cluster of singular values. Four types of convolutional filters of pre-trained networks on ImageNet dataset are considered from top left to bottom right with sizes $96 \times 3 \times 11 \times 11$ (AlexNet conv1),  $128 \times 48 \times 5 \times 5$ (AlexNet conv2), $64 \times 3 \times 7 \times 7$ (DenseNet201 conv1), $96 \times 32 \times 5 \times 5$ (GoogLeNet Inception\_3b), $64 \times 48 \times 5 \times 5$ (InceptionResNetv2 conv2d\_8), $64 \times 48 \times 5 \times 5$ (Inceptionv3 conv2d\_8), $64 \times 48 \times 5 \times 5$ (Inceptionv3 conv2d\_15), $64 \times 48 \times 5 \times 5$ (Inceptionv3 conv2d\_22), respectively. }
\label{fig:Fig-3}
\end{center}
\vskip -0.35in
\end{figure}

\begin{figure}[t]
\vskip 0.1in
{\includegraphics[width=0.26\columnwidth]{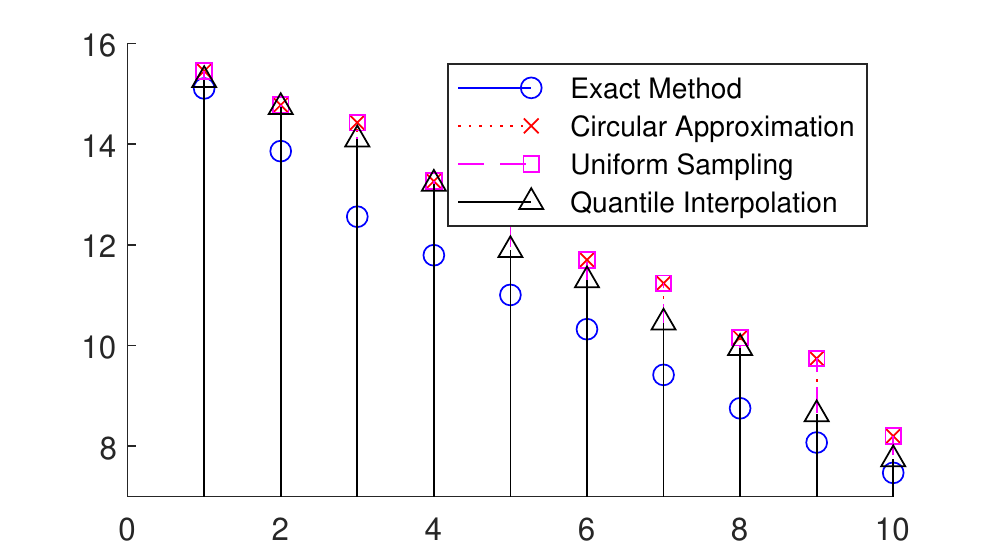}}
\hspace{-12pt}
{\includegraphics[width=0.26\columnwidth]{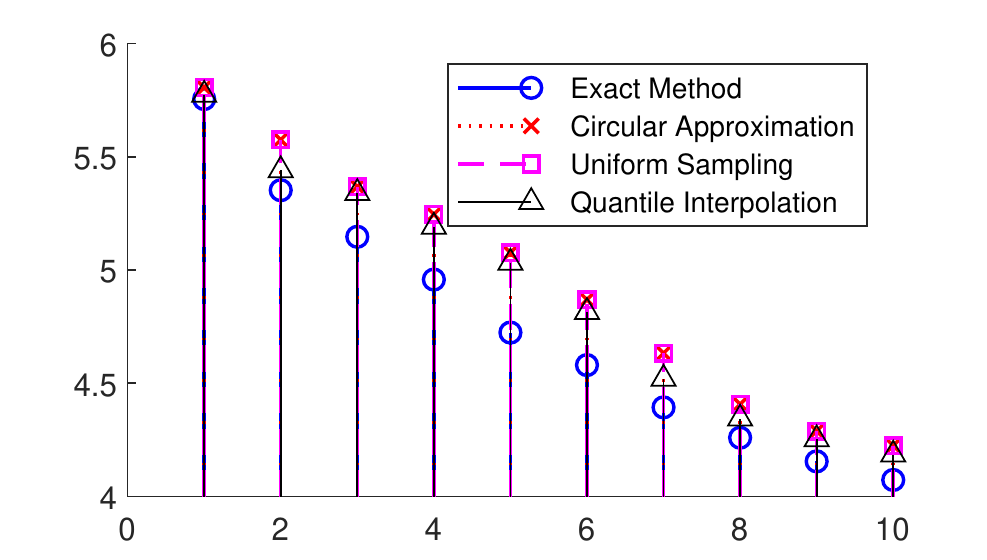}}
\hspace{-12pt}
{\includegraphics[width=0.26\columnwidth]{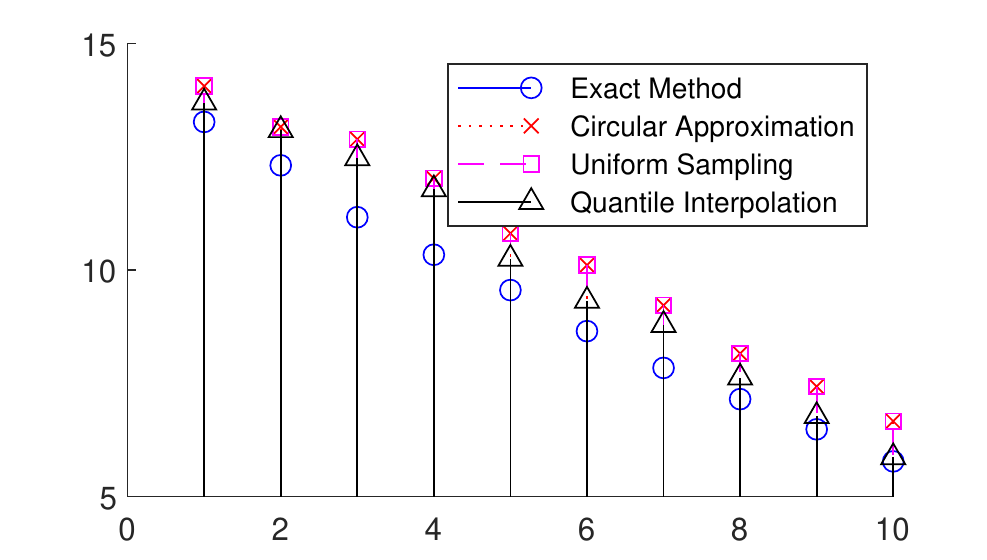}}
\hspace{-12pt}
{\includegraphics[width=0.26\columnwidth]{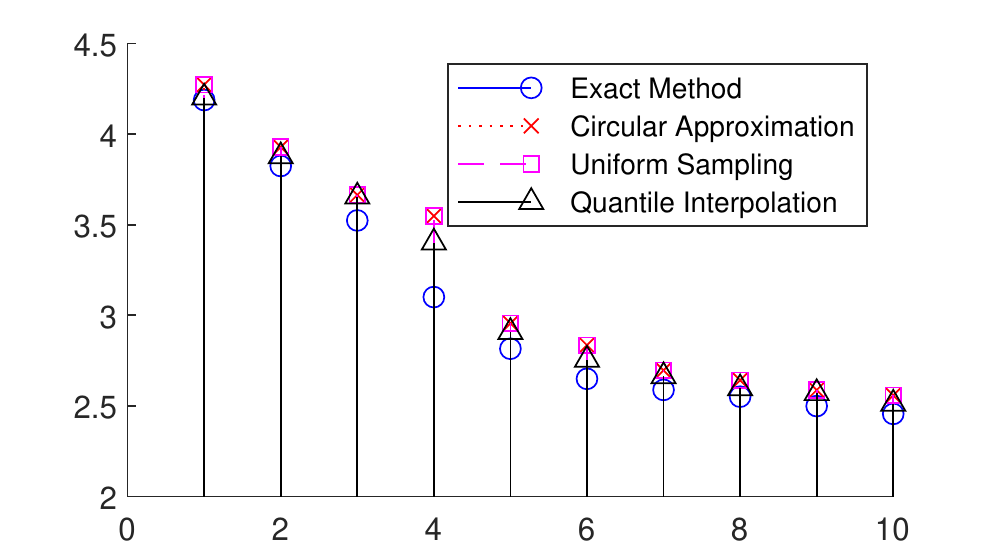}}
\\
\hspace{-12pt}
{\includegraphics[width=0.26\columnwidth]{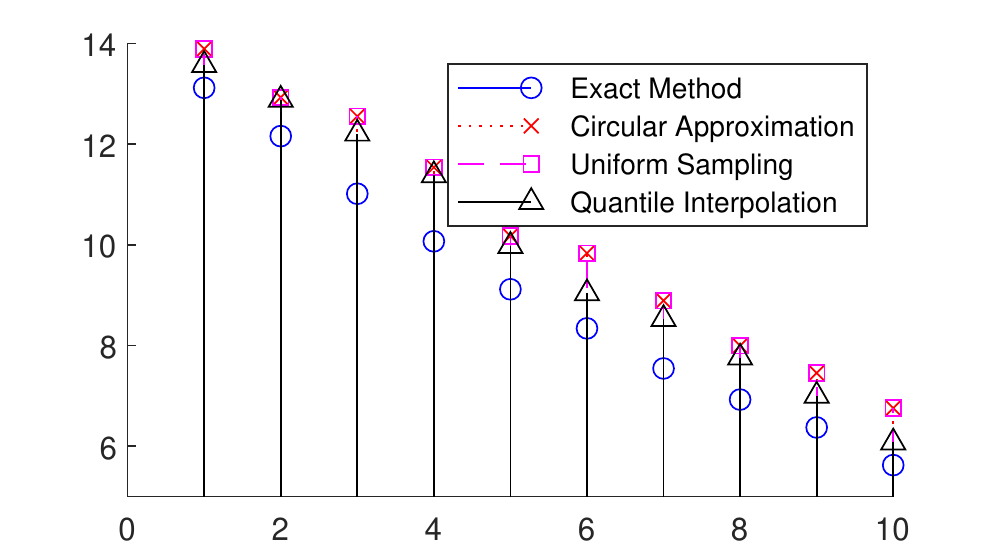}}
\hspace{-12pt}
{\includegraphics[width=0.26\columnwidth]{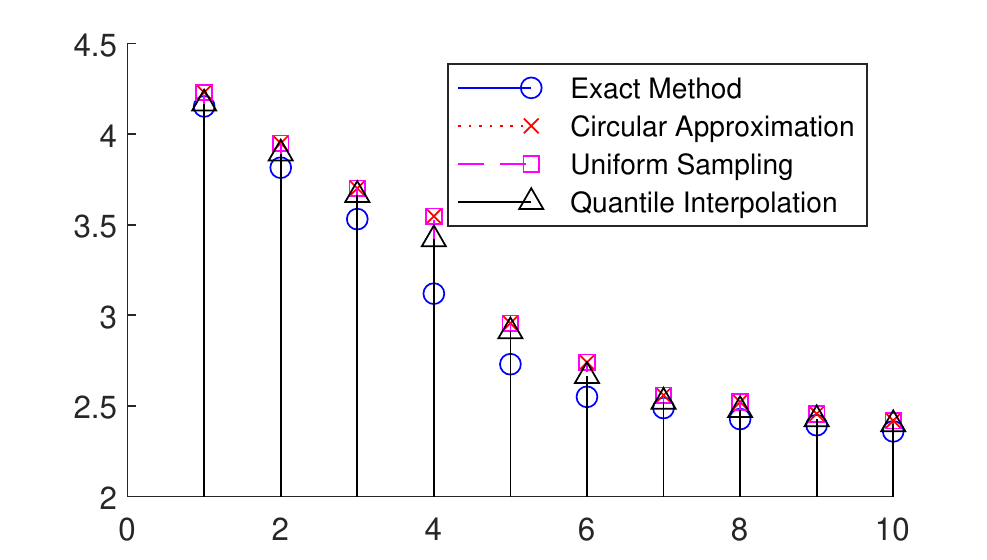}}
\hspace{-12pt}
{\includegraphics[width=0.26\columnwidth]{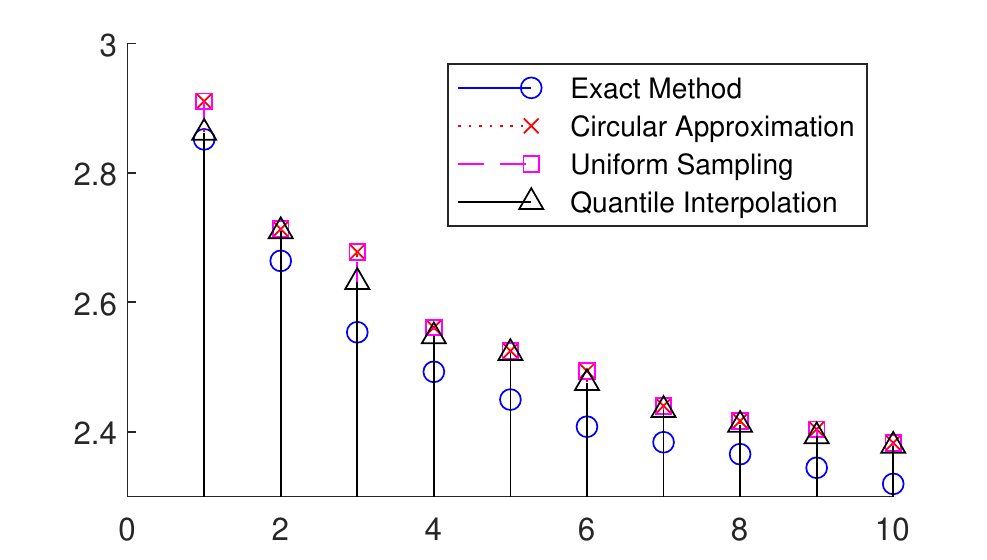}}
\hspace{-12pt}
{\includegraphics[width=0.26\columnwidth]{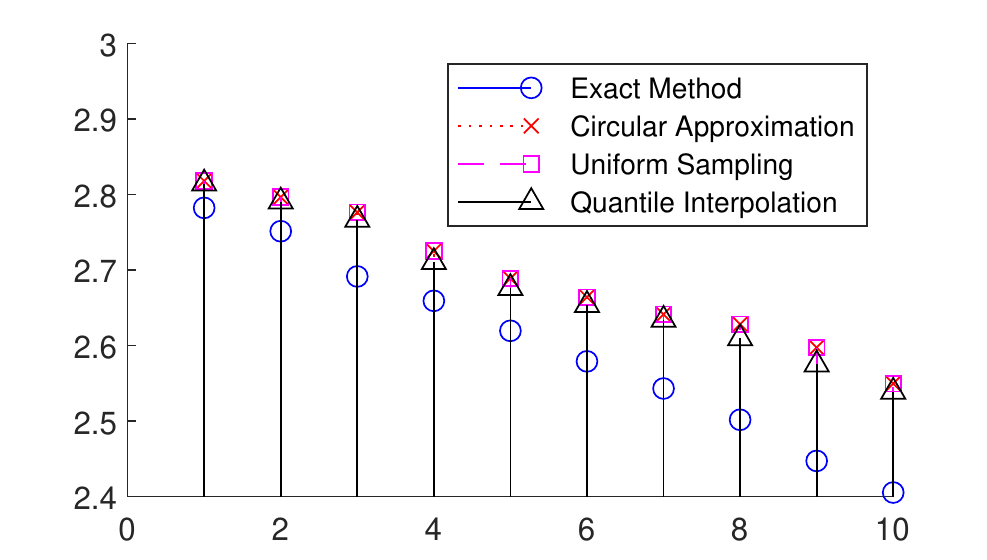}}
\hspace{-12pt}
\begin{center}
\vspace{-0.1in}
\caption{Exact and approximated singular values of linear convolutional layers arranged in descending order. For illustration, only 10 singular values are plotted, each of which represents the behavior of a cluster of singular values. Four types of convolutional filters of pre-trained ResNets are considered from top left to bottom right with sizes $64 \times 3 \times 7 \times 7$ (ResNet-18 conv1),  $64 \times 64 \times 3 \times 3$ (ResNet-18 res2a), $64 \times 3 \times 7 \times 7$ (ResNet-50 conv1), and $64 \times 64 \times 3 \times 3$ (ResNet-50 res2a), $64 \times 3 \times 7 \times 7$ (ResNet-101 conv1), $64 \times 64 \times 3 \times 3$ (ResNet-101 res2a), $64 \times 64 \times 3 \times 3$ (ResNet-101 res2b), $64 \times 64 \times 3 \times 3$ (ResNet-101 res2c), respectively. }
\label{fig:Fig-4}
\end{center}
\vskip -0.35in
\end{figure}

Figure \ref{fig:Fig-3} presents the singular values from another set of pre-trained networks, for which we select the filters with larger size, i.e., $h=w=5,7,11$. It is observed that the improvement of the quantile interpolation over the circular approximation is enhanced for the filters with larger size. The improvement of the largest singular values is more significant than VGG networks.

Figure \ref{fig:Fig-4} is dedicated to ResNets in which we present singular values for the convolutional layers in ResNet-18, ResNet-50, and ResNet-101. For the filters with size $3 \times 3$, the improvement of the largest singular value is subtle, while the smaller singular values contribute much on the improvement, as observed in VGG networks. For the filter with size $7 \times 7$, the major improvement of singular values occurs in the intermediate ones. The negative result is that it seems both circular approximation and quantile interpolation do not work well for ResNet-101 the convolutional layer res2c. The quantile approach with linear interpolation relies much on the circular approximation - if the latter does not work well, so does the former very likely. It may require the nonlinear interpolation.

Table \ref{tab:running} summarizes the accuracy and running time of different singular value computation methods on various convolutional layers of pre-trained networks on ImageNet dataset. The experiments have been conducted in MATLAB on an HP EliteBook (Intel i5 CPU with 8G RAM). The sizes of different filters can be referred as above. The numbers ``$a/b$'' should read as $a\%$ difference from the exact method and the running time is $b$ seconds. Note that for the accuracy we consider the sum of all singular values and use the exact method as the reference. For instance, for the convolutinal layer named ``conv1\_1'' in VGG16 model, the running time of the exact method is 0.1223 seconds, compared with 0.0106 seconds and 0.0245 seconds for uniform sampling and quantile interpolation methods, respectively. For both the circular approximation and uniform sampling, the accuracy of all singular values is 7.51\% larger than the exact value computed by the exact method, while the quantile interpolation reduces such a difference to 1.47\%.

It can be observed from experimental results that quantile approximation always outperforms the circular approximation by more than 5\% in overall accuracy for most cases at the expense of extra running time. The running time is negligible compared with that using SVD in the exact method.
The most significant improvement in approximation accuracy is for AlexNet conv1 with filter size $96\times3\times11\times11$. This confirms our observation earlier that the quantile interpolation has more substantial improvement over the circular approximation for the larger filter size. The least improvement happens for ResNet-101 Res2b/c with filter size $3 \times 3$.

\begin{table}[t]
\caption{Comparison of approximation accuracy and running time.}
\label{tab:running}
\vskip -0.1in
\begin{center}
\begin{small}
\begin{sc}
\begin{tabular}{lcccc}
\toprule
Filter & Exact &   Sampling & Quantile\\
\midrule
VGG16 conv1\_1 ($64\times3\times3\times3$) & -/0.1233 &  7.51\%/0.0106 & 1.47\%/0.0245 \\
VGG16 conv1\_2 ($64\times64\times3\times3$) & -/120.53 &  7.79\%/0.1043 & 1.96\%/0.2269 \\
VGG16 conv2\_1 ($128\times64\times3\times3$)& -/145.10 &  7.31\%/0.3483 & 3.32\%/0.5894 \\
VGG16 conv2\_2 ($128\times128\times3\times3$) & -/958.84 &  7.86\%/0.5402 & 3.60\%/0.9414 \\
VGG19 conv1\_1 ($64\times3\times3\times3$) &-/0.1591 &  7.37\%/0.0079 & 1.36\%/0.0173 \\
VGG19 conv1\_2 ($64\times64\times3\times3$) &-/120.87 &  7.81\%/0.0956 & 2.08\%/0.2482 \\
VGG19 conv2\_1 ($128\times64\times3\times3$) & -/141.39 &  7.32\%/0.2218 & 3.47\%/0.4453 \\
VGG19 conv2\_2 ($128\times128\times3\times3$) & -/964.5 &  7.88\%/0.4840 & 3.90\%/0.7565 \\
AlexNet conv1 ($96\times3\times11\times11$) & -/0.3429 &  22.47\%/0.1218 & 10.56\%/0.1252 \\
AlexNet conv2 ($128\times48\times5\times5$) & -/44.03& 9.65\%/0.2835 & 3.90\%/0.4551\\
DenseNet201 conv1 ($64\times3\times7\times7$) & -/0.1481 &  11.79\%/0.0347 & 5.11\%/0.0429 \\
GoogLeNet conv1 ($64\times3\times7\times7$) & -/0.1352 &  14.13\%/0.0248 & 6.71\%/0.0340 \\
GoogLeNet Inception\_3a ($32\times16\times5\times5$) & -/2.6050 &  10.56\%/0.0341 & 5.30\%/0.0686 \\
GoogLeNet Inception\_3b ($96\times32\times5\times5$) & -/15.869 &  13.53\%/0.1138 & 8.68\%/0.1764 \\
InceptionResNetv2 conv2d\_8 ($64\times48\times5\times5$) & -/22.997 &  10.15\%/0.1218 & 4.81\%/0.2149 \\
Inceptionv3 conv2d\_8 ($64\times48\times5\times5$) & -/18.055 &  9.36\%/0.1105 & 4.99\%/0.2049 \\
Inceptionv3 conv2d\_15 ($64\times48\times5\times5$) & -/18.281 &  10.32\%/0.1143 & 6.40\%/0.2054 \\
Inceptionv3 conv2d\_22 ($64\times48\times5\times5$) & -/18.656 &  10.17\%/0.1808 & 6.50\%/0.2931 \\
ResNet-18 conv1 ($64\times3\times7\times7$) & -/0.1416 &  12.23\%/0.0240 & 5.74\%/0.0358 \\
ResNet-18 res2a ($64\times64\times3\times3$) & -/122.55 &  5.59\%/0.0886 & 2.46\%/0.2086 \\
ResNet-50 conv1 ($64\times3\times7\times7$) & -/0.1463 &  13.39\%/0.0291 & 6.78\%/0.0383 \\
ResNet-50 res2a ($64\times64\times3\times3$) & -/133.45 &  6.26\%/0.0918 & 3.08\%/0.2106 \\
ResNet-101 conv1 ($64\times3\times7\times7$) & -/0.1412  & 13.63\%/0.0274 & 7.25\%/0.0395 \\
ResNet-101 res2a  ($64\times64\times3\times3$) & -/131.003 & 6.04\%/0.0984 & 2.94\%/0.2163 \\
ResNet-101 res2b  ($64\times64\times3\times3$) & -/126.58 & 6.66\%/0.0929 & 4.47\%/0.2147 \\
ResNet-101 res2c  ($64\times64\times3\times3$) & -/120.56 &  7.18\%/0.0946 & 5.49\%/0.2144 \\
\bottomrule
\end{tabular}
\end{sc}
\end{small}
\end{center}
\vskip -0.25in
\end{table}

\subsubsection{Weights from Training Process}

Figure \ref{fig:Fig-5} presents the singular value approximation for the weights extracted from the training process of ResNet-20 on CIFAR-10 dataset. We consider the filters of four convolutional layers after 10 and 100 training epochs. It is observed that, as the larger singular values increase with training epochs, the improvement of quantile approach over circular approximation is enlarged, while the improvement of small singular values is moderate during the training. %
It suggests that for ResNet models, while the accuracy of circular approximation is relatively reasonable for smaller singular values, it calls for more accurate approximation methods for larger singular values.

\begin{figure}[t]
\vskip 0.1in
{\includegraphics[width=0.26\columnwidth]{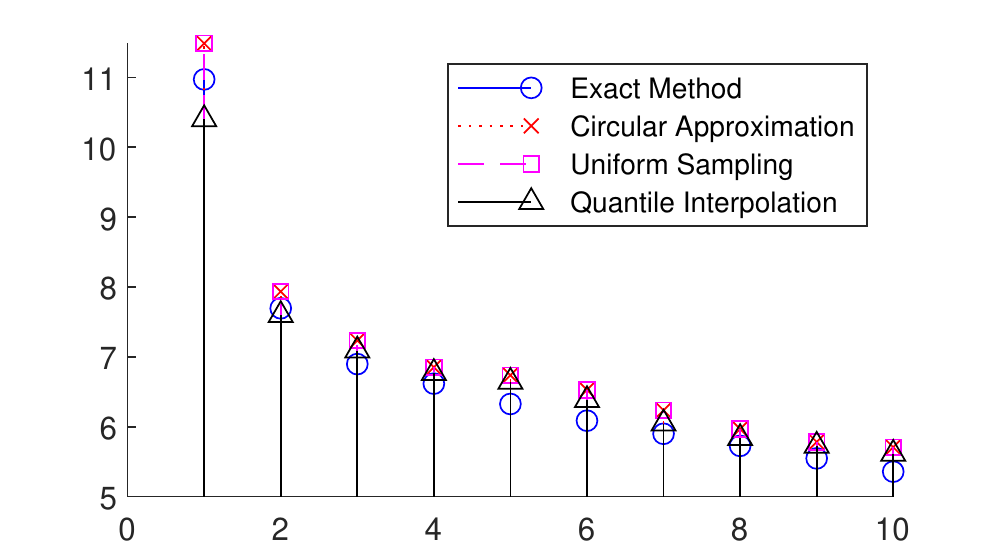}}
\hspace{-12pt}
{\includegraphics[width=0.26\columnwidth]{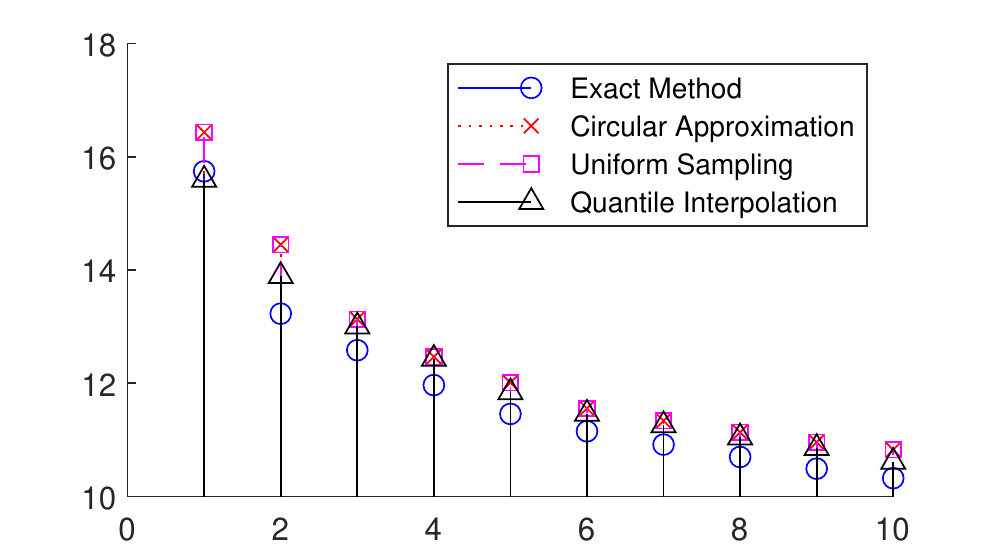}}
\hspace{-12pt}
{\includegraphics[width=0.26\columnwidth]{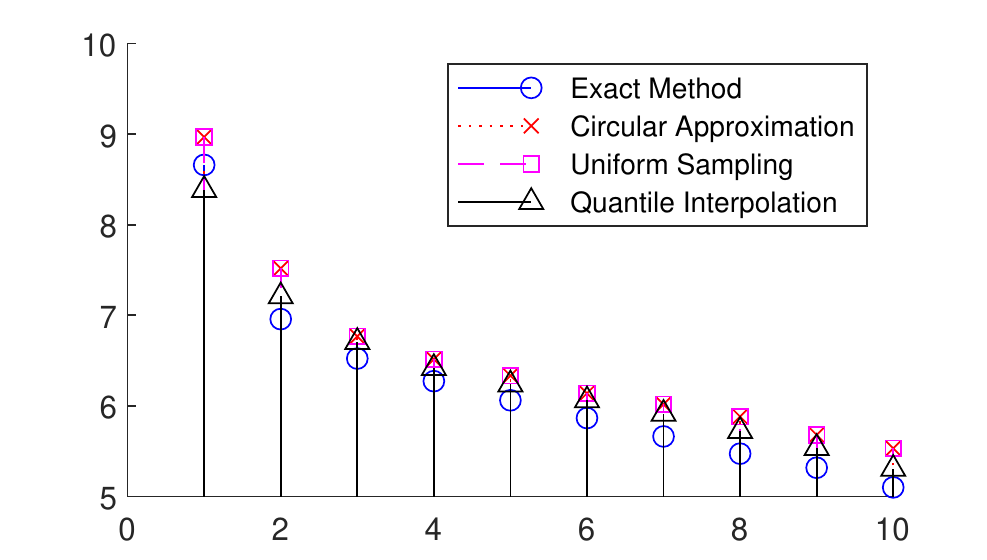}}
\hspace{-12pt}
{\includegraphics[width=0.26\columnwidth]{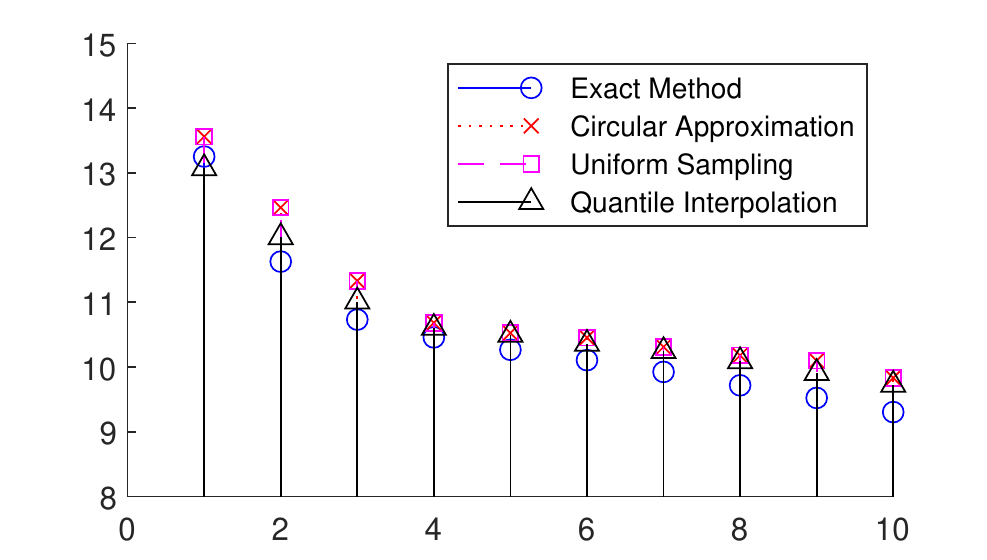}}
\hspace{-12pt}
\hspace{-12pt}
{\includegraphics[width=0.26\columnwidth]{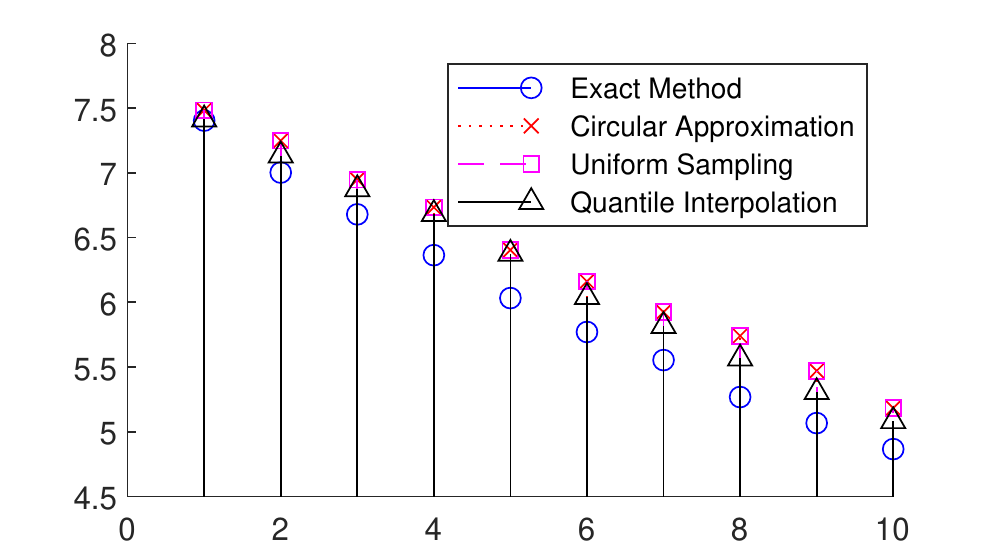}}
\hspace{-12pt}
{\includegraphics[width=0.26\columnwidth]{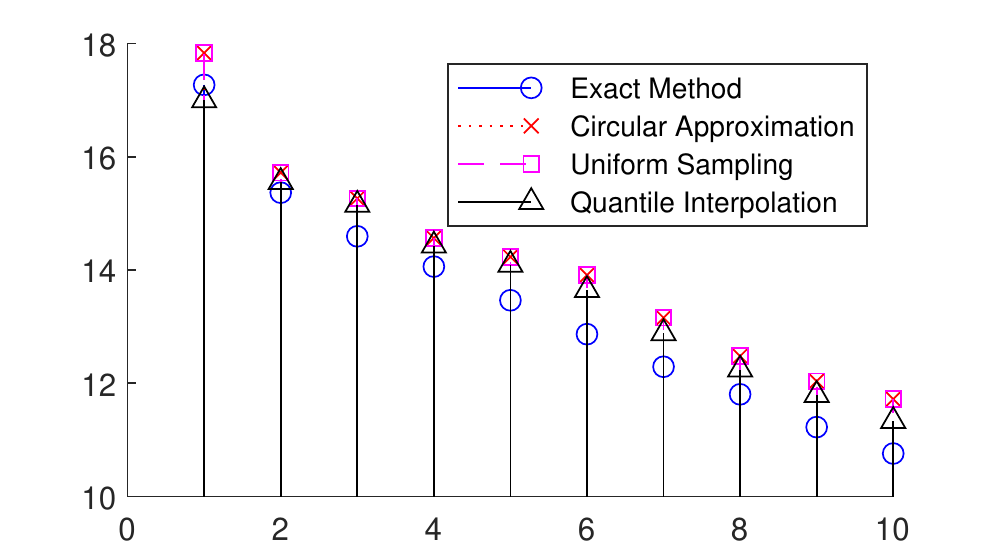}}
\hspace{-12pt}
{\includegraphics[width=0.26\columnwidth]{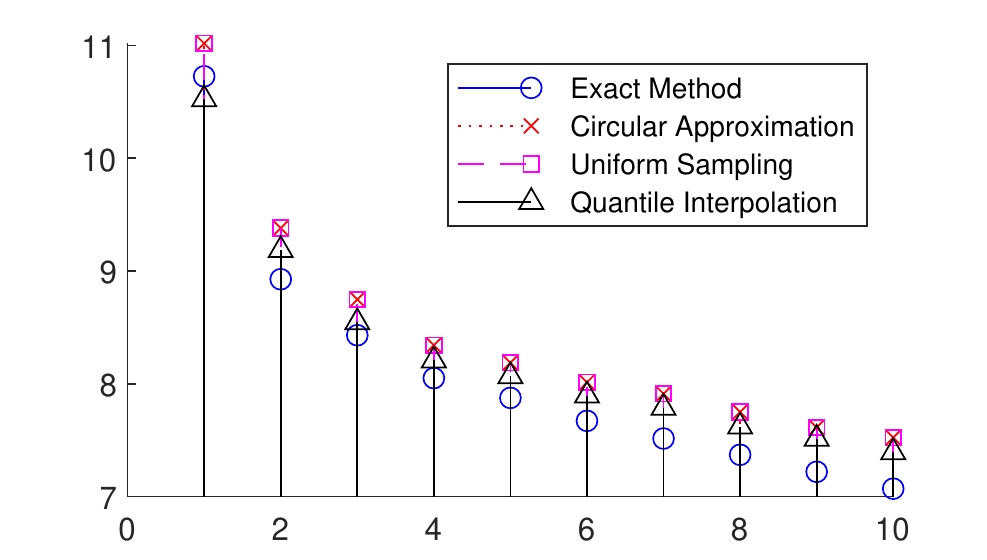}}
\hspace{-12pt}
{\includegraphics[width=0.26\columnwidth]{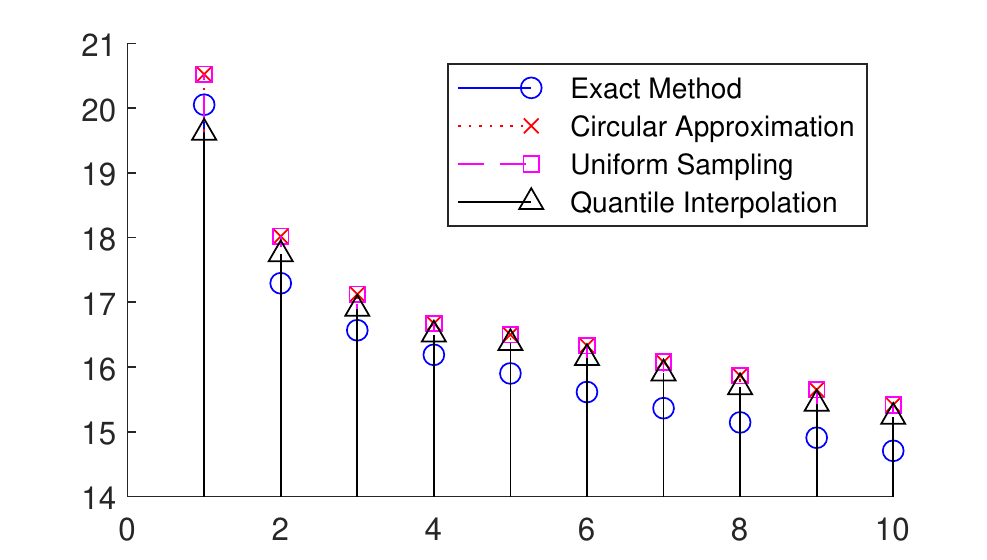}}
\hspace{-12pt}
\begin{center}
\caption{Exact and approximated singular values of linear convolutional layers arranged in descending order. For illustration, only 10 singular values are plotted, each of which represents the behavior of a cluster of singular values. Four types of convolutional filters of the ResNet-20 network trained on CIFAR-10 dataset are considered with sizes $16 \times 16 \times 3 \times 3$ (layer1-conv1),  $16 \times 16 \times 3 \times 3$ (layer1-conv2), $32 \times 16 \times 3 \times 3$ (layer2-conv1), and $64 \times 64 \times 3 \times 3$ (layer3-conv2), respectively. For each filter, the plot after 10 training epochs comes first, followed by the one after 100 training epochs.}
\label{fig:Fig-5}
\end{center}
\vskip -0.35in
\end{figure}

\subsection{Spectral Norm Bounding}
\label{sec:experiments-snr}
\subsubsection{Accuracy vs. Running Time}
As did in the main text, we evaluate the accuracy of spectral norm bound \eqref{eq:upper-bound-F-2norm}, \eqref{eq:upperibound-oneinfnorm}, and \eqref{eq:upperibound-2norm} against the running time for different pre-trained convolutional layers with input size $10 \times 10$. The experiments are conducted on an HP EliteBook with Intel i5 CPU.
Table \ref{tab:running-spectral} presents the accuracy and the running time for different convolutional layers, where $a/b$ reads as the spectral norm bound is $a$ times of the circular approximation and the computation takes $b$ seconds. 
We have the similar observations as those in the main text. In particular, the spectral norm bound \eqref{eq:upperibound-2norm} has comparable accuracy as \eqref{eq:upper-bound-F-2norm} but the computation of the former takes much less time than that of the latter. The computation of matrix norms uses the function ``norm'' in MATLAB.

\begin{table}[t]
\caption{Comparison of spectral norm bounding accuracy and running time.}
\label{tab:running-spectral}
\vskip -0.1in
\begin{center}
\begin{small}
\begin{sc}
\begin{tabular}{lcccc}
\toprule
Filter  & \eqref{eq:upper-bound-F-2norm} &  \eqref{eq:upperibound-oneinfnorm} & \eqref{eq:upperibound-2norm}\\
\midrule
VGG16 conv1\_1 ($64\times3\times3\times3$) &  1.3974/0.0407 & 1.6218/0.0221 & 1.9398/0.0011 \\
VGG16 conv1\_2 ($64\times64\times3\times3$) &  1.3726/0.0954 & 2.5652/0.0600 & 1.6452/0.0037 \\
VGG16 conv2\_1 ($128\times64\times3\times3$) & 1.3146/0.1336 & 3.1937/0.0771 & 2.0145/0.0064 \\
VGG16 conv2\_2 ($128\times128\times3\times3$) & 1.3680/0.2517 & 4.4693/0.1715 & 1.9320/0.0136 \\
VGG19 conv1\_1 ($64\times3\times3\times3$) & 1.4092/0.0200 & 1.6601/0.0194 & 1.9479/0.0012 \\
VGG19 conv1\_2 ($64\times64\times3\times3$)  & 1.3791/0.0601 & 2.6451/0.0520 & 1.6618/0.0026 \\
VGG19 conv2\_1 ($128\times64\times3\times3$) & 1.3995/0.1239 & 3.1582/0.0908 & 2.0804/0.0057 \\
VGG19 conv2\_2 ($128\times128\times3\times3$) & 1.3692/0.2485 & 4.2432/0.1524 & 1.9724/0.0122 \\
AlexNet conv1 ($96\times3\times11\times11$)  & 2.9577/0.0230 & 2.3555/0.1005 & 4.9880/0.0032\\
AlexNet conv2 ($128\times48\times5\times5$) & 2.0185/0.1143 & 3.7714/0.1284 & 2.6927/0.0105\\
DenseNet201 conv1 ($64\times3\times7\times7$) & 2.3323/0.0114 & 2.0559/0.0421 & 3.4654/0.0009 \\
GoogLeNet conv1 ($64\times3\times7\times7$) & 2.4652/0.0136 & 2.3505/0.0328 & 3.8639/0.0007 \\
GoogLeNet Inception\_3a ($32\times16\times5\times5$) & 1.6482/0.0123 & 2.0477/0.0383 & 2.6716/0.0012 \\
GoogLeNet Inception\_3b ($96\times32\times5\times5$)  & 1.3553/0.0546 & 2.5794/0.0806 & 2.1185/0.0035 \\
InceptionResNetv2 conv2d\_8 ($64\times48\times5\times5$) & 1.3767/0.0643 & 2.6342/0.0958 & 1.7666/0.0046 \\
Inceptionv3 conv2d\_8 ($64\times48\times5\times5$) & 1.6534/0.0665 & 2.8253/0.0728 & 2.3093/0.0051 \\
Inceptionv3 conv2d\_15 ($64\times48\times5\times5$) & 2.0562/0.0568 & 3.2000/0.0891 & 2.9470/0.0052 \\
Inceptionv3 conv2d\_22 ($64\times48\times5\times5$) & 2.2887/0.0552 & 4.1094/0.0840 & 3.8964/0.0071 \\
ResNet-18 conv1 ($64\times3\times7\times7$) & 3.0015/0.0130 & 2.1353/0.0251 & 4.3294/0.0009 \\
ResNet-18 res2a ($64\times64\times3\times3$) & 1.6284/0.0905 & 3.2053/0.0543 & 2.1965/0.0035 \\
ResNet-50 conv1 ($64\times3\times7\times7$) & 2.9348/0.0162 & 2.1886/0.0336 & 3.9626/0.0004 \\
ResNet-50 res2a ($64\times64\times3\times3$) & 1.3949/0.0609 & 2.9974/0.0451 & 2.0898/0.0044 \\
ResNet-101 conv1 ($64\times3\times7\times7$) & 2.9633/0.0095 & 1.9349/0.0302 & 3.8502/0.0004 \\
ResNet-101 res2a  ($64\times64\times3\times3$) & 1.4508/0.0615 & 2.9684/0.0500 & 2.1092/0.0032 \\
ResNet-101 res2b  ($64\times64\times3\times3$) & 1.7636/0.0631 & 3.3880/0.0481 & 2.5061/0.0026 \\
ResNet-101 res2c  ($64\times64\times3\times3$) & 1.6225/0.0643 & 3.5631/0.0449 & 2.3516/0.0027 \\
\bottomrule
\end{tabular}
\end{sc}
\end{small}
\end{center}
\vskip -0.15in
\end{table}

\subsubsection{Regularization}
We use spectral norm bounds as regularizers during the training of ResNet-20 model on CIFAR-10 dateset.
According to the accuracy and running time of different spectral norm bounds in Table \ref{tab:running-spectral}, we place our focus on the first \eqref{eq:upper-bound-F-2norm} and the third bounds \eqref{eq:upperibound-2norm} for spectral regularization. 

Given the training data samples $\{(\xv_i,y_i)\}_{i=1}^N$ drown from an unknown distribution of $(\xv,y)$ for training an $L$-layer deep neural network model $y=f_{\Theta}(\xv)$ with parameters $\Theta$, the spectral regularization is to minimize the following objective function
\begin{align}
    \min_{\Theta} \ \E_{(\xv,y)} \ell (f_{\Theta}(\xv),y) + \beta \sum_{j=1}^{L} R^u_j
\end{align}
where $\ell(f)$ is the loss function of the model for training, $R^u_j$ is a regularization term using the spectral norm upper bounds of the $j$-th layer, e.g.,  \eqref{eq:upper-bound-F-2norm}-\eqref{eq:upperibound-2norm}, and $\beta>0$ is a constant to balance between the loss function and the spectral norm regularizer.

In the experiments, the cross entropy function is chosen as the loss function. For the $j$-th convolutional layer, the regularization term $R^u_j$ is the spectral norm upper bounds chosen from \eqref{eq:upper-bound-F-2norm} with $R^u_j=\sqrt{hw}\min\{\norm{\Rm}_2, \norm{\Lm}_2\}$ and \eqref{eq:upperibound-2norm} with $R^u_j=\sum_k \sum_l \norm{\Tm_{k,l}}_2$, respectively. For the fully-connected layers, $R^u_j$ is directly chosen as the exact spectral norm of the weight matrices. As both the upper bounds in \eqref{eq:upper-bound-F-2norm} and \eqref{eq:upperibound-2norm} are in the form of spectral norm, we adopt power method to compute it in the forward propagation. 

As shown in the proof of Theorem \ref{theorem:bounding-norm}, $\Rm$ and $\Lm$ are reshaped matrices of the convolutional filter $\Km$ with sizes $hc_{out} \times wc_{in}$ and $wc_{out} \times hc_{in}$, respectively, in contrast to the set of $hw$ matrices $\{\Tm_{k,l}\}$ with size $c_{out} \times c_{in}$ each rearranged from $\Km$. For a matrix $\Am \in \RR^{m \times n}$, the computational complexity of power method is $O(mn)$. While both bounds \eqref{eq:upper-bound-F-2norm} and \eqref{eq:upperibound-2norm} have the same level of computational complexity $O(hwc_{out}c_{in})$, it turns out computing \eqref{eq:upperibound-2norm} with power method is much faster as the matrices has smaller size.
In the backward propagation, the derivative of spectral norms of a matrix $\Am$ can be computed as
$
    \nabla_{\Am} \norm{\Am}_2 =  \vv_1 \uv_1^\T 
$
where $\uv_1$ and $\vv_1$ are the left and right singular vectors corresponding to the largest singular value, respectively.
Such a derivative is used to update weights for SGD in the backward propagation.

Due to limited computing resource, we focus on the training and testing of the ResNet-20 model on CIFAR-10 dataset. The ResNet-20 model has 20 convolutional layers, most of which have a $3 \times 3$ filter. The CIFAR-10 dataset consists of 50,000 training and 10,000 testing images with size $32 \times 32$ in 10 classes.
The batch size is 128, and the learning rate is initialized as 0.1 and changed to 0.01 after 100 training epochs. The weight decay is set to 0, and the momentum is 0.9. 
The final prediction accuracy is collected after in total 150 training epochs.

\begin{figure}[t]
\vskip 0.1in
\hspace{-12pt}
{\includegraphics[width=0.55\columnwidth]{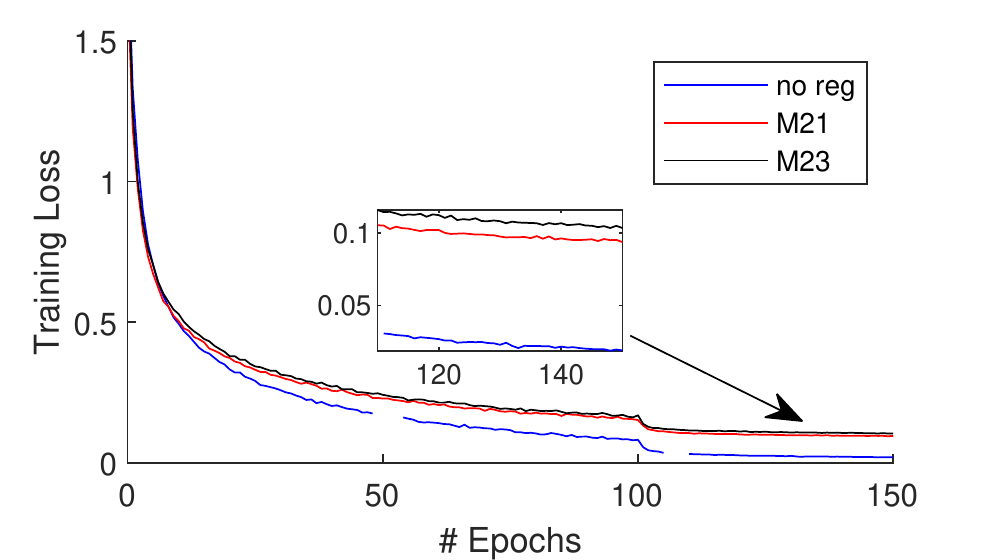}}
\hspace{-12pt}
{\includegraphics[width=0.55\columnwidth]{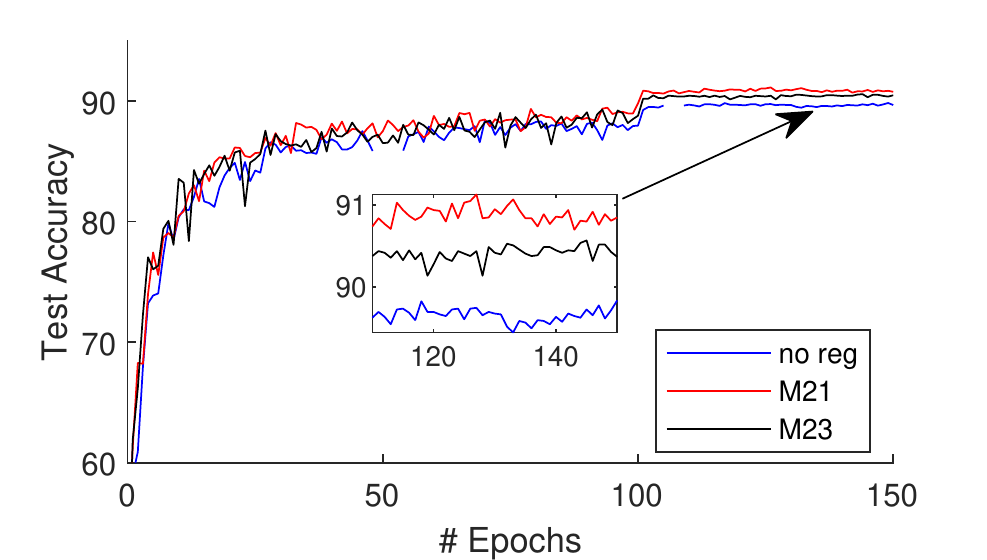}}
\hspace{-12pt}
\begin{center}
\vspace{-0.1in}
\caption{The training loss (left) and test accuracy (right) versus the number of training epochs for ResNet-20 on CIFAR-10 dataset with or without  regularization using spectral norm bounds \eqref{eq:upper-bound-F-2norm} and \eqref{eq:upperibound-2norm} with $\beta=0.0014$.}
\label{fig:Fig-6}
\end{center}
\vskip -0.35in
\end{figure}

For comparison, we use the case with no regularization ($\beta=0$), which has a test accuracy 89.67\%, and the case with regularization ($\beta=0.0014$) using spectral norm bound \eqref{eq:upper-bound-F-2norm}, which has a test accuracy 90.77\%, as references. 
Figure \ref{fig:Fig-6} presents the training loss and test accuracy versus the number of training epochs for ResNet-20 on CIFAR-10 dataset. The training loss keeps decreasing and becomes stable after 120 epochs with a smaller learning rate. Note that the original case with no regularization term has the smallest training loss, and the upper bound \eqref{eq:upperibound-2norm} has a larger training loss because it is less tighter than \eqref{eq:upper-bound-F-2norm}. The test accuracy has a similar behavior, and the regularizer using \eqref{eq:upper-bound-F-2norm} has a higher accuracy (0.3\%) than \eqref{eq:upperibound-2norm}, due to the more tighter upper bound. Both spectral norm regularizers have improvement, 1.1\% with \eqref{eq:upper-bound-F-2norm} as the regularizer and 0.8\% with \eqref{eq:upperibound-2norm} as the regularizer, over the the one with no regularizer, which demonstrates the effective of spectral regularization in enhancing generalization performance. 

\begin{table}[t]
\caption{Comparison of test accuracy with spectral norm regularization.}
\label{tab:prediction}
\vskip -0.1in
\begin{center}
\begin{small}
\begin{sc}
\begin{tabular}{c|c|c|c|c}
\toprule
$\beta$ & 0.0008 & 0.001 &  0.0014 &  0.0018\\
\midrule
Accuracy & 90.40\% & 90.35\%  & {\bf 90.48\%} & 90.24\%\\
\bottomrule
\end{tabular}
\end{sc}
\end{small}
\end{center}
\vskip -0.25in
\end{table}

Table \ref{tab:prediction} collects the test accuracy with regularization using the spectral norm bound \eqref{eq:upperibound-2norm} with different values of $\beta$.
In addition to the observations in the main text, we observe that different values of $\beta$ make different trade-off between loss and spectral regularization, and the choice of $\beta=0.0014$ as that in \cite{Singla2019} for ResNet-34 yields the best generalization performance.

\end{document}